\documentclass[journal]{IEEEtran}
\usepackage{amsmath,amssymb,amsfonts}
\usepackage{cite}
\usepackage{algorithm}
\usepackage{algorithmic}
\usepackage{graphicx}
\usepackage{caption}
\usepackage{lipsum}
\usepackage{enumerate}
\usepackage{amsthm}
\usepackage{color}
\usepackage{threeparttable}
\usepackage{booktabs}
\usepackage{listings}
\usepackage{gensymb}
\usepackage{multirow}
\usepackage{color,soul}

\newtheorem{lemma}{Lemma}
\newtheorem{theorem}{Theorem}
\newtheorem{assumption}{Assumption}

\ifCLASSOPTIONcompsoc
  \usepackage[caption=false,font=normalsize,labelfont=sf,textfont=sf]{subfig}
\else
  \usepackage[caption=false,font=footnotesize]{subfig}
\fi
\begin{document}
\title{Integrated Decision and Control: Towards Interpretable and Computationally Efficient Driving Intelligence}

\author{Yang Guan\textsuperscript{1}, Yangang Ren\textsuperscript{1}, Qi Sun\textsuperscript{1}, Shengbo Eben Li*\textsuperscript{1}, Haitong Ma\textsuperscript{1}, Jingliang Duan\textsuperscript{1}, Yifan Dai\textsuperscript{2}, Bo Cheng\textsuperscript{1}
\thanks{This work is supported by International Science \& Technology Cooperation Program of China under 2019YFE0100200, NSF China with 51575293, and U20A20334. It is also partially supported by Tsinghua University-Toyota Joint Research Center for AI Technology of Automated Vehicle.}
\thanks{\textsuperscript{1}School of Vehicle and Mobility, Tsinghua University, Beijing, 100084, China. \textsuperscript{2}Suzhou Automotive Research Institute, Tsinghua University, Suzhou, 215200, China. All correspondence should be sent to S. Eben Li. $<$lisb04@gmail.com$>$.}
}
\maketitle

\begin{abstract}
Decision and control are core functionalities of high-level automated vehicles. Current mainstream methods, such as functionality decomposition and end-to-end reinforcement learning (RL), either suffer high time complexity or poor interpretability and adaptability on real-world autonomous driving tasks.
In this paper, we present an interpretable and computationally efficient framework called integrated decision and control (IDC) for automated vehicles, which decomposes the driving task into static path planning and dynamic optimal tracking that are structured hierarchically. First, the static path planning generates several candidate paths only considering static traffic elements. Then, the dynamic optimal tracking is designed to track the optimal path while considering the dynamic obstacles. To that end, we formulate a constrained optimal control problem (OCP) for each candidate path, optimize them separately and follow the one with the best tracking performance. To unload the heavy online computation, we propose a model-based reinforcement learning (RL) algorithm that can be served as an approximate constrained OCP solver. Specifically, the OCPs for all paths are considered together to construct a single complete RL problem and then solved offline in the form of value and policy networks, for real-time online path selecting and tracking respectively. We verify our framework in both simulations and the real world. Results show that compared with baseline methods IDC has an order of magnitude higher online computing efficiency, as well as better driving performance including traffic efficiency and safety. In addition, it yields great interpretability and adaptability among different driving tasks. The effectiveness of the proposed method is also demonstrated in real road tests with complicated traffic conditions.
\end{abstract}


\section{Introduction}
Intelligence of automobile technology and driving assistance system has great potential to improve safety, reduce fuel consumption and enhance traffic efficiency, which will completely change the road transportation. Decision and control are indispensable for high-level autonomous driving, which are in charge of computing the expected instructions of steering and acceleration relying on the environment perception results. It is generally believed that there are two technical routes for the decision and control of automated vehicles: decomposed scheme and end-to-end scheme.

Decomposed scheme splits the decision and control functionality into several serial submodules, such as prediction, behavior selection, trajectory planning and control \cite{paden2016survey}. Prediction is to predict the future trajectory of traffic participants to determine the feasible region in future time steps \cite{lefevre2014survey}. It is further decomposed into behavior recognition \cite{li2017estimation, sun2018proba} and trajectory prediction \cite{alahi2016social, hou2019interactive}.
Since the prediction algorithms usually works on each surrounding vehicle, it means that the more the number of vehicles, the more computation is needed. Behavior selection is then used to choose a high-level driving behavior relying on an expert system in which many designed rules are embedded. Typical methods include finite state machine \cite{montemerlo2008junior} and decision tree \cite{olsson2016dt}.
Based on the selected behavior, a collision free space-time curve satisfying vehicle dynamics is calculated according to the predicted trajectories and road constraints by the trajectory planning submodule. Three main categories of the planning algorithms include optimization-based, search-based and sample-based. Optimization-based methods formulate the planning problem into an optimization problem, where specific aspects of trajectory are optimized and constraints are considered \cite{gray2012predictive, nilsson2014manoeuvre}. However, it suffers from long computational time for large-scale nonlinear and non-convex problem. The search-based methods represented by A* and Rapidly-exploring Random Tree are more efficient \cite{dolgov2010path, lee2014local, ardakani2015decremental, lavalle1998rapidly, kuwata2009real,karaman2010incremental}, but they usually lead to low-resolution paths and can barely take dynamic obstacles into consideration. The sample-based methods also have poor computing efficiency because they sample points and interpolate them evenly in the whole state space \cite{shah2015autonomous,mouhagir2017trajectory}. Xin \emph{et al.} proposed a combination of the optimization-based and search-based methods, where a trajectory is searched in the space-time by A* and then smoothed by model predictive control (MPC), yielding the best performance in terms of planning time and comfort \cite{xin2021enable}. Finally, the controller is used to follow the planned trajectory and calculate the expected controls by linear quadratic regulator or MPC \cite{eben2013economy, li2010model}.
The decomposed scheme requires large amount of human design but is still hard to cover all possible driving scenarios due to the long tail effect. Besides, the real time performance cannot be guaranteed because it is time-consuming to complete all the works serially in a limited time for industrial computers.

End-to-end scheme computes the expected instructions directly from inputs given by perception module using a policy usually carried out by a deep neural network (NN).
Reinforcement learning (RL) methods do not rely on labelled driving data but learn by trial-and-error in real-world or a high fidelity simulator \cite{EbenRL,Kiran2021deep}.
Early RL applications on autonomous driving mainly focus on learning a single driving behavior, e.g., lane keeping \cite{sallab2017deep}, lane changing \cite{wang2018reinforcement} or overtaking \cite{ngai2011multiple}. They usually employ deep Q-networks \cite{mnih2015human} or deep deterministic policy gradient method \cite{lillicrap2015ddpg} to learn policy in discrete or continuous domain. Besides, they own different reward functions for their respective goals. Recently, RL has been applied in certain driving scenarios. Duan \emph{et al.} realized decision making under a virtual two-lane highway using hierarchical RL, which designs complicated reward functions for its high-level maneuver selection and three low-level maneuvers respectively. Guan \emph{et al.} achieved centralized control in a four-leg single-lane intersection with sparse rewards, in which only eight cars are considered \cite{guan2020centralized}. Chen \emph{et al.} designed a bird-view representation and used visual encoding to capture the low-dimensional latent states, solving the driving task in a roundabout scenario with dense surrounding vehicles in a high-definition driving simulator \cite{chen2019model}. However, they reported limited safety performance and poor learning efficiency. Current RL methods mostly work on a specific task, in which a set of complicated reward functions is required to offer guidance for policy optimization, such as, distance travelled towards a destination, collisions with other road users or scene objects, maintaining comfort and stability while avoiding extreme acceleration, braking or steering. It is non-trivial and needs a lot of human efforts to tune, causing poor adaptability among driving scenarios and tasks. Besides, the outcome of the policy is hard to interpret, which makes it barely used in real autonomous driving tasks. Moreover, they cannot deal with safety constraints explicitly and suffer from low convergence speed.

In this paper, we propose an integrated decision and control framework (IDC) for automated vehicles, which has great interpretability and online computing efficiency, and is applicable in different driving scenarios and tasks.
The contributions emphasize in three parts: 

1) We proposed an IDC framework for automated vehicles, which decomposes driving tasks into static path planning and dynamic optimal tracking hierarchically. The high-level static path planning is used to generate multiple paths only considering static constraints such as road topology, traffic lights. The low-level dynamic optimal tracking is used to select the optimal path and track it considering dynamic obstacles, wherein a finite-horizon constrained optimal control problem (OCP) is constructed and optimized for each candidate path. The optimal path is selected as the one with the lowest optimal cost function. The IDC framework is computationally efficient because we unload the heavy online optimizations by solving the constrained OCPs offline in the form of value and policy NNs using RL for path selecting and tracking thereafter. It is interpretable in the sense that the solved value and policy functions are the approximation for the optimal cost and the optimal action of the constrained OCP. Moreover, the IDC employs RL to solve a task-independent OCP with tracking errors as objective and safety constraints, making it applicable among a variety of scenarios and tasks.

2) We develop a model-based RL algorithm called generalized exterior point method (GEP) for the purpose of approximately solving OCP with large-scale state-wise constraints. The GEP is in fact an extension of the exterior point method in the optimization domain to the field of NN, in which it first constructs an extensive problem involving all the candidate paths and transforms it into an unconstrained problem with a penalty on safety violations. Afterward, the approximate feasible optimal control policy is obtained by alternatively performing gradient descent and enlarging the penalty. The convergence of the GEP is proved. The GEP is the core of IDC because it can deal with a large number of state-wise constraints explicitly and update NNs efficiently with the guidance of model. To the best of our knowledge, GEP is the first model-based solver for OCPs with large-scale state-wise constraints that are parameterized by NNs.

3) We evaluate the proposed method extensively in both simulations and in a real-world road to verify the performance in terms of online computing efficiency, safety, and task adaptation, etc. The results show the potential of the method to be applied in real-world autonomous driving tasks.

\section{Integrated decision and control framework}
In this section, we introduce the framework of integrated decision and control framework (IDC). As shown in Fig. \ref{fig.overall_framework}, the framework consists of two layers: static path planning and dynamic optimal tracking.

Different from existing schemes, the upper layer aims to generate multiple candidate paths only considering static information such as road structure, speed limit, traffic signs and lights. Note that these paths will not include time information. Each candidate path is attached with an expected velocity determined by rules from human experience. 

The lower layer further considers the candidate paths and the dynamic information such as surrounding vehicles, pedestrians and bicycles. For each candidate path, a constrained optimal control problem (OCP) is formulated and optimized to choose the optimal path and find the control command. The objective function is to minimize the tracking error within a finite horizon and the constraints characterize safety requirements. In each time step, the optimal path is chosen as the one with the lowest optimal cost function and thereafter tracked. The core of our method is to substitute all the expensive online optimizations with feed-forward of two neural networks (NNs) trained offline by reinforcement learning (RL). Specifically, we first formulate a complete RL problem considering all the candidate paths. And then we develop a model-based RL algorithm to solve this problem to obtain a policy NN called actor that is capable of tracking different shape of paths while maintaining the ability to avoid collisions. Meanwhile, a value NN called critic is learned to approximate the optimal cost of tracking different paths, for the purpose of online path selection.
\begin{figure*}[!htb]
\centering{\includegraphics[width=0.7\textwidth]{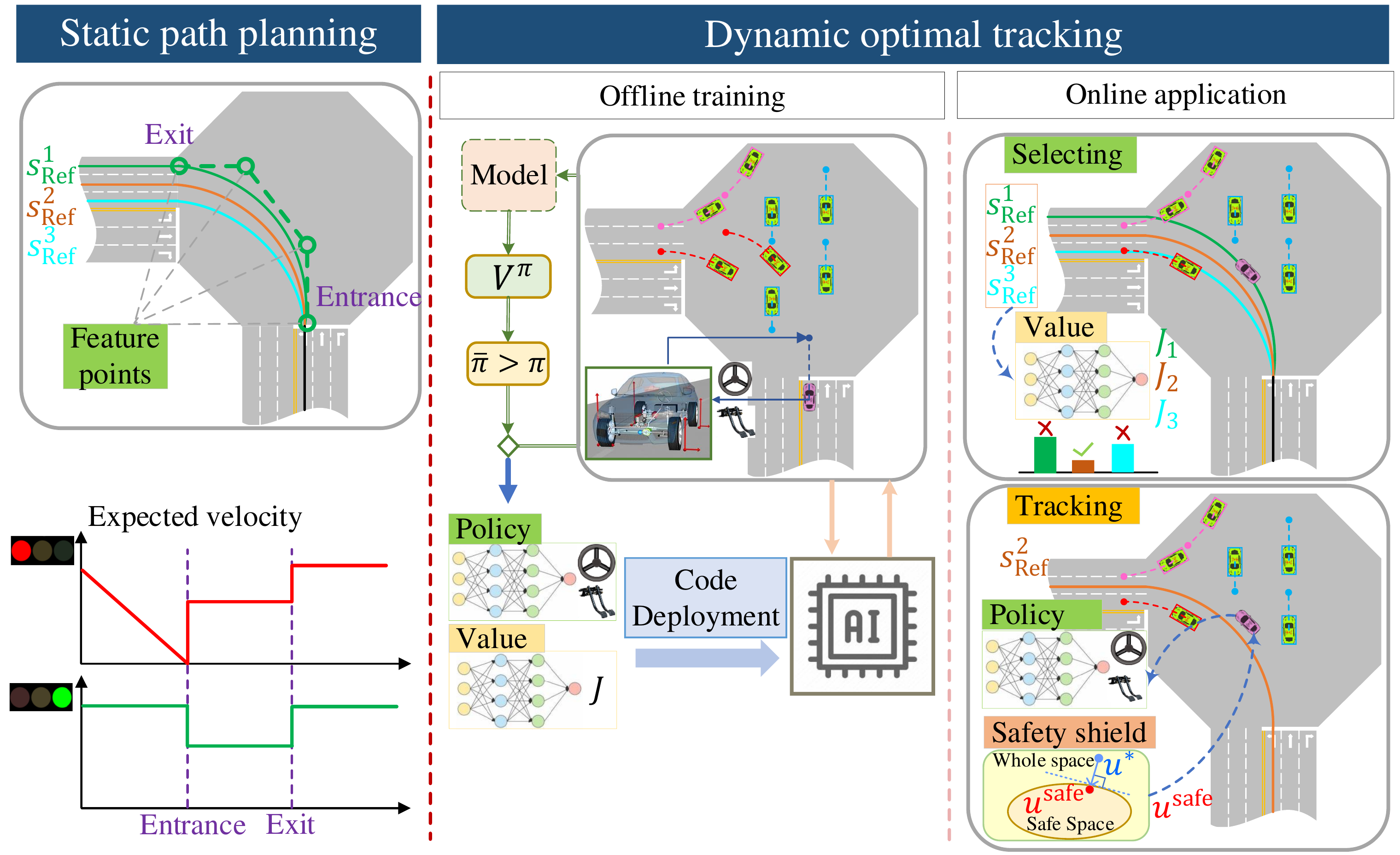}}
\caption{Illustration of the integrated decision and control framework.}
\label{fig.overall_framework}
\end{figure*}

The advantages of the IDC framework are summarized in three points. 
First, it has high online computing efficiency. The upper layer can be quite efficient because it involves only static information. It even allows to embed key parameters of paths planned in advance into the electronic map and read from it directly. On the other hand, the lower layer utilizes two trained networks to compute optimal path and control command, which is also time-saving due to the fast propagation of NNs.
Second, it can be easily transferred among different driving scenarios without a lot of human design. As the upper layer only uses the static information, the multiple paths can be easily generated by the road topology of various scenes such as intersections, roundabouts and ramps. Besides, the lower-layer always formulates a similar tracking problem with safety constraints no matter what the task is, which all can be solved by the developed model-based solver, saving time to design separate reward functions for different tasks.
Third, the IDC framework are interpretable in the way that the learned value and policy function approximate the optimal value and the optimal action of the constrained OCPs. These optimal solutions can also be obtained by model predictive control (MPC) to verify the correctness of the trained NNs. Besides, the functionality partitioning in the IDC framework helps us to explain the results of path selection and tracking accordingly.
\section{Static path planning}
This module aims to generate multiple candidate paths for optimal tracking of the lower layer meanwhile maintaining high computational efficiency. For that purpose, the Cubic Bezier curve is adopted to obtain a continuous and smooth path. Given the road map, we choose four feature points to shape of Bezier curve and generate multiple paths with considerations of static traffic information such as road topology, traffic rules, etc., following the Algorithm \ref{alg.multipath}. We summary two strategies to generate candidate paths. One is the pre-generating method, in which these paths are produced in advance and do not change with the riding of the ego vehicle, as illustrated in Fig.~\ref{fig.overall_framework}.
The other one is the real-time generating method where the paths are planned in real-time and always start from the ego vehicle.
Apparently, the former is simple and convenient to be directly embedded in the map, and has higher efficiency because it only needs to read the pre-stored paths every step. By contrast, the latter has higher flexibility and may conduct more complex driving behaviors, but its online computing efficiency will be a bit lower than the former. Both methods will be much more efficient than current existing methods as they simply generate regular candidate paths without considering collision avoidance with the dynamic obstacles. Note that the planned paths only serve as references for the lower layer, the actual travelled path can be largely different due to further considerations on dynamic obstacles.

\begin{algorithm}[htbp]
  \caption{Static path planning}
  \label{alg.multipath}
\begin{algorithmic}
  \STATE {\bfseries Initialize:} discrete point number $M$
    \FOR{each path}
        \STATE Choose four feature points: start point $(X_0,Y_0)$, control points $(X_1, Y_1)$,$(X_2,Y_2)$, and end point $(X_3,Y_3)$ 
    \FOR{$t=0:1/M:1$}
         \STATE \small$p^{\rm ref}_{\rm x}(t)=X_0(1-t)^3+3X_1t(1-t)^2+3X_2t^2(1-t)+X_3t^3$
         \STATE $p^{\rm ref}_{\rm y}(t)=Y_0(1-t)^3+3Y_1t(1-t)^2+3Y_2t^2(1-t)+Y_3t^3$
         \STATE $\phi^{\rm ref}(t)=\arctan(\frac{Y(t)-Y(t-1)}{X(t)-X(t-1)})$
    \ENDFOR
    \STATE Output $\left\{(p^{\rm ref}_{\rm x}, p^{\rm ref}_{\rm y}, \phi^{\rm ref})\right\}$
    \ENDFOR
\end{algorithmic}
\end{algorithm}

As for the expected velocity, we heuristically assign different speed levels with respect to road regions, traffic signals as well as traffic rules such as speed limits or stop signs, which can be quickly designed according to human knowledge. An example is shown in Fig.~\ref{fig.overall_framework}. Similar to the candidate paths, the expected velocity provides a goal for the lower layer to track but not necessarily to follow strictly, so that the ego vehicle seeks to minimize the tracking error while satisfying safety constraints. Actually, it can be simply a fixed value, the lower layer will still always learn a driving policy to balance the safety requirements and tracking errors.

\section{Dynamic optimal tracking}
\subsection{Problem formulation}\label{sec.problem_formulation}
In each time step $t$, provided multiple candidate paths generated, the lower layer is designed to first select an optimal path $\tau^*\in\Pi$ according to a certain criterion, where $\Pi$ denotes a collection of $N$ candidate paths. And then it obtains the control quantities $u_t$ by optimizing a finite horizon constrained OCP, in which the objective is to minimize the tracking error as well as the control energy, and the constraints are to keep a safe distance from obstacles:
\begin{equation}\label{eq.lower_layer_ocp}
\begin{aligned}
\min\limits_{u_{i|t}, i=0:T-1}\quad &J=\sum_{i=0}^{T-1}(x^{\rm ref}_{i|t}-x_{i|t})^\top Q (x^{\rm ref}_{i|t}-x_{i|t}) + u^{\top}_{i|t} R u_{i|t}\\
{\rm s.t.} \quad& x_{i+1|t} = F_{\rm ego}(x_{i|t}, u_{i|t}),\\
                & x^j_{i+1|t} = F_{\rm pred}(x^j_{i|t}),\\
                & (x_{i|t}-x^j_{i|t})^\top M (x_{i|t}-x^j_{i|t}) \ge D^{\rm safe}_{\rm veh},\\
                & (x_{i|t}-x^{\rm road}_{i|t})^\top M (x_{i|t}-x^{\rm road}_{i|t}) \ge D^{\rm safe}_{\rm road},\\
                & x_{i|t}\le L_{\rm stop}, {\rm if\  light=red}\\
                & x_{0|t} = x_t, x^j_{0|t} = x^j_t, u_{0|t} = u_t\\
                &  i=0:T-1, j\in I\\
\end{aligned}
\end{equation}
where $T$ is the prediction horizon, $x_{i|t}$ and $u_{i|t}$ are the ego vehicle state and control in the virtual predictive time step $i$ starting from the current time step $t$, where ``virtual" means the time steps within the predictive horizon. $x^{\rm ref}_{i|t}$ and $x^{\rm road}_{i|t}$ are the closest point from $x_{i|t}$ on the selected reference $\tau^*$ and on the road edge, respectively. $x^j_{i|t}$ is the state of the $j$-th vehicle in the interested vehicle set $I$. $Q, R, M$ are positive-definite weighting matrices. $F_{\rm ego}$ represents the bicycle vehicle dynamics with linear tire model. $F_{\rm pred}$, on the other hand, is the surrounding vehicle prediction model. Besides, $D^{\rm safe}_{\rm veh}$ and $D^{\rm safe}_{\rm road}$ denote the safe distance from other vehicles and the road edge. $L_{\rm stop}$ is the position of stop line. Note that in \eqref{eq.lower_layer_ocp} the virtual states are all produced by the dynamics model and the prediction model except that the $x_{0|t}$ is assigned with the current real state $x_t$. The variables and functions in \eqref{eq.lower_layer_ocp} are further defined as:
\begin{equation}\label{eq.variable_definition1}
\begin{aligned}
&x^{\rm ref}_{i|t}=
\begin{bmatrix}
p^{\rm ref}_{\rm x}\\
p^{\rm ref}_{\rm y}\\
v^{\rm ref}_{\rm lon}\\
0\\
\phi^{\rm ref}\\
0\\
\end{bmatrix}_{i|t}
x^{\rm road}_{i|t}=
\begin{bmatrix}
p^{\rm road}_{\rm x}\\
p^{\rm road}_{\rm y}\\
0\\
0\\
0\\
0\\
\end{bmatrix}_{i|t}
x_{i|t}=
\begin{bmatrix}
p_{\rm x}\\
p_{\rm y}\\
v_{\rm lon}\\
v_{\rm lat}\\
\phi\\
\omega\\
\end{bmatrix}_{i|t}\\
&x^j_{i|t}=
\begin{bmatrix}
p^j_{\rm x}\\
p^j_{\rm y}\\
v^j_{\rm lon}\\
0\\
\phi^j\\
0\\
\end{bmatrix}_{i|t}
u_{i|t}=
\begin{bmatrix}
\delta\\
a\\
\end{bmatrix}_{i|t}
M=
\begin{bmatrix}
1& 0\\
0& 1\\
\end{bmatrix}\\
\end{aligned}
\end{equation}

\begin{equation}\label{eq.variable_definition2}
\begin{aligned}
&F_{\rm ego}=
\begin{bmatrix}
p_{\rm x} +\Delta t (v_{\rm lon}\cos{\phi}-v_{\rm lat}\sin{\phi})\\
p_{\rm y} +\Delta t (v_{\rm lon}\sin{\phi}+v_{\rm lat}\cos{\phi})\\
v_{\rm lon}+\Delta t (a+v_{\rm lat}\omega)\\
\frac{mv_{\rm lon}v_{\rm lat}+\Delta t[(L_f k_f-L_r k_r)\omega-k_f \delta v_{\rm lon}-mv_{\rm lon}^2\omega]}{mv_{\rm lon}-\Delta t (k_f + k_r)}\\
\phi + \Delta t \omega\\
\frac{-I_z \omega v_{\rm lon}-\Delta t[(L_f k_f-L_r k_r)v_{\rm y}-L_f k_f \delta v_{\rm lon}]}{\Delta t(L^2_f k_f+L^2_r k_r)-I_z v_{\rm lon}}\\
\end{bmatrix}\\
\end{aligned}
\end{equation}

\begin{equation}\label{eq.variable_definition3}
\begin{aligned}
&F_{\rm pred}=
\begin{bmatrix}
p^j_{\rm x} +\Delta t (v^j_{\rm lon}\cos{\phi^j}-v^j_{\rm lat}\sin{\phi^j})\\
p^j_{\rm y} +\Delta t (v^j_{\rm lon}\sin{\phi^j}+v^j_{\rm lat}\cos{\phi^j})\\
v^j_{\rm lon}\\
0\\
\phi^j + \Delta t \omega^j_{\rm pred}\\
0\\
\end{bmatrix}\\
\end{aligned}
\end{equation}
where $p_{\rm x}, p_{\rm y}$ are the position coordinates, for ego and other vehicles, it is the position of their respective center of gravity (CG), $v_{\rm lon}, v_{\rm lat}$ are the longitudinal and lateral velocities, $\phi$ is the heading angle, $\omega$ is the yaw rate, $\delta$ and $a$ are the front wheel angle and the acceleration commands, respectively. The ego dynamics model is discretized by the first-order Euler method, which has been proved to be numerically stable at any low speed \cite{ge2020numerically}. Other vehicle parameters are listed in Table \ref{tab.vehicle_parameters}. The vehicle prediction model is a simple deduction from the current state with constant speed and turning rate $w^j_{\rm pred}$, which depends on the driving scenarios and will be specified in the experiments.

The optimal path is chosen as the one with the best tracking performance while satisfying the safety requirements, i.e.,
\begin{equation}\label{eq.path_selection}
\tau^* = \arg\min\limits_{\tau}\{J^*_{\tau}|\tau\in\Pi\}
\end{equation}
where $J_{\tau}^*$ is the optimal cost of the path $\tau$. This means that for each path candidate $\tau\in\Pi$ we ought to construct such an OCP \eqref{eq.lower_layer_ocp} and to optimize it to obtain its optimal value $J^*_{\tau}$.
Such a criterion of path selection is consistent with the objective of the optimal control problem \eqref{eq.lower_layer_ocp} in the sense that it optimizes the problem with respect to paths within the candidate path set, compared to the original optimization with respect to control quantities.

In such framework of selecting and tracking, the lower layer is able to well determine a control quantity that yields good driving efficiency and, meanwhile, the safety guarantee. But unfortunately, so far, the time complexity is extremely high for on-board vehicle computing devices, as in each time step we need to solve $N$ constrained optimal control problem, while each of them owns up to hundreds of variables and thousands of constrains. Therefore, we employ RL to unload the online optimization burden. Specifically, we show that this framework naturally corresponds to the actor-critic architecture of RL, where the critic, with the function of judging state goodness, can be served as the path selector while the actor is in charge of action output and used for tracking. With a paradigm of offline training and online application, the computation burden in the lower layer can be almost nearly eliminated.

\subsection{Offline training}
\subsubsection{Complete RL problem formulation}\label{sec.multitask}
We aim to solve the path selecting and path tracking problems \eqref{eq.lower_layer_ocp} and \eqref{eq.path_selection} by RL. Notice that there are several significant differences between the OCP \eqref{eq.lower_layer_ocp} and RL problems, originated from the online or offline optimizations. The OCP, which is designed for online optimization, seeks to find a single optimal control quantity of a single state given a specific path at time step $t$. RL problems, on the other hand, aim to solve a parameterized control policy called actor that maps from state space $\mathcal{S}$ to control space $\mathcal{A}$ as well as a parameterized value function called critic that evaluates the preference to a state, in an offline style. Therefore, in RL problems the variable to be optimized is no longer the control quantity but the parameters of the actor and critic that usually in the form of NN. In addition, the objective function and constraints of RL are not about a single state any more, but about a state distribution in the state space. The state $s\in\mathcal{S}$ is the input of the actor and critic, which is designed to contain necessary information to determine driving actions. Except for the information of the ego vehicle and surrounding vehicles, we also incorporate the path information as a part of states to get a policy that can handle tracking tasks for different paths. The resulting RL problem is shown as:
\begin{equation}\label{eq.converted_rl_problem_actor}
\begin{aligned}
\min\limits_{\theta} \quad &J_{\rm actor} = \mathbb{E}_{s_{0|t}}\bigg\{\sum^{T-1}_{i=0}l(s_{i|t}, \pi_{\theta}(s_{i|t}))\bigg\}\\
{\rm s.t.}\quad & s_{i+1|t}=f(s_{i|t}, \pi_{\theta}(s_{i|t}))\\
& g_e(s_{i|t})\ge0, e\in E\\
& s_{0|t} = s_t\leftarrow\{\tau, x_t, x_t^j, j\in J\} \sim d,\\
& i=0:T-1\\
\end{aligned}
\end{equation}

\begin{equation}\label{eq.converted_rl_problem_critic}
\begin{aligned}
\min\limits_{w} \quad &J_{\rm critic} = \mathbb{E}_{s_{0|t}}\bigg\{\bigg(\sum^{T-1}_{i=0}l(s_{i|t}, \pi_{\theta}(s_{i|t})) - V_w(s_{0|t})\bigg)^2\bigg\}\\
{\rm s.t.}\quad & s_{i+1|t}=f(s_{i|t}, \pi_{\theta}(s_{i|t}))\\
& s_{0|t} = s_t \sim d,\\
& i=0:T-1\\
\end{aligned}
\end{equation}
where $s_t\leftarrow\{\tau, x_t, x_t^j, j\in J\}$ denotes the state is constructed using the information of reference path, ego vehicle state and surrounding vehicle states. $l(s_{i|t}, \pi_{\theta}(s_{i|t})):=(x^{\rm ref}_{i|t}-x_{i|t})^\top Q (x^{\rm ref}_{i|t}-x_{i|t})+\pi_{\theta}^{\top}(s_{i|t}) R \pi_{\theta}(s_{i|t})$. $f$ denotes the system model, which is an aggregation of the $F_{\rm ego}$ and $F_{\rm pred}$. $g_e(s_{i|t}), e\in E$ denotes all the constraints about the state $s_{i|t}$, including that with other vehicles, road, and traffic rules. $d$ denotes the state distributions sampled from the environment. $\pi_{\theta}: \mathcal{S}\rightarrow\mathcal{A}$ and $V_w: \mathcal{S}\rightarrow\mathbb{R}$ are actor and critic, parameterized by $\theta$ and $w$ that are generally in form of NNs, respectively. From the results of \cite{allen2019convergence}, given over-parameterized NNs, the optimal policy $\pi_{\theta_*}$ of \eqref{eq.converted_rl_problem_actor} maps to an optimal action of the original OCP \eqref{eq.lower_layer_ocp} with arbitrary initial state $s_t$. Consequently, the optimal value $J^*$ from \eqref{eq.lower_layer_ocp}, of course, would be equal to the one mapped by the optimal value function $V_{w_*}$ of \eqref{eq.converted_rl_problem_critic} from $s_t$, i.e.,
\begin{equation}\label{eq.equivalence}
\begin{aligned}
u^*_t &= \pi_{\theta_*}(s_t), \ \forall{s_t \in \mathcal{S}}\\
J^* &= V_{w_*}(s_t), \ \forall{s_t \in \mathcal{S}}\\
\end{aligned}
\end{equation}
In other words, the optimal policy and value function can output the optimal control and value under arbitrary states, i.e., arbitrary combinations of paths, ego state and surrounding vehicle states.

\subsubsection{Solver - GEP}
To solve the converted RL problem, we adopt the policy iteration framework, wherein two procedures, namely policy evaluation and policy improvement, are alternatively performed to update the critic and actor. Since the critic update is an unconstrained problem that can be optimized by ordinary gradient descent methods, we mainly focus on the actor update which is quite challenging because of its large-scale parameter space, nonlinear property and infinite number of state constraints. To tackle this, we propose a model-based RL algorithm called generalized exterior point method (GEP) adapted from the one in the optimization field. It first transforms the constrained problem \eqref{eq.converted_rl_problem_actor} into an unconstrained one by the exterior penalty function, shown as:
\begin{equation}\label{eq.unconstrained_rl}
\begin{aligned}
\min\limits_{\theta} \quad J_p&=J_{\rm actor} + \rho J_{\rm penalty}\\
&=\mathbb{E}_{s_{0|t}}\bigg\{\sum^{T-1}_{i=0}l(s_{i|t}, \pi_{\theta}(s_{i|t}))\bigg\}+\rho\mathbb{E}_{s_{0|t}}\bigg\{\sum^{T-1}_{i=0}\varphi_i(\theta)\bigg\}\\
{\rm s.t.}\quad& s_{i+1|t}=f(s_{i|t}, \pi_{\theta}(s_{i|t}))\\
& \varphi_i(\theta) = \sum_{e\in E}[\max\{0, -g_e(s_{i|t})\}]^2\\
& s_{0|t} = s_t \sim d,\\
& i=0:T-1\\
\end{aligned}
\end{equation}
where $\varphi$ is the penalty function, $\rho$ is the penalty factor. After that, we alternatively optimize the policy parameters by performing $m$ iterations of gradient descent and increase the penalty factor by multiplying a scalar $c>1$. We call the former the optimizing procedure and the latter the amplifying procedure. Different from exterior point method, the optimizing procedure of GEP does not necessarily find the optimal solution of the unconstrained problem \eqref{eq.unconstrained_rl}, but we will prove that GEP still converges to the optimal policy under certain conditions. It can be seen that GEP is simple to implement and is powerful to deal with large-scale parameter space facilitated by the gradient descent technique. Besides, from the form of \eqref{eq.unconstrained_rl}, numerous state constraints can be handled naturally by regarding the constraint violation as a term of utility function multiplied by $\rho$. The training pipeline is shown in Algorithm \ref{alg.offline}.
\begin{algorithm}[tbp]
  \caption{Dynamic optimal tracking - Offline training}
  \label{alg.offline}
\begin{algorithmic}
  \STATE {\bfseries Initialize:} critic network $V_w$ and actor network $\pi_{\theta}$ with random paramaters $w, \theta$, buffer $\mathcal{B}\leftarrow\emptyset$, learning rates $\beta_w, \beta_{\theta}$, penalty factor $\rho=1$, penalty amplifier $c$, update interval $m$
  \FOR{each iteration $i$}
      \STATE // Sampling (from environment)
      \STATE Randomly select a path $\tau\in\Pi$, initialize ego state $x_t$ and vehicle states $x^j_t, j\in I$
      \FOR{each environment step}
      \STATE $s_{t}\leftarrow\{\tau, x_t, x^j_t, j\in I\}$
      \STATE $\mathcal{B}\cup\{s_t\}$
      \STATE $u_t=\pi_{\theta}(s_t)$
      \STATE Apply $u_t$ to observe $x_{t+1}$ and $x^j_{t+1}, j\in I$
      \ENDFOR
      \STATE 
      \STATE // Optimizing (GEP)
      \STATE Fetch a batch of states from $\mathcal{B}$, compute $J_{\rm critic}$ and $J_p$ by $f$ and $\pi_{\theta}$
      \STATE \textbf{PEV}: $w\leftarrow w-\beta_w \nabla_w J_{\rm critic}$
      \STATE \textbf{PIM}: if $i\mod m$:$\rho\leftarrow c\rho$; $\theta\leftarrow\theta-\beta_{\theta}\nabla_{\theta}J_p$
  \ENDFOR
\end{algorithmic}
\end{algorithm}

Next, we present the convergence proof of GEP. We say a ``round" completes when an optimizing procedure is finished. Since the optimizing procedure only improves \eqref{eq.unconstrained_rl} a fixed number of times to obtain a fair solution but does not necessarily find the optimal one, we first give an assumption about how well the solution is.
\begin{assumption}\label{assu1}
After the round $k$ completes, we have the penalty factor $\rho_k$ and an optimized policy parameter $\theta_{k}$. We assume that $\theta_{k}$ satisfies
\begin{equation}
J_p(\theta_k, \rho_{k})\le \min_{\theta} J_p(\theta, \rho_{k}) + \Delta_k, k=1,2,\dots
\end{equation}
where $\Delta_k\ge0, k\ge1$ is a positive non-increasing sequence that has finite series, i.e., $\Delta_k\ge\Delta_{k+1}, \sum_{i=0}^{\infty}\Delta_i< \infty$.
\end{assumption}
The Assumption \ref{assu1} describes that with the convergence of NNs, the gap between the solution of the optimizing procedure and the optimal one is gradually eliminated, i.e., $\lim_{k\rightarrow\infty}\Delta_k=0$, as indicated by the finite series. About the gap, we have the following Lemma.
\begin{lemma}\label{lemma.1}
There exists a positive non-increasing sequence $\delta_k, k\ge1$ that satisfies
\begin{equation}
\begin{aligned}
    &\Delta_k = \delta_k - \delta_{k+1},\\
    &\delta_k\ge\delta_{k+1},
    \lim_{k\rightarrow\infty}\delta_k=0
\end{aligned}
\end{equation}
\end{lemma}
\begin{proof}
We can construct such a sequence by setting
\begin{equation}
\begin{aligned}
    \delta_1 = \sum_{i=1}^{\infty}\Delta_i,
    \delta_{k+1} = \delta_k-\Delta_k, k\ge1
\end{aligned}
\end{equation}
Then $\delta_1<0$ holds by Assumption \ref{assu1} and the convergence of $\delta_k$ naturally holds by
\begin{equation}
\begin{aligned}
\lim_{k\rightarrow\infty}\delta_k =\lim_{k\rightarrow\infty} \delta_1-\sum_{i=1}^{k-1}\Delta_i
= \delta_1-\lim_{k\rightarrow\infty}\sum_{i=1}^{k-1}\Delta_i=0
\end{aligned}
\end{equation}
\end{proof}

Next, we first prove the following two Lemmas about the unconstrained objective.
\begin{lemma}\label{lemma.2}
For the solution sequence generated after each round \{$\theta_k$\}, we have
\begin{equation}
    J_p(\theta_{k+1}, \rho_{k+1}) - \delta_{k+1} \ge J_p(\theta_{k}, \rho_{k}) - \delta_{k} 
\end{equation}
\end{lemma}
\begin{proof}
By $J_p(\theta, \rho) = J_{\rm actor}(\theta)+\rho J_{\rm penalty}(\theta)$ and $\rho_{k+1}>\rho_k$,
\begin{equation}
\begin{aligned}
J_p(\theta_{k+1}, \rho_{k+1})
&=J_{\rm actor}(\theta_{k+1})+\rho_{k+1}J_{\rm penalty}(\theta_{k+1})\\
&\ge J_{\rm actor}(\theta_{k+1})+\rho_k J_{\rm penalty}(\theta_{k+1})\\
&=J_p(\theta_{k+1}, \rho_k)
\end{aligned}
\end{equation}
Then by the Assumption \ref{assu1}, for $\forall{\theta}$, $J_p(\theta, \rho_k)\ge \min_{\theta}J_p(\theta, \rho_k)\ge J_p(\theta_k, \rho_k)-\Delta_k$, thus we have $J_p(\theta_{k+1}, \rho_k)\ge J_p(\theta_k, \rho_k)-\Delta_k$, therefore
\begin{equation}
\begin{aligned}
    &J_p(\theta_{k+1}, \rho_{k+1})\ge J_p(\theta_k, \rho_k)-\Delta_k\\
    &J_p(\theta_{k+1}, \rho_{k+1})-\delta_{k+1}\ge J_p(\theta_k, \rho_k)-\delta_k\\
\end{aligned}
\end{equation}
\end{proof}

\begin{lemma}\label{lemma.3}
Suppose $\theta_*=\arg\min_{\theta}J_{\rm actor}(\theta)$, then for $\forall{k}\ge1$,
\begin{equation}
J_{\rm actor}(\theta_*)-\delta_{k+1} \ge J_p(\theta_k, \rho_k)-\delta_k \ge J_{\rm actor}(\theta_k)-\delta_k
\end{equation}
\end{lemma}
\begin{proof}
Because $\theta_*$ is the optimal solution of the problem \eqref{eq.converted_rl_problem_actor}, it has $J_{\rm penalty}(\theta_*)=0$. Then from the Assumption \ref{assu1} the first inequality is obtained,
\begin{equation}
\begin{aligned}
&J_{\rm actor}(\theta_*) = J_p(\theta_*, \rho_k) \ge \min_{\theta}J_p(\theta, \rho_k) \ge J_p(\theta_k, \rho_k) - \Delta_k\\
&J_{\rm actor}(\theta_*) - \delta_{k+1}\ge J_p(\theta_k, \rho_k)-\delta_k
\end{aligned}
\end{equation}
The second inequality is got by $\rho_k J_{\rm penalty}(\theta_k) \ge 0$
\begin{equation}
J_p(\theta_k, \rho_k) = J_{\rm actor}(\theta_k) + \rho_k J_{\rm penalty}(\theta_k)\ge J_{\rm actor}(\theta_k)
\end{equation}
\end{proof}

The convergence can be revealed by Theorem \ref{theorem.1}.
\begin{theorem}\label{theorem.1}
Assume that $J_{\rm actor}$ and $J_{\rm penalty}$ are continuous functions defined on the parameter space. Suppose \{$\theta_k$\} is the solution sequence generated after each round. The limit of any of its convergent subsequence is the optimal solution.
\end{theorem}
\begin{proof}
Suppose $\{\theta_{k_j}\}$ is an arbitrary convergent subsequence of \{$\theta_k$\} with the limit $\bar{\theta}$. By the continuity of $J_{\rm actor}$, $\lim_{k_j\rightarrow\infty}J_{\rm actor}(\theta_{k_j})=J_{\rm actor}(\bar{\theta})$. Define $J_{\rm actor}^*=\min_{\theta}J_{\rm actor}(\theta)$ as the optimal value of the problem \eqref{eq.converted_rl_problem_actor}. From Lemma \ref{lemma.2} and Lemma \ref{lemma.3}, we can see that \{$J_p(\theta_{k_j}, \rho_{k_j})-\delta_{k_j}$\} is a non-increasing sequence with upper limit $J_{\rm actor}^*$, therefore
\begin{equation}
\lim_{k_j\rightarrow\infty} (J_p(\theta_{k_j}, \rho_{k_j})-\delta_{k_j}) = \lim_{k_j\rightarrow\infty} J_p(\theta_{k_j}, \rho_{k_j})= J_p^*\le J_{\rm actor}^*
\end{equation}
Then because $J_p(\theta_{k_j}, \rho_{k_j})=J_{\rm actor}(\theta_{k_j})+\rho_{k_j}J_{\rm penalty}(\theta_{k_j})$,
\begin{equation}
\lim_{k_j\rightarrow\infty}\rho_{k_j}J_{\rm penalty}(\theta_{k_j}) = J_p^*-J_{\rm actor}(\bar{\theta})
\end{equation}
by $J_{\rm penalty}(\theta_{k_j})\ge 0$, $\rho_{k_j}\rightarrow\infty$ and the continuity of $J_{\rm penalty}$, we have
\begin{equation}
    \lim_{k_j\rightarrow\infty} J_{\rm penalty}(\theta_{k_j}) = J_{\rm penalty}(\bar{\theta}) = 0
\end{equation}
which indicates that the limit $\bar{\theta}$ is a feasible solution. Furthermore, by Lemma \ref{lemma.3}, $J_{\rm actor}(\theta_{k_j})\le J_{\rm actor}^*$, together with the continuity of $J_{\rm actor}$, we have
\begin{equation}\label{eq.gep_optimality}
J_{\rm actor}(\bar{\theta})\le J_{\rm actor}^* \Rightarrow J_{\rm actor}(\bar{\theta})= J_{\rm actor}^*
\end{equation}
where the equation holds by the definition of $J_{\rm actor}^*$. Equation \eqref{eq.gep_optimality} indicates the optimality of the limit $\bar{\theta}$.
\end{proof}

\subsection{Online application}
Ideally, we expect to get an optimal policy which is able to output the optimal action within the safety action space in any given point in the state space. Unfortunately, it is impossible to acquire such a policy in both theory and practice. First, the converted RL problem \eqref{eq.converted_rl_problem_actor} enforces constraint on every single state point, leading to an infinite number of constraints in the continuous state space. But nevertheless when we solve the equivalent unconstrained problem \eqref{eq.unconstrained_rl}, we approximate the expectation by an average of samples, that means we only consider finite constraints of the set of sample that can vary in different iterations. That is why there is no strict safety guarantee of the policy but only an approximately safe performance. Second, the condition of \eqref{eq.equivalence} that the approximation function has infinite fitting power cannot be established in practical, resulting in a suboptimal solution without safety guarantee. To ensure the safety performance, we adopt a multi-step safety shield after the output of the policy. 

\subsubsection{Multi-step safety shield}
The safety shield aims to find the nearest actions in the safe action space in state $s_t$, which is formulated as a quadratic programming problem:
\begin{equation}\label{eq:QP}
  u^{\text{safe}}_t = \left\{
\begin{aligned}
    &u^*_t, \quad \text{if}\quad u^*_t \in \mathcal{U}_{\text{safe}}(s_t) \\
    &\arg\min\limits_{u \in\mathcal{U}_{\text{safe}}(s_t)} \Vert u-u_t^*\Vert_2^2 , \quad \text{else}\\
\end{aligned}\right.
\end{equation}
where $u^*_t$ is the policy output. Rather than designing $\mathcal{U}_{\text{safe}}(s_t)$ to guarantee the safety of only the next state, we design it to guarantee that the next $n_{ss}$ prediction states are safe, i.e., collision-free with the surrounding vehicles and road edges. Formally,
\begin{equation}
\mathcal{U}_{\text{safe}}(s_t) = \{u_t|g_e(s_{i|t})\ge 0, i=1,\dots,n_{ss},e\in E\}
\end{equation}

\subsubsection{Algorithm for online application}
Given the trained policy and value functions, in online application, we simply construct a set of states for different paths, then pass them to the trained value function to select the one with the lowest value, which is next passed to the trained policy to get the optimal control, as summarized in Algorithm \ref{alg.online}.

\begin{algorithm}[htbp]
  \caption{Dynamic optimal tracking - Online application}
  \label{alg.online}
\begin{algorithmic}
  \STATE {\bfseries Initialize:} Path set $\Pi$ from upper layer, trained critic network $V_{w^*}$ and actor network $\pi_{\theta^*}$, $\lambda$, ego state $x_t$ and vehicle states $x^j_t, j\in I$
  \FOR{each environment step}
      \STATE // Selecting
      \FOR{each $\tau \in \Pi$}
      \STATE $s_{t, \tau}\leftarrow\{\tau, x_t, x^j_t, j\in I\}$
      \STATE $V^*_{\tau} = V_{w^*}(s_{t, \tau})$
      \ENDFOR
      \STATE $\tau^*=\arg\min\limits_{\tau}\{V^*_{\tau}|\tau\in\Pi\}$
      \STATE 
      \STATE // Tracking
      \STATE $s_t\leftarrow\{\tau^*, x_t, x^j_t, j\in I\}$
      \STATE $u^*_t = \pi_{\theta^*}(s_t)$
      \STATE Calculate $u^*_{\rm safe}$ by \eqref{eq:QP}
      \STATE Apply $u^*_{\rm safe}$ to observe $x_{t+1}$ and $x^j_{t+1}, j\in I$
  \ENDFOR
\end{algorithmic}
\end{algorithm}

\section{Simulation verification}\label{sec.simulation}
\subsection{Scenario and task descriptions}
We first carried out our experiments on a regular signalized four-way intersection built in the simulation, where the roads in different directions are all the six-lane dual carriageway, as shown in Fig. \ref{fig.scenario}. The junction is a square with a side length of 50m. Each entrance of the intersection has three lanes, each with a width of 3.75m, for turning left, going straight and turning right, respectively. With the help of SUMO \cite{SUMO2018}, we generate a dense traffic flow of 800 vehicles per hour on each lane. These vehicles are controlled by the car-following and lane-changing models in the SUMO, producing a variety of traffic behaviors. Moreover, a two-phase traffic signal is included to control the traffic flow of turning left and going straight. We verify our algorithm in three tasks: turn left, go straight and turn right. In each task, the ego vehicle is initialized randomly from the south entrance and is expected to drive safely and efficiently to pass the intersection.

\begin{figure}[htbp]
\centerline{\includegraphics[width=0.8\linewidth]{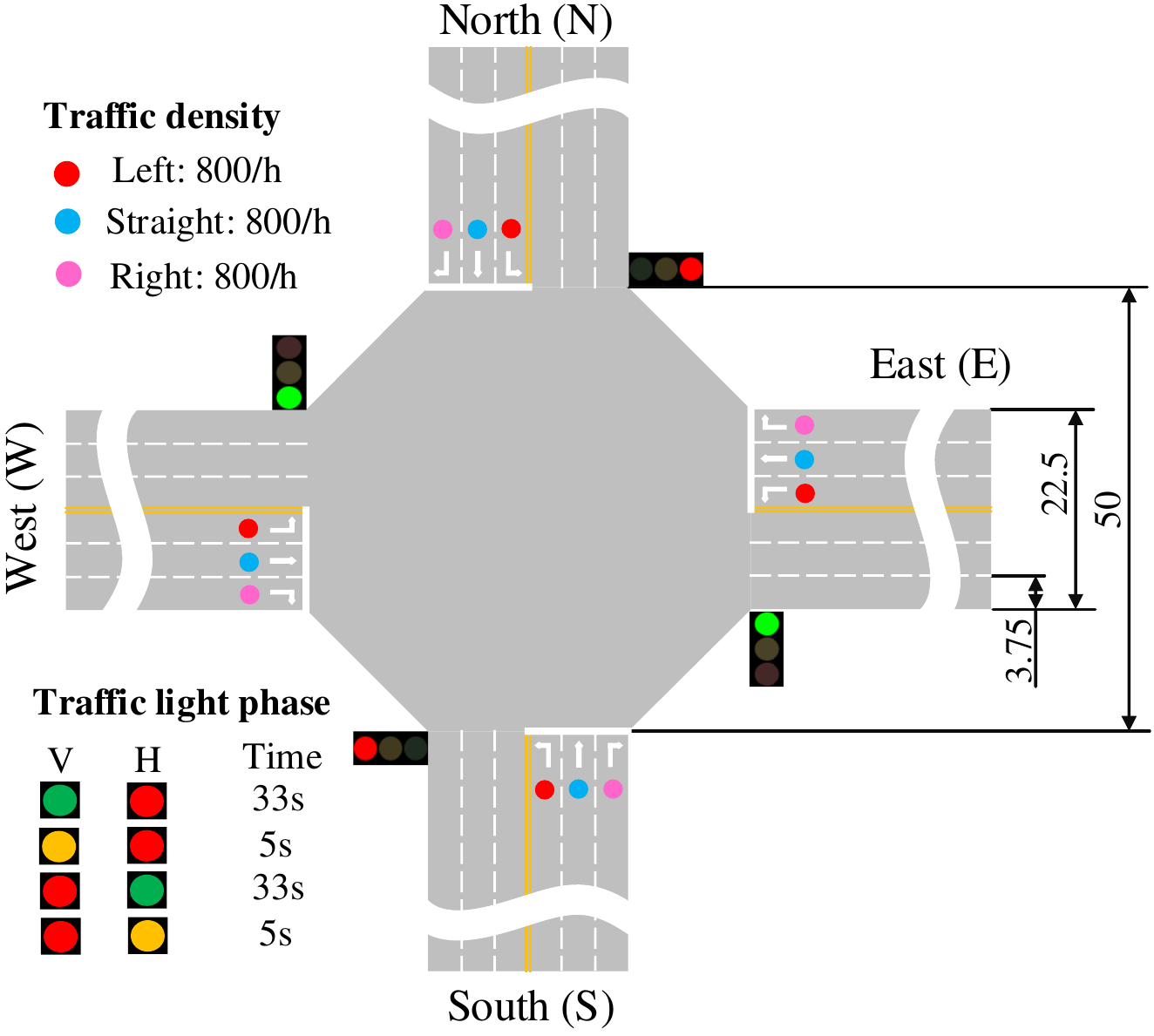}}
\caption{The scenario used for experiment verification.}
\label{fig.scenario}
\end{figure}

\subsection{Implementation of our algorithm}
\subsubsection{Path planning}
In this paper, we adopt the static path planning method to generate multiple candidate paths. Specially, in our scenario, each task is assigned three paths according to the lane number of exits. The paths are generated by the cubic Bezier curve featured by four key points. The expected velocity is chosen as a fixed value for simplicity.

\subsubsection{State and utility function}\label{sec.state_action_utility}
As mentioned in section \ref{sec.multitask}, the state should be designed to include information of the ego vehicle, the surrounding vehicles and the reference path, i.e.,
\begin{equation}\label{eq.state_design}
s_t = [s^{\rm ego}_t, s^{\rm other}_t, s^{\rm ref}_t]
\end{equation}
where $s^{\rm ego}_t=x_t$ is the ego dynamics defined in section \ref{sec.problem_formulation}, $s^{\rm other}_t$ is the concatenation of the interested surrounding vehicles. Take the turn left task as an example, it is defined as:
\begin{equation}\label{eq.other_vector}
\begin{aligned}
s_t^{\rm other} &= [p^j_x, p^j_y, \phi^j, v^j_{\rm lon}]_{t, j\in I_{\rm left}}\\
I_{\rm left} &= [{\rm SW1},{\rm SW2},{\rm SN1},{\rm SN2},{\rm NS1},{\rm NS2},{\rm NW1},{\rm NW2}]
\end{aligned}
\end{equation}
where $I_{\rm left}$ is an ordered list of vehicles that have potential conflicts with the ego. They are encoded by their respect route start and end, as well as the order on that. Correspondingly, one can define the go straight and turn right task in a similar way. The information of the reference $s^{\rm ref}_t$, however, is designed in an implicit way by the tracking errors with respect to the position, the heading angle, and the velocity:
\begin{equation}\label{eq.ref_vector}
s^{\rm ref}_t = [\delta_p, \delta_{\phi}, \delta_{v}]_t
\end{equation}
where $\delta_p$ is the position error, $|\delta_p|=\sqrt{(p_x-p^{\rm ref}_x)^2+(p_y-p^{\rm ref}_y)^2}$, ${\rm sign}(\delta_p)$ is positive if the ego is on the left side of the reference path, or else is negative. $\delta_{\phi} = \phi-\phi^{\rm ref}$ is the error of heading angle, and $\delta_v=v_{\rm lon}-v^{\rm ref}_{\rm lon}$ is the velocity error.
The overall state design is illustrated in Fig. \ref{fig.statedesign}.
\begin{figure}[htbp]
\centerline{\includegraphics[width=0.7\linewidth]{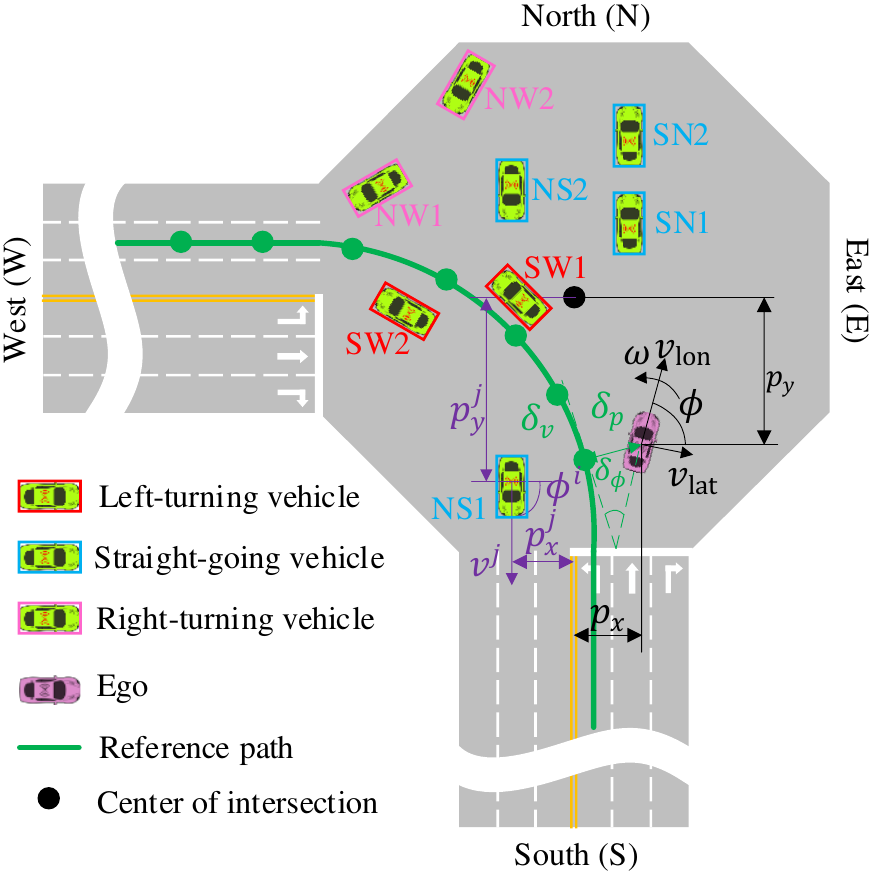}}
\caption{State design in the scenario.}
\label{fig.statedesign}
\end{figure}
The weighting matrices in the utility function are designed as $Q=diag(0.04,0.04,0.01,0.01,0.1,0.02)$, $R=diag(0.1, 0.005)$. The predictive horizon $T$ is set to be 25, which is 2.5s in practical. 

\subsubsection{Constraint construction}
Slightly different from the one in \eqref{eq.lower_layer_ocp}, we further refine the constraint in a way that represents the ego vehicle and each of the surrounding vehicles by two circles as illustrated by Fig. \ref{fig.constraint}, where $r_{\rm veh}$ and $r_{\rm ego}$ are radii of circles of a vehicle and the ego. Then, in each time step, we impose four constraints for each vehicle between each center of the ego circle and that of the other vehicle rather than between only their CGs. Similar are the constraints between the ego and the road edge. Parameters for constraints: $M=diag(1,1,0,0,0,0)$, $D^{\rm safe}_{\rm veh}=r_{\rm veh}+r_{\rm ego}$, $D^{\rm safe}_{\rm road}=r_{\rm ego}$, where $r_{\rm ego}=r_{\rm veh}=2.5$m. For the traffic light constraint, we convert it to constraints between ego and vehicles by placing two virtual vehicles on the stop line, as shown in Fig. \ref{fig.constraint}. 

\subsubsection{Vehicle dynamics and prediction model}
$F_{\rm ego}$ has been shown in \eqref{eq.variable_definition2}, where all the vehicles parameters are displayed in Table \ref{tab.vehicle_parameters}.
Moreover, according to the type and position of the vehicle $j, j\in I$, the turning rate $\omega^j_{\rm pred}$ in the prediction model is determined, as shown in Table \ref{tab.pred_parameters}. 
\begin{figure}[htbp]
\centerline{\includegraphics[width=0.6\linewidth]{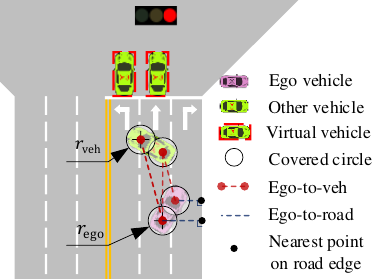}}
\caption{Design of the state constraints.}
\label{fig.constraint}
\end{figure}

\begin{table}
\centering
\caption{Parameters for $F_{\rm ego}$}
\label{tab.vehicle_parameters}
\begin{tabular}{clr}
\toprule
Parameter & Meaning & Value \\
\midrule
$k_f$ & Front wheel cornering stiffness & -155495 [N/rad]  \\
$k_r$ & Rear wheel cornering stiffness & -155495 [N/rad] \\
$L_f$ & Distance from CG to front axle & 1.19 [m] \\
$L_r$ & Distance from CG to rear axle & 1.46 [m] \\
$m$ & Mass & 1520 [kg] \\
$I_z$ & Polar moment of inertia at CG & 2640 [kg$\cdot\mathrm{m}^2$] \\
$\Delta t$& System frequency & 0.1 [s]\\ 
\bottomrule
\end{tabular}
\end{table}

\begin{table}
\centering
\caption{Parameters for $F_{\rm pred}$}
\label{tab.pred_parameters}
\begin{tabular}{cc|c|c}
\hline
\multirow{2}*{$\omega^j_{\rm pred}$} &  & \multicolumn{2}{c}{Position reletive to the intersection}\\
\cline{3-4}                          & & Out of & Within \\
\hline
\multirow{3}*{Vehicle type}  & Left-turning & 0 & $v^j_{\rm lon}/26.875$\\
\cline{3-4}                 & Straight-going & 0 & 0  \\
\cline{3-4}                 & Right-turning & 0 & $-v^j_{\rm lon}/15.625$ \\
\hline
\end{tabular}
\end{table}

\subsubsection{Training settings}
We implement the offline training Algorithm \ref{alg.offline} in an asynchronous learning architecture proposed in \cite{guan2021mixed}. For value function and policy, we use a multi-layer perceptron (MLP) with 2 hidden layers, consisting of 256 units per layer, with Exponential Linear Units (ELU) each layer \cite{clevert2015fast}. The Adam method \cite{kingma2014adam} with a polynomial decay learning rate is used to update all the parameters. Specific hyperparameter settings are listed in Table \ref{table.hyper}. We train 5 different runs with different random seeds on a single computer with a 2.4 GHz 50 core Inter Xeon CPU, with evaluations every 100 iterations.
\begin{table}[!htp] 
\centering
\caption{Detailed hyperparameters.}
\label{table.hyper}
\begin{threeparttable}[h]
\begin{tabular}{lc}
\toprule
Hyperparameters & Value \\
\hline
\quad Optimizer &  Adam ($\beta_{1}=0.9, \beta_{2}=0.999$)\\
\quad Approximation function  & MLP \\
\quad Number of hidden layers & 2\\
\quad Number of hidden units & 256\\
\quad Nonlinearity of hidden layer& ELU\\
\quad Replay buffer size & 5e5\\
\quad Batch size & 1024\\
\quad Policy learning rate & Linear decay 3e-4 $\rightarrow$ 1e-5 \\
\quad Value learning rate & Linear decay 8e-4 $\rightarrow$ 1e-5\\
\quad Penalty amplifier $c$ & 1.1\\
\quad Total iteration & 200000 \\
\quad Update interval $m$ & 10000 \\
\quad Safety shield $n_{ss}$ & 5 \\
\quad Number of Actors &4\\
\quad Number of Buffers &4\\
\quad Number of Learners &30\\
\bottomrule
\end{tabular}
\end{threeparttable}
\end{table}

\subsection{Simulation results}
Followed the Algorithm \ref{alg.multipath}, the planned paths are shown in Fig. \ref{fig.multipath}. We also demonstrate the tracking and safety performances of different tasks during the training process in Fig.~\ref{fig.simu_results}, indicated by $J_{\rm actor}$ and $J_{\rm penalty}$ respectively, and the value loss $J_{\rm critic}$ to exhibit the performance of the value function. 
\begin{figure}[htbp]
\centering
\captionsetup[subfigure]{justification=centering}
\subfloat[]{\label{fig.multipath}\includegraphics[width=0.24\textwidth]{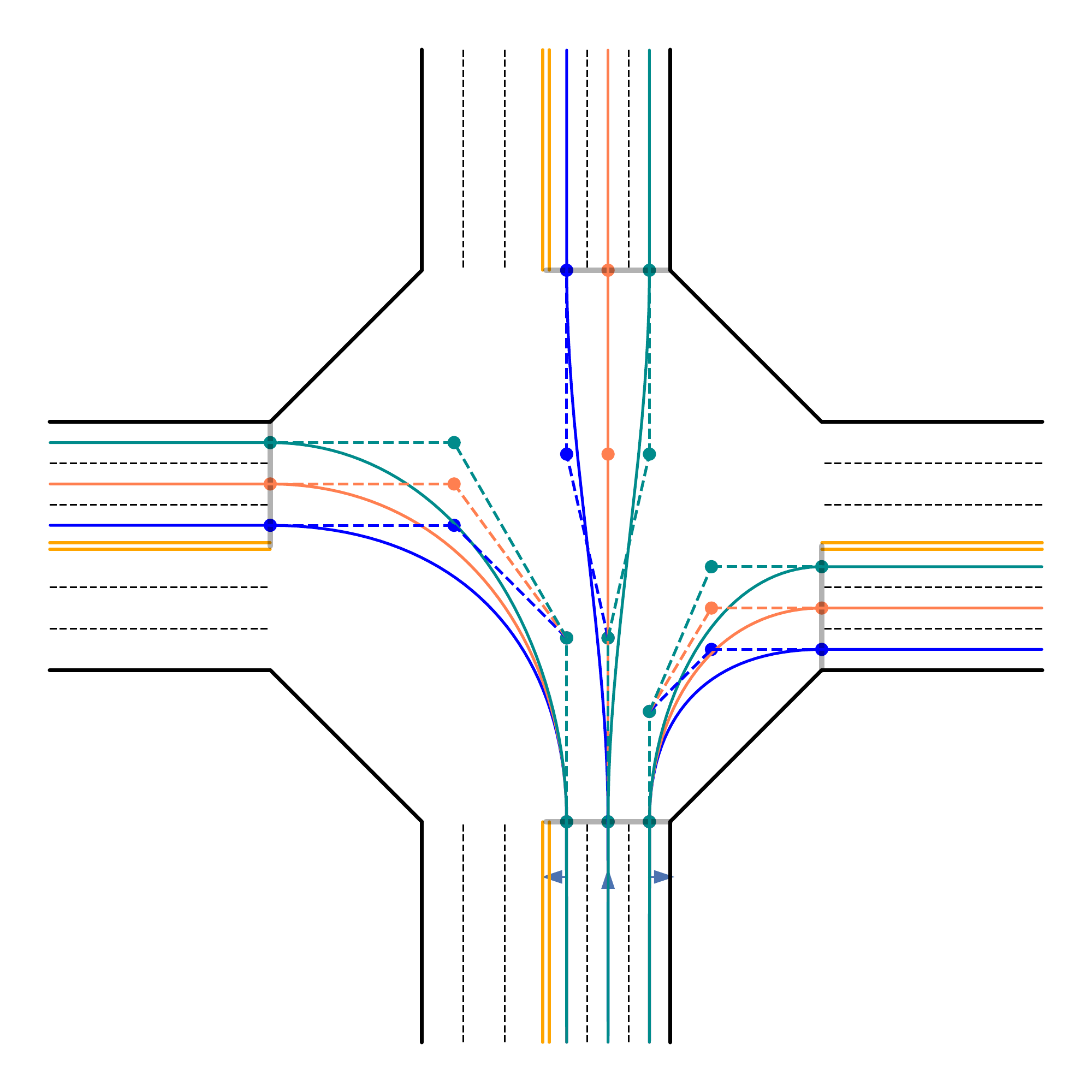}}
\subfloat[]{\label{fig.offline_tracking}\includegraphics[width=0.24\textwidth]{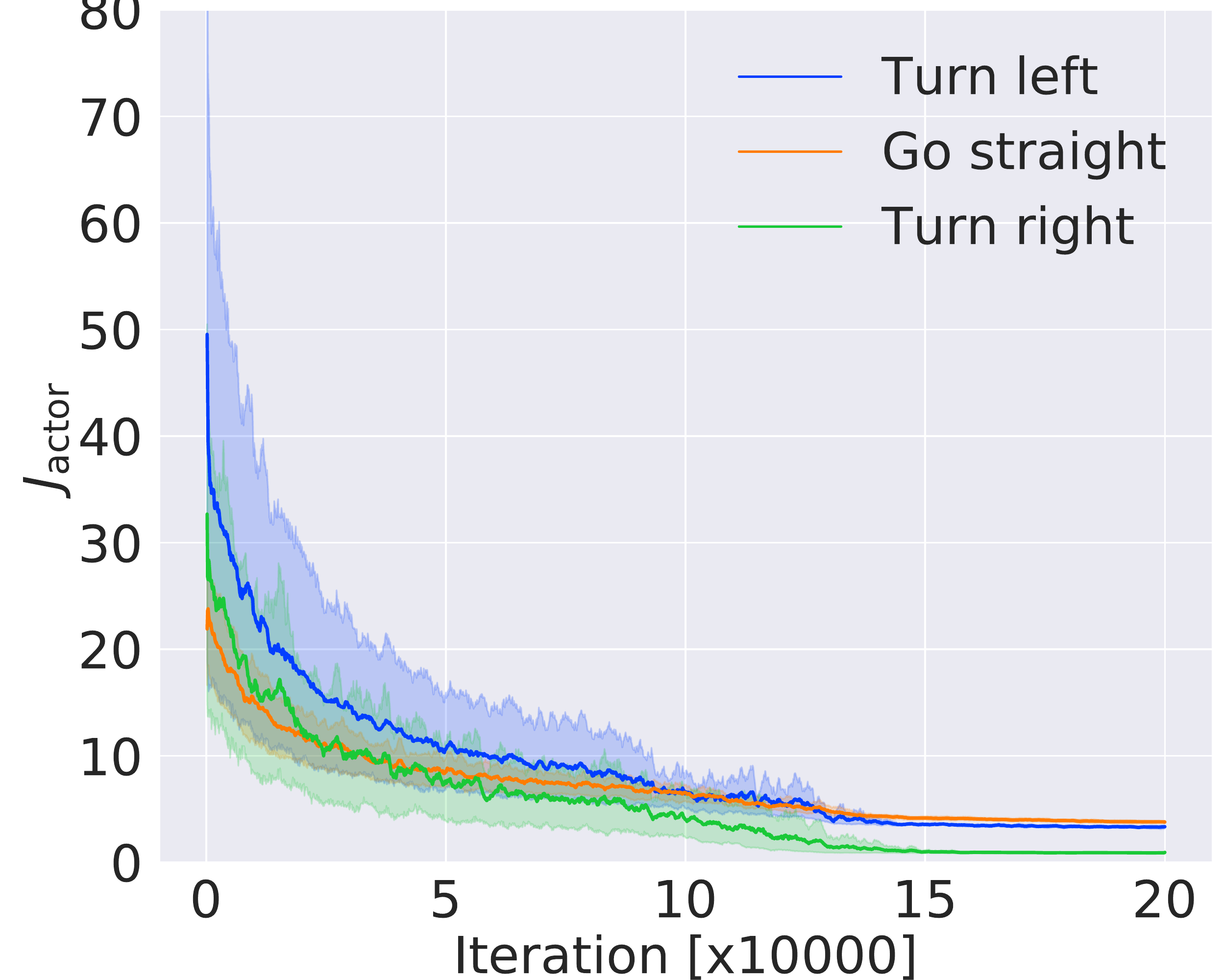}}\\
\subfloat[]{\label{fig.offline_safety}\includegraphics[width=0.24\textwidth]{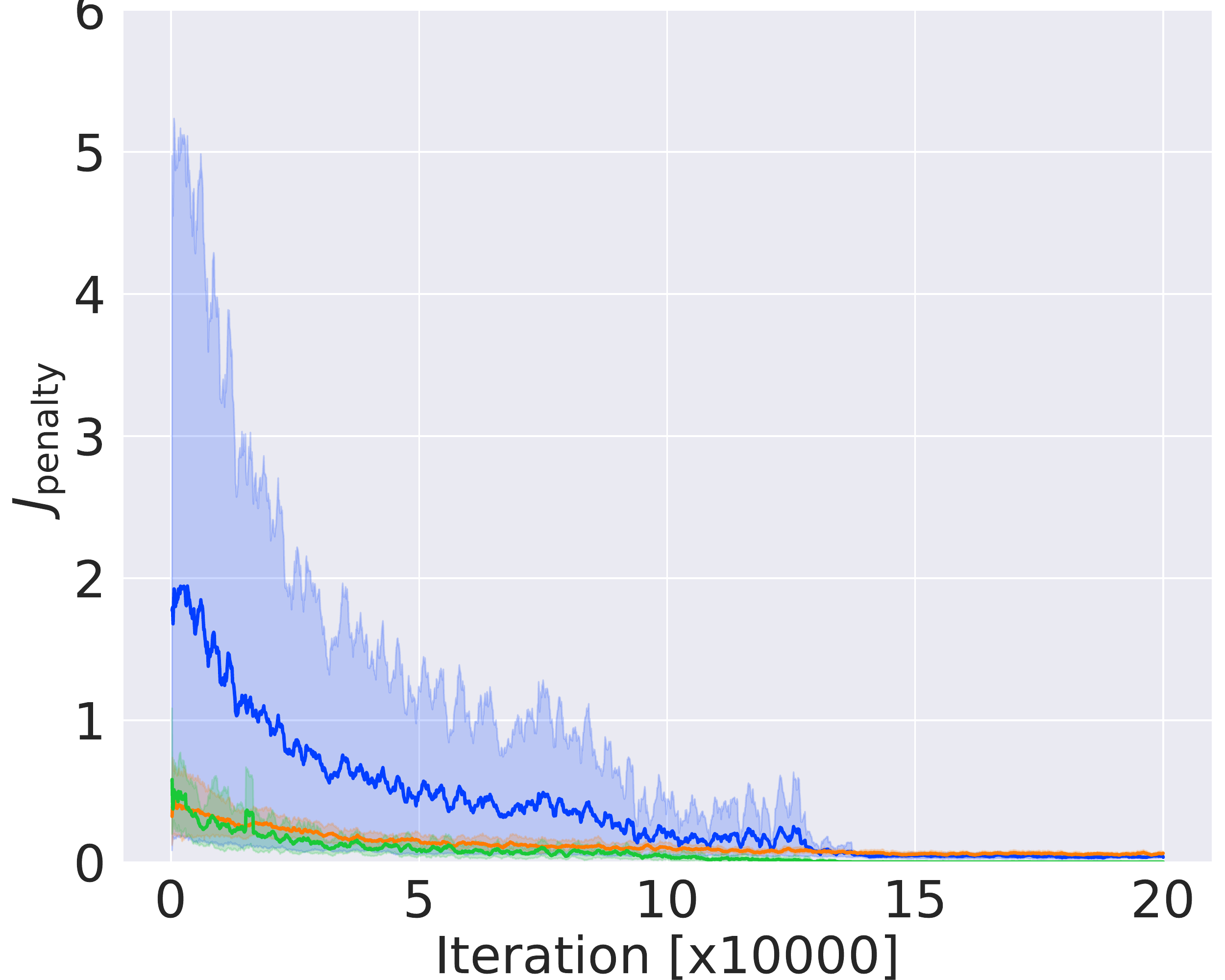}}
\subfloat[]{\label{fig.offline_lossv}\includegraphics[width=0.24\textwidth]{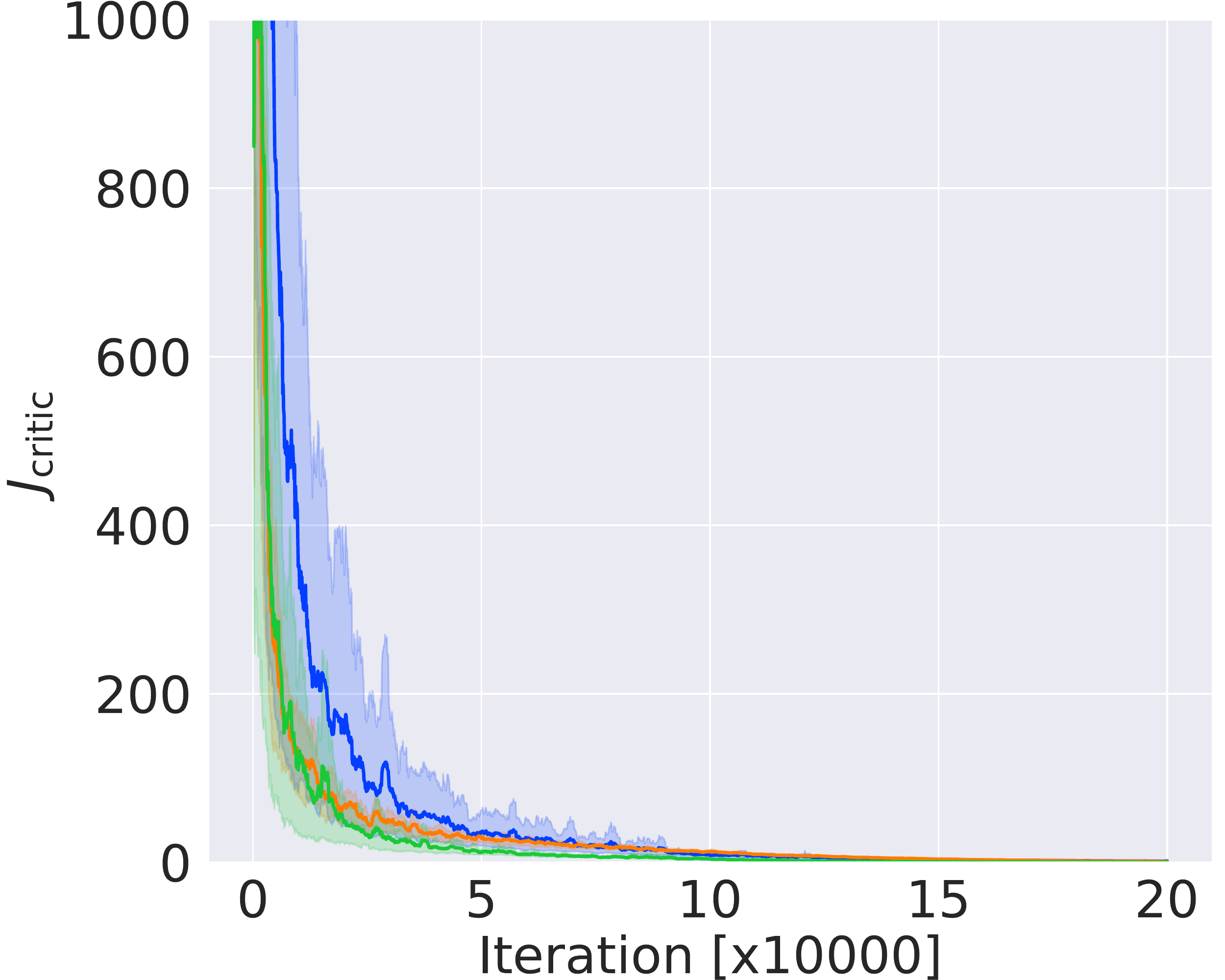}}
\caption{Results of static path planning and dynamic optimal tracking. (a) Planned paths for each task. (b) Tracking performance during training process. (c) Safety performance during training process. (d) Value loss during training process. For (b)-(d), The solid lines correspond to the mean and the shaded regions correspond to 95\% confidence interval over 5 runs.}
\label{fig.simu_results}
\end{figure}
Along the training process, the policy loss, the penalty and the value loss decrease consistently for all the tasks, indicating an improving tracking and safety performance. Specially, the penalty and value loss decrease to zero approximately, proving the effectiveness of the proposed RL-based solver for constrained OCPs. In addition, the convergence speed, variance across different seeds, and the final performance vary with tasks. That is because the surrounding vehicles that have potential conflicts with the ego are different across tasks, leading to differences in task difficulty.

In addition, we apply the trained policy of the left-turning task in the environment to visualize one typical case where a dense traffic and different phase of traffic light are included. Note that different from the training, we carry out all the simulation tests on a 2.90GHz Intel Core i9-8950HK CPU. As shown in Fig. \ref{fig.simulation_single_demo}, when the traffic light is red, the ego pulls up to avoid collision and to obey the rule. Then when the light turns green, the ego starts itself off to enter the junction, where it meets several straight-going vehicles from the opposite direction. Therefore, the ego chooses the upper path and slows down to try to bypass the first one. Notice that the safety shield works here to avoid potential collisions. After that, it speeds up to follow the middle path, the one with the lowest value in that case, with which it can go first and meanwhile avoid the right-turning vehicles so that the velocity tracking error can be largely reduced. The computing time in each step is under 10ms, making our method considerably fast to be applied in the real world.

\begin{figure*}[htbp]
\centering
\captionsetup[subfigure]{justification=centering}
\subfloat[t=0.3s]{\includegraphics[width=0.2\linewidth]{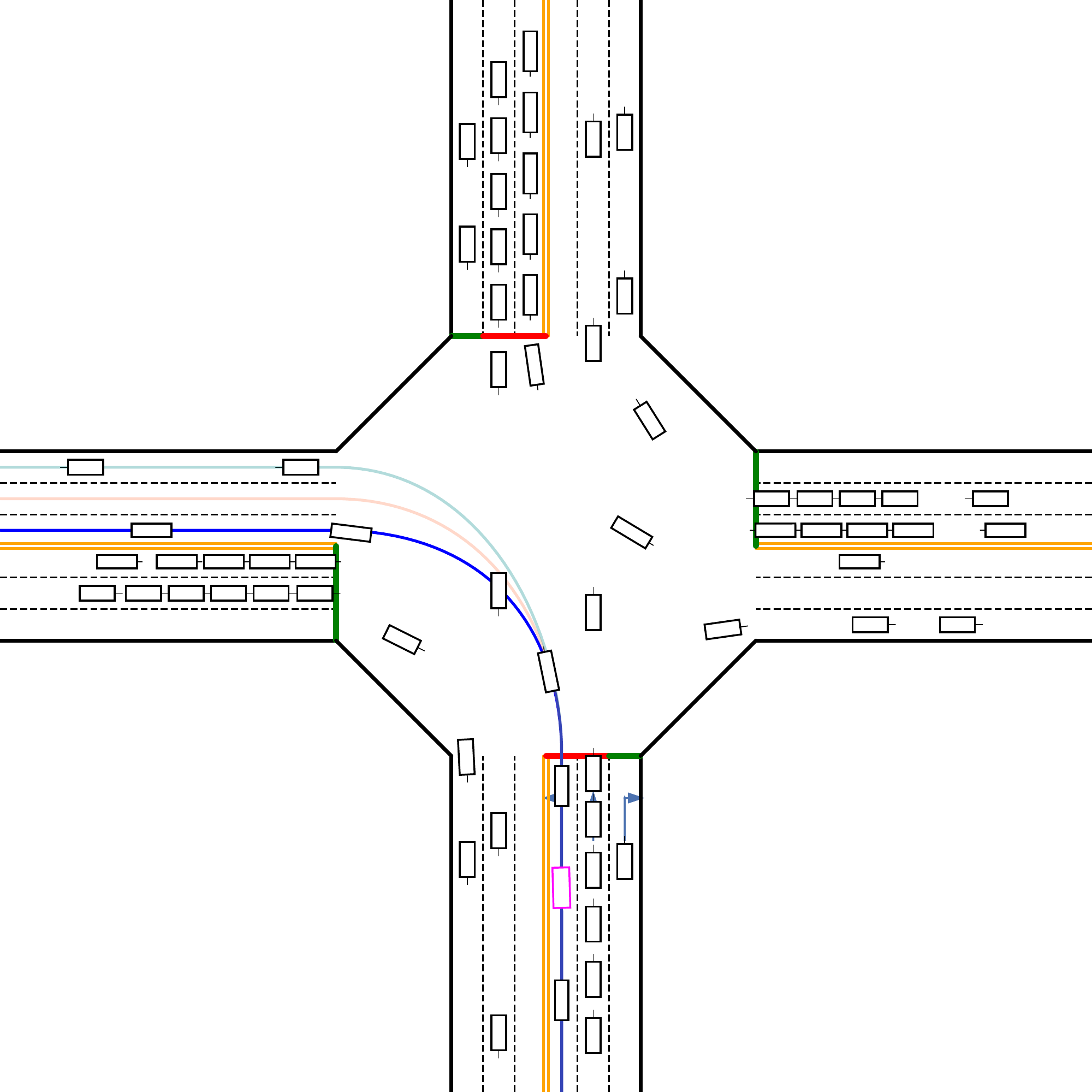}}
\subfloat[t=6.5s]{\includegraphics[width=0.2\linewidth]{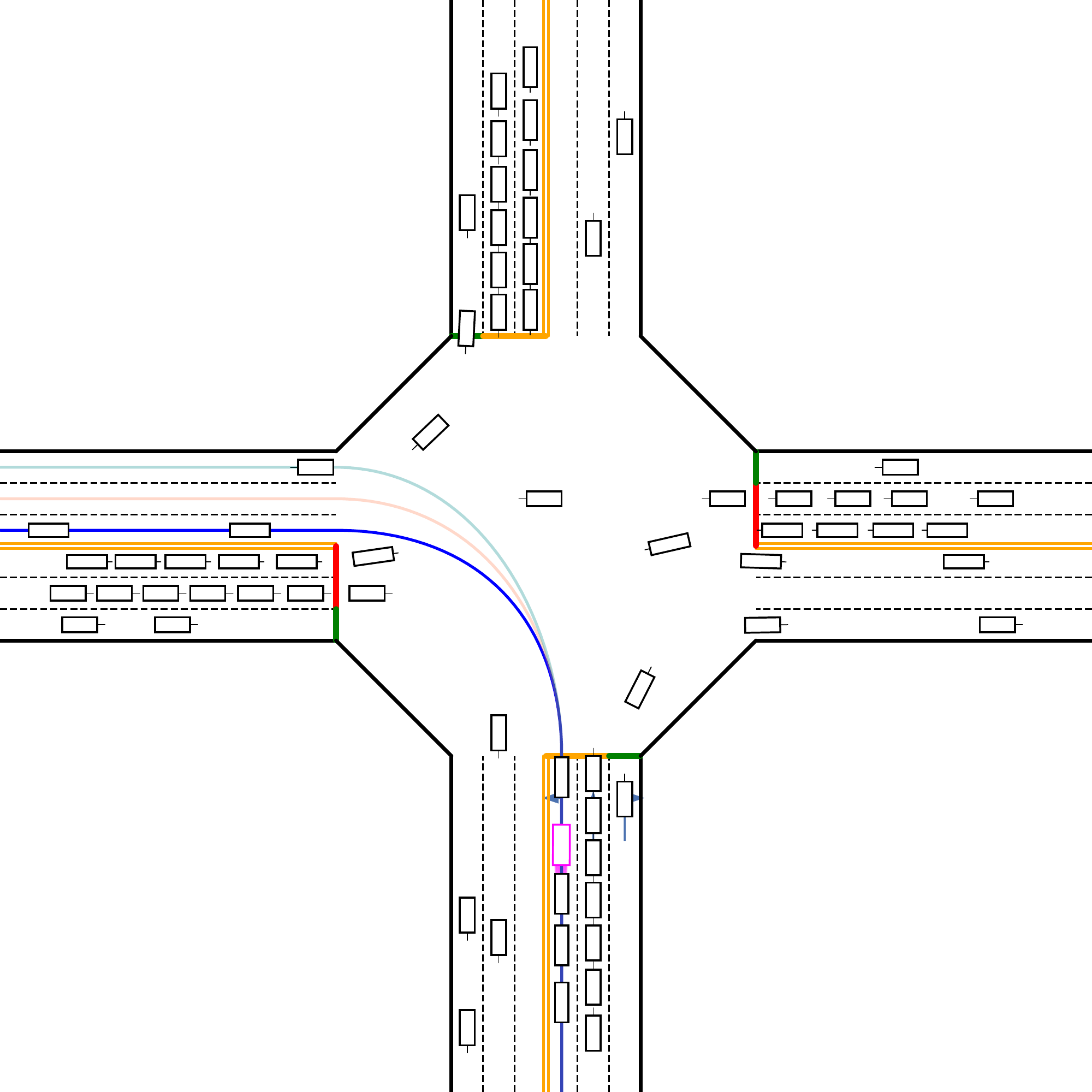}}
\subfloat[t=12.8s]{\includegraphics[width=0.2\linewidth]{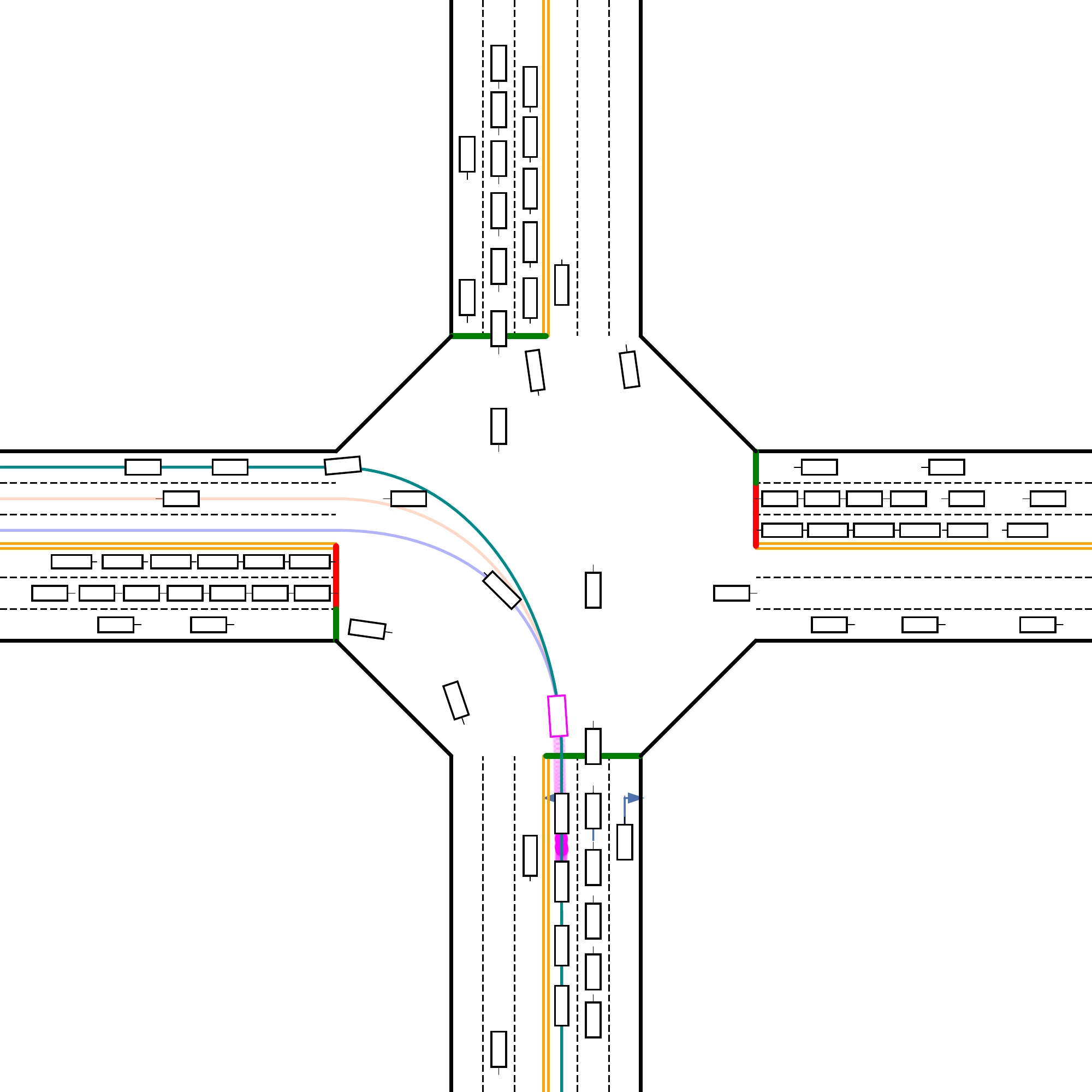}}
\subfloat[t=19.1s]{\includegraphics[width=0.2\linewidth]{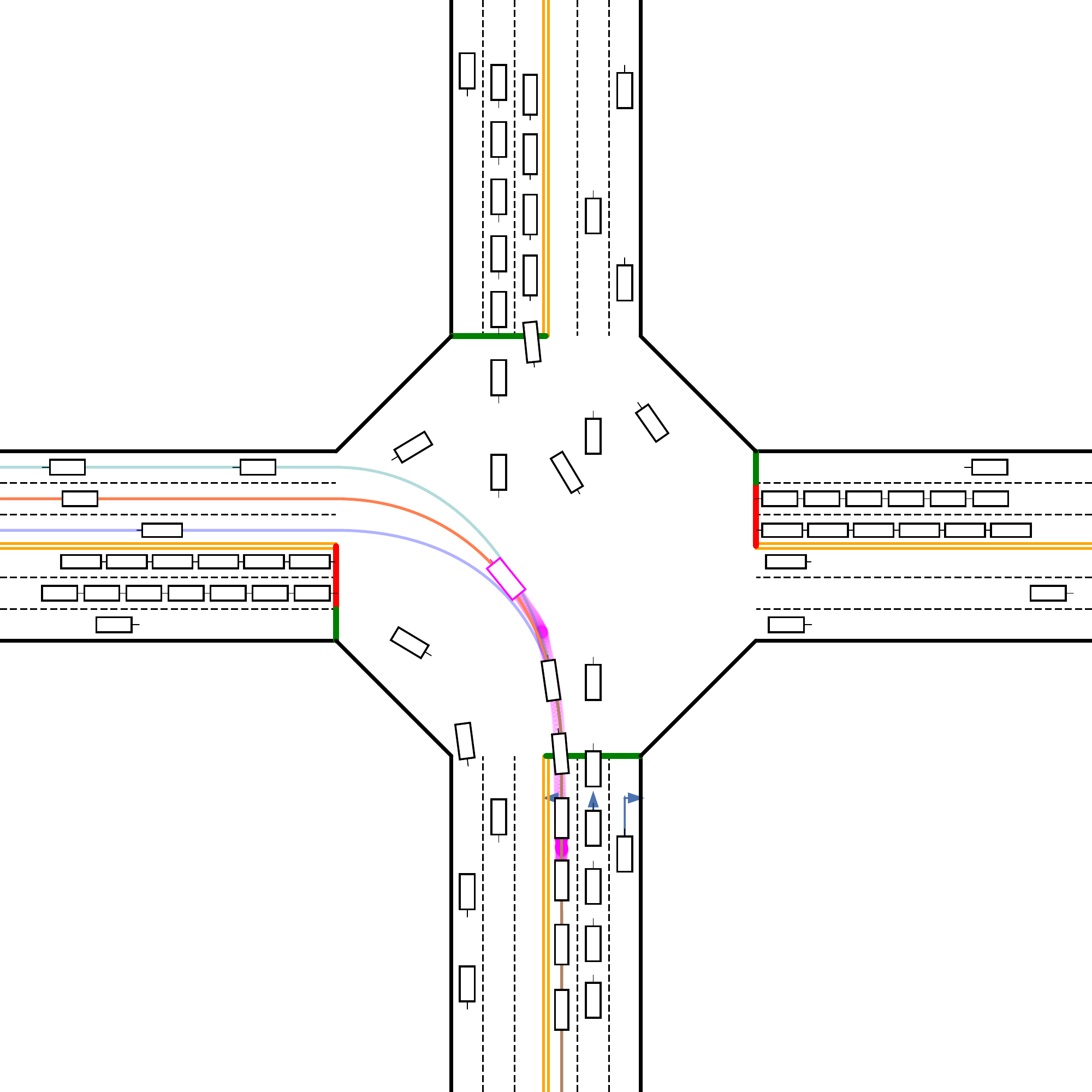}}
\subfloat[t=23.3s]{\includegraphics[width=0.2\linewidth]{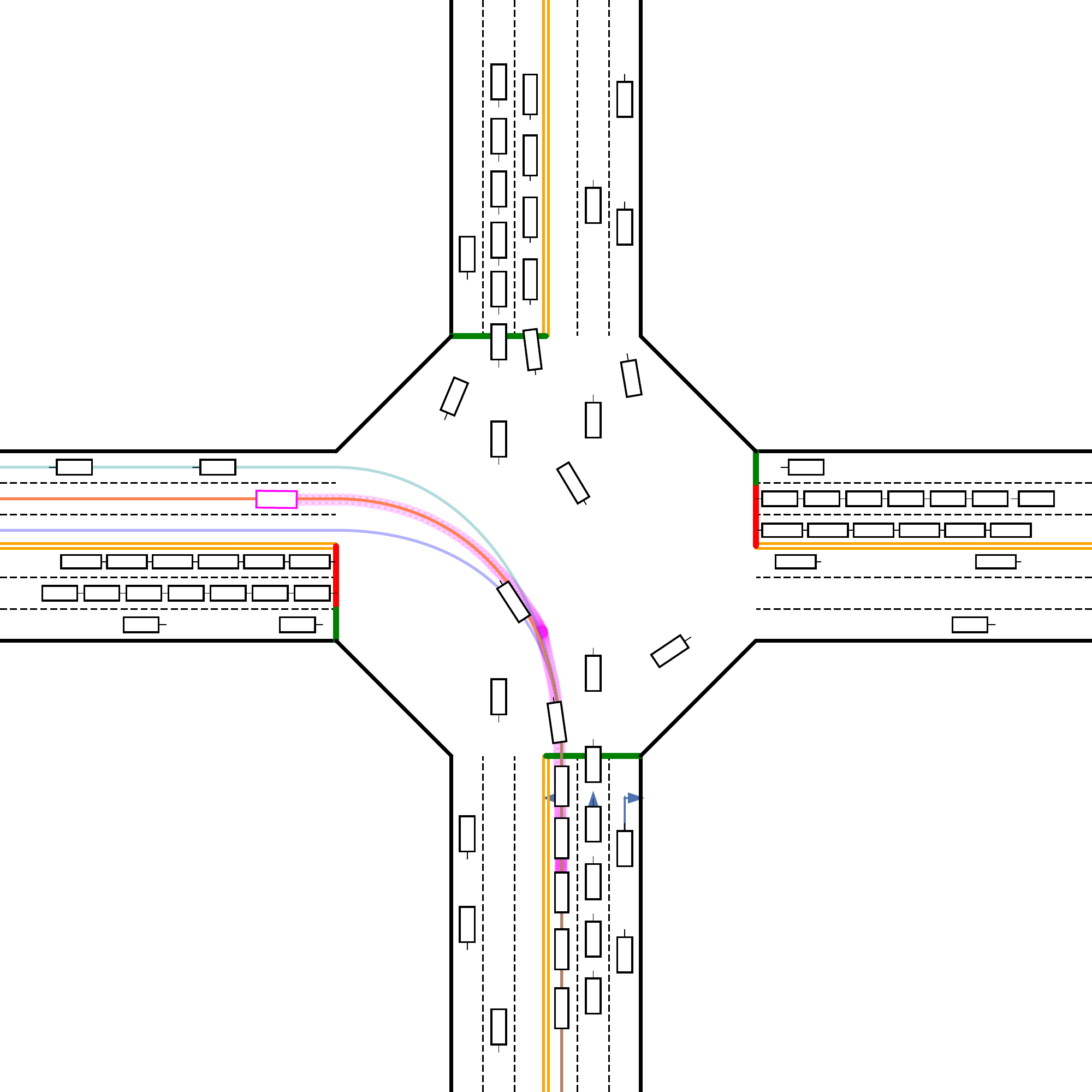}}\\
\subfloat[Speed]{\includegraphics[width=0.2\linewidth]{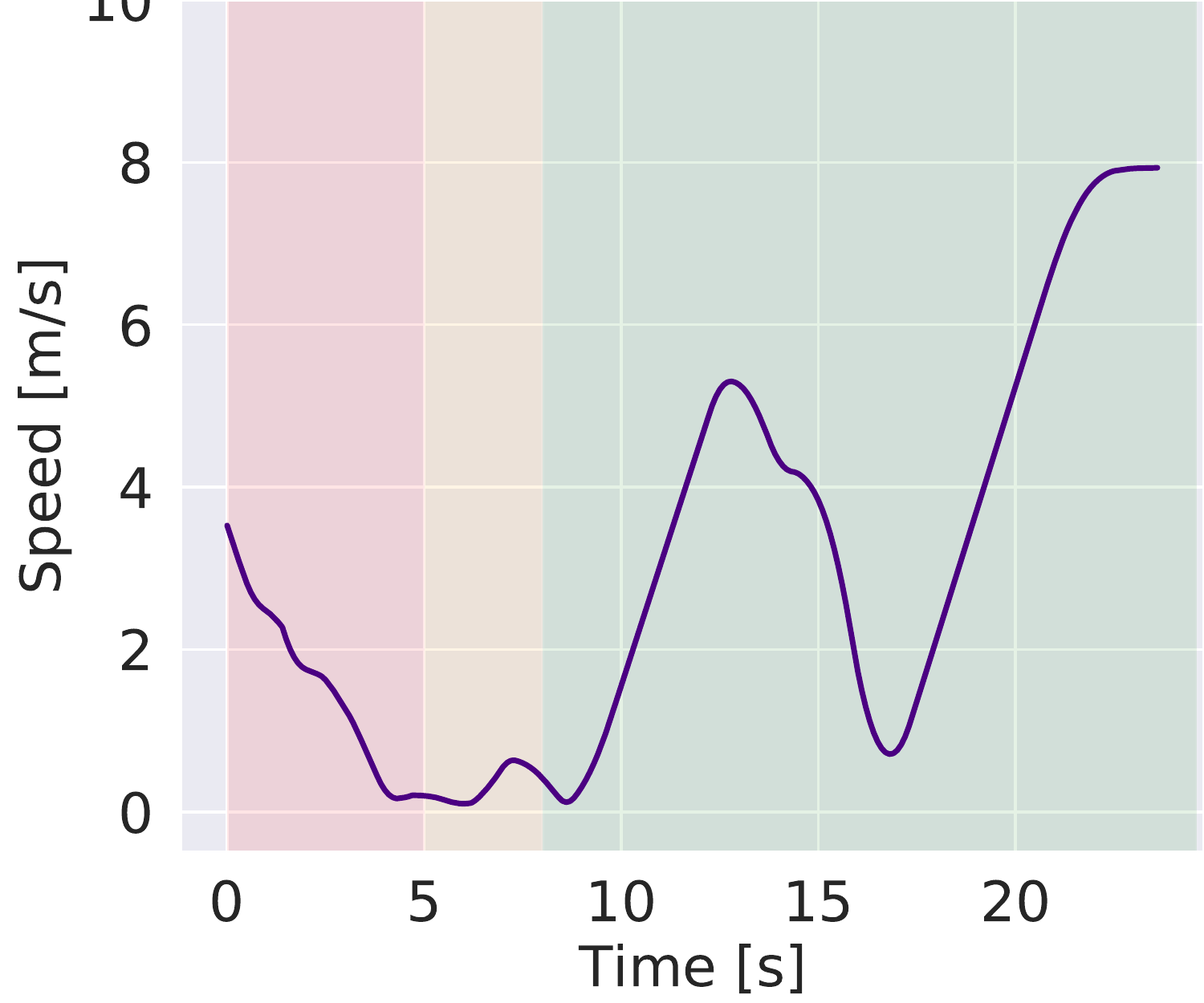}}
\subfloat[Ref index]{\includegraphics[width=0.2\linewidth]{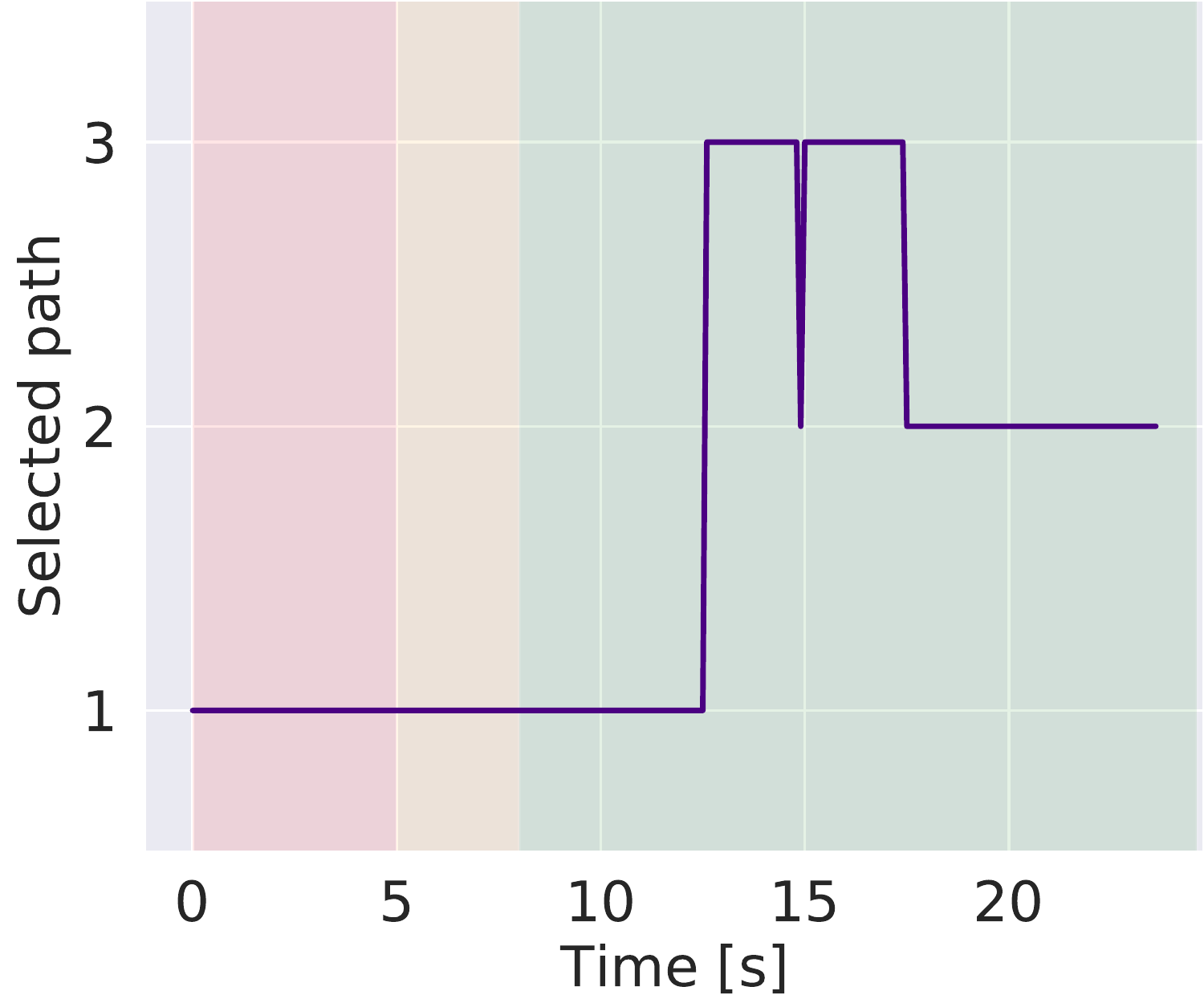}}
\subfloat[Safety shield]{\includegraphics[width=0.2\linewidth]{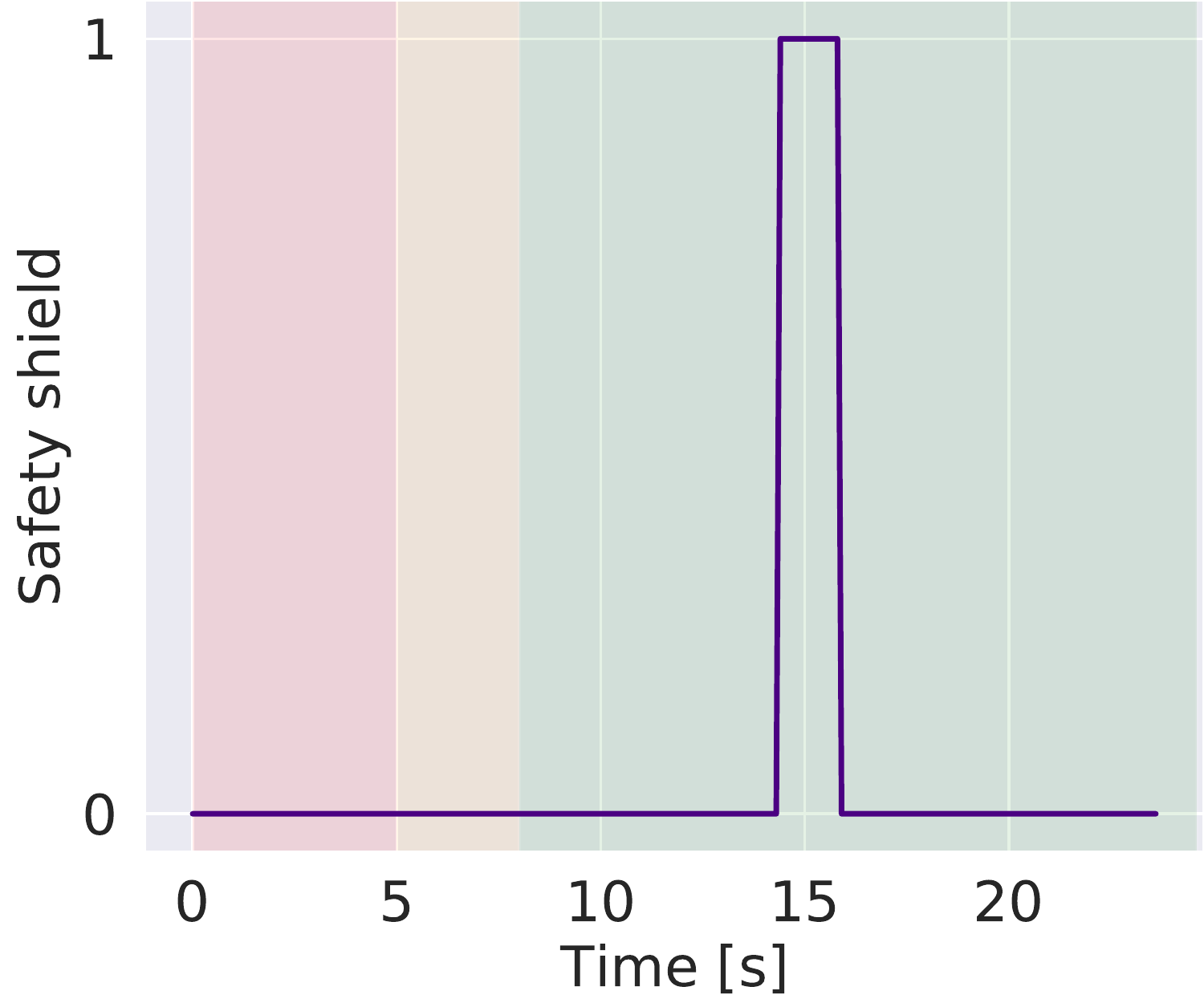}}
\subfloat[Steering angle]{\includegraphics[width=0.2\linewidth]{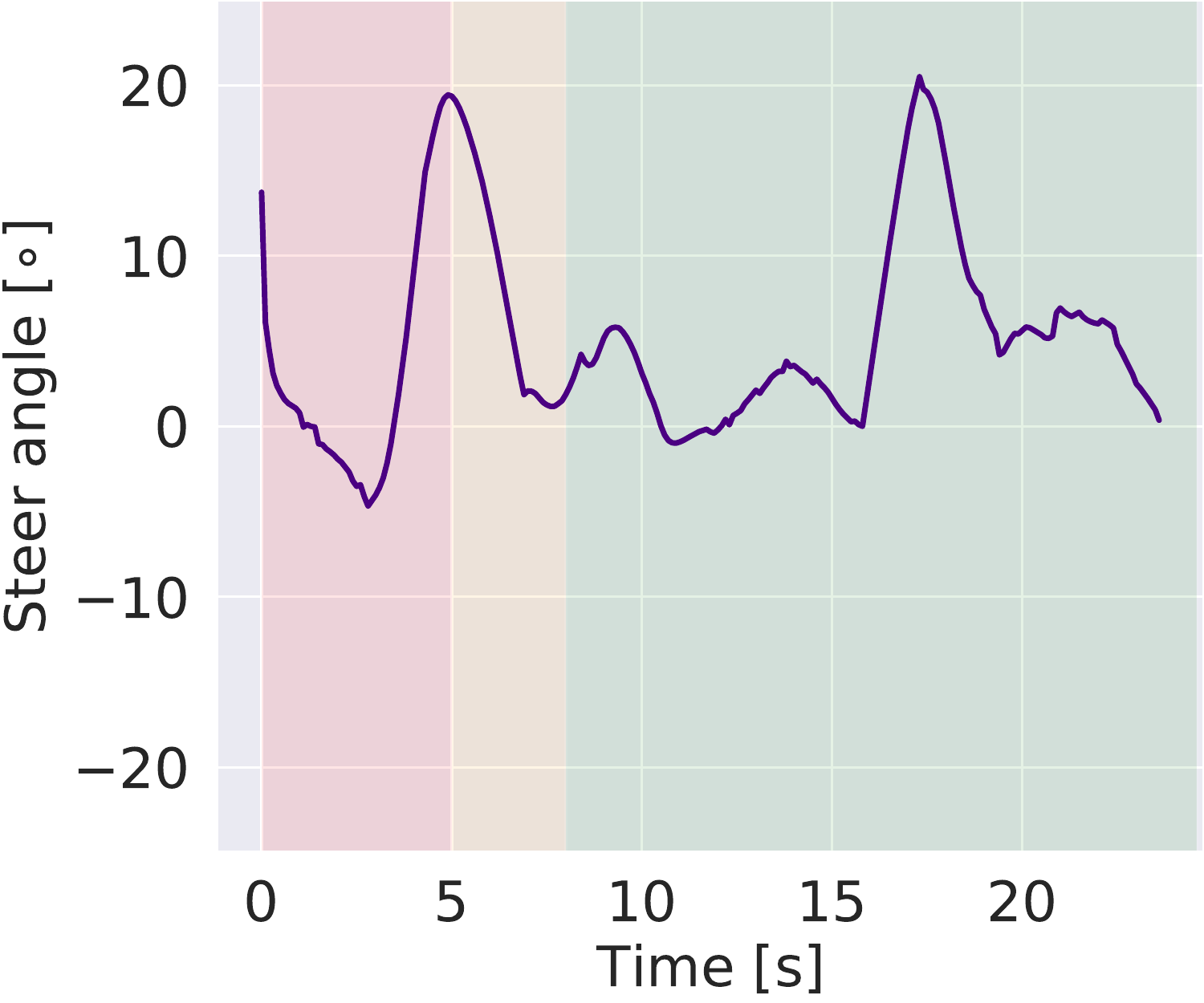}}
\subfloat[Acceleration]{\includegraphics[width=0.2\linewidth]{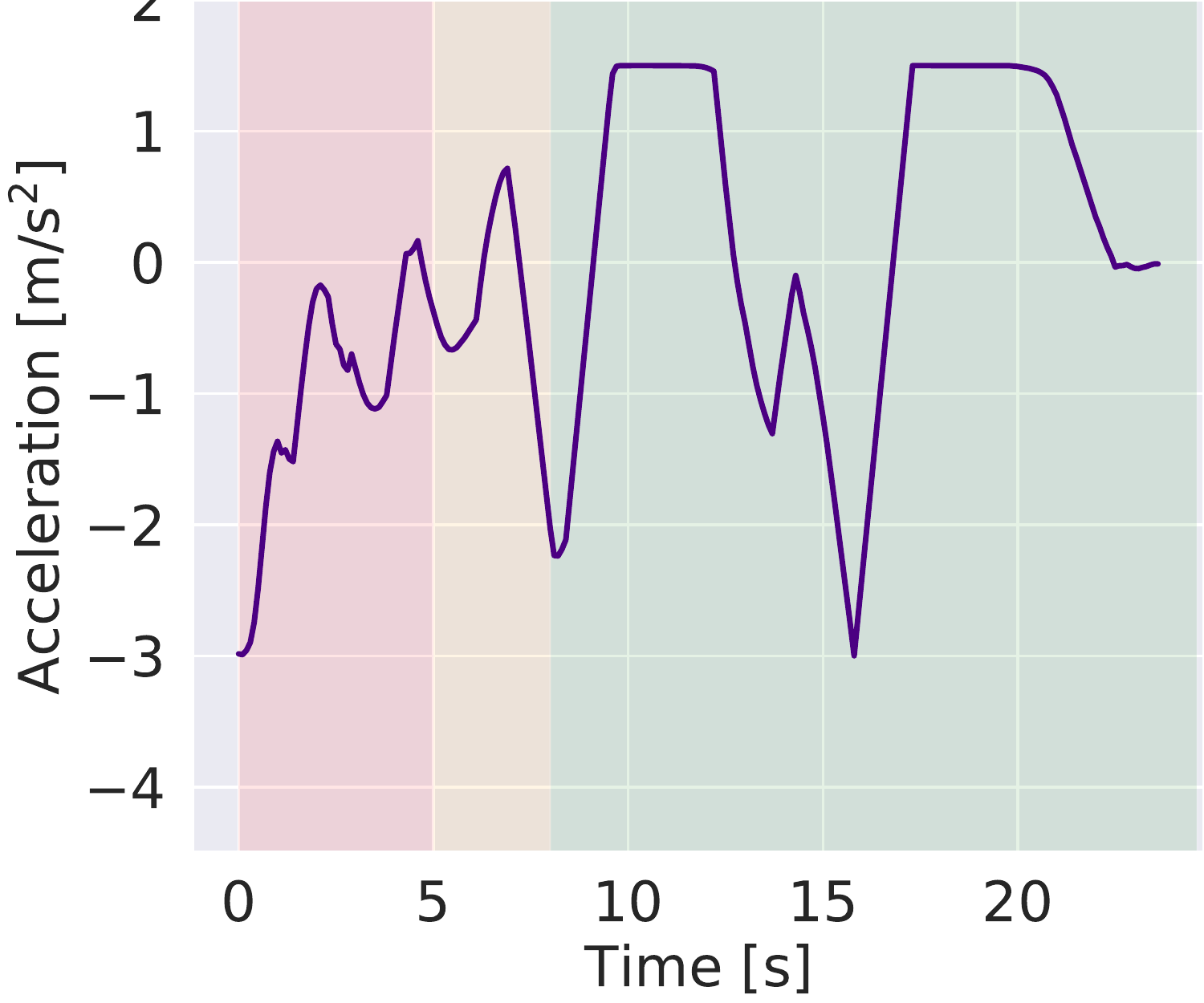}}\\
\caption{Visualization of one typical episode driven by the trained policy.}
\label{fig.simulation_single_demo}
\end{figure*}

\subsection{Experiment 1: Performance of GEP}
To verify the control precision and computing efficiency of our model-based solver on constrained optimal problem, we conduct comparison with the classic Model Predictive Control (MPC), which utilizes the receding horizon optimization online and can deal with constraints explicitly.
Formally, the problem \eqref{eq.lower_layer_ocp} is defined on one certain path, and thus MPC method solves the same number of problems as that of candidate paths and calculate cost function of each path respectively. Then optimal path is selected by the minimum cost function and its corresponding action is used as the input signal of ego vehicle. Here we adopt the Ipopt solver to obtain the exact solution of the constrained optimal control problem, which is an open and popular package to solve nonlinear optimization problem.
Fig.~\ref{fig.exp1} demonstrates the comparison of our algorithm on control effect and computation time.
Results show that the optimal path of two methods are identical and the output actions, steer wheel and acceleration, have similar trends, which indicates our proposed algorithm can approximate the solution of MPC with a small error. However, there exists the obvious difference in computation time that our method can output the actions within 10ms while MPC will take 1000ms to perform that. Although MPC can find the optimal solution by its online optimization, its computation time also increases sharply with the number of constraints, probably violating the real-time requirements of autonomous driving.

\begin{figure}[htbp]
\centering
\captionsetup[subfigure]{justification=centering}
\subfloat[]{\label{fig.exp1.front}\includegraphics[width=0.22\textwidth]{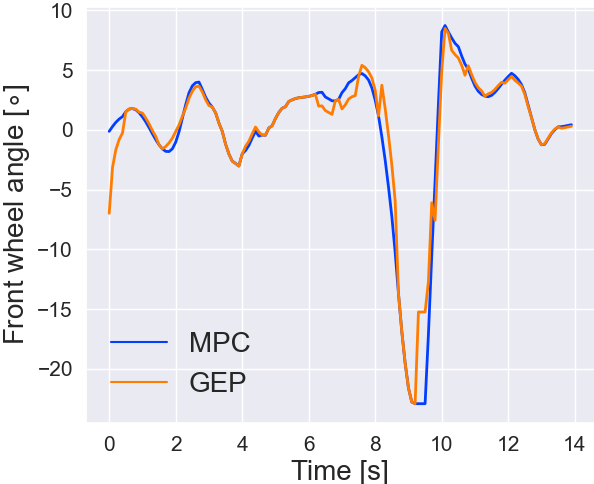}}
\subfloat[]{\label{fig.exp1.acc}\includegraphics[width=0.225\textwidth]{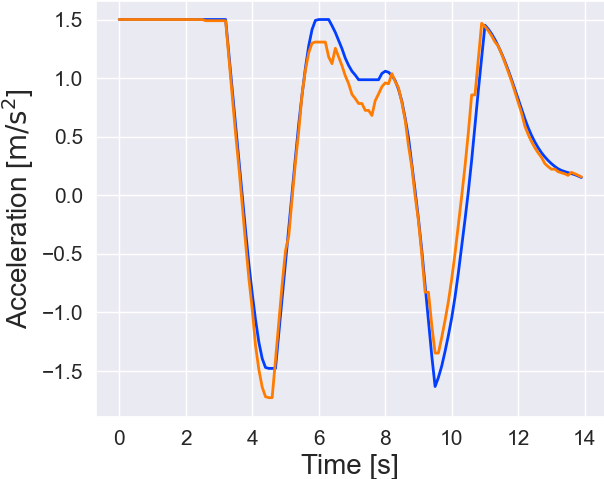}}\\
\subfloat[]{\label{fig.exp1.path}\includegraphics[width=0.21\textwidth]{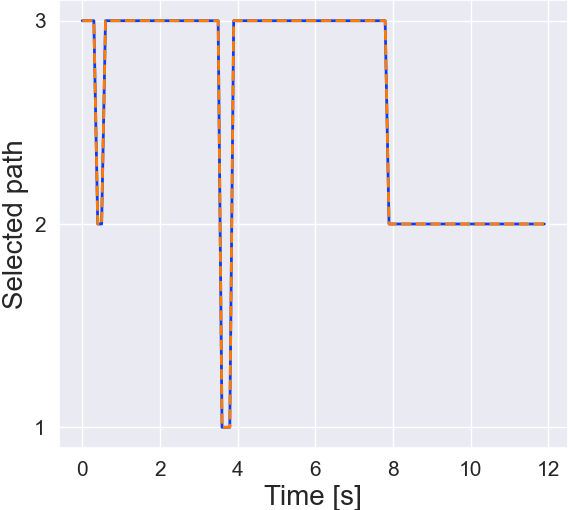}}
\subfloat[]{\label{fig.exp1.time}\includegraphics[width=0.23\textwidth]{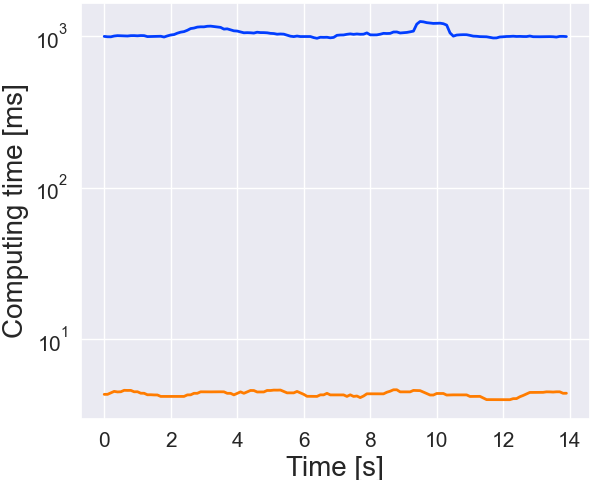}}

\caption{Comparison with MPC. (a) Front wheel angle. (b) Acceleration. (c) Optimal path. (d) Computing time}
\label{fig.exp1}
\end{figure}

\subsection{Experiment 2: Comparison of driving performance}
We compare our method with a rule-based method which adopts A* algorithm to generate a feasible trajectory and a PID controller to track it\cite{xin2021enable}, as well as a model-free RL method which uses a punish and reward system to learn a policy for maximizing the long-term reward (If the ego vehicle passes the intersection safely, a large positive reward 100 is given, otherwise -100 is given wherever a collision happens)\cite{haarnoja2018soft}. We choose six indicators including computing time, comfort, travel efficiency, collisions, failure rate and driving compliance to evaluate the three algorithms. Comfort is reflected by the mean root square of lateral and longitudinal acceleration, i.e., $I_{\text {comfort }}=1.4 \sqrt{\left( a_{x}^{2}\right)+\left(a_{y}^{2}\right)}$. Travel efficiency is evaluated by the average time used to pass the intersection. Failure rate means the accumulated times of that decision signal is generated for more than 1 seconds and driving compliance shows times of breaking red light. The results of 100 times simulation are shown in Table \ref{table.comp}. Rule-based method is more likely to stop and wait for other vehicles, leading to higher passing time but better safety performance. However, it tends to take much more time to give a control signal in the dense traffic flow and suffers the highest failure rate. Model-free RL is eager to achieve the goal, but incurs the most collisions and decision incompliance due to its lack of safety guarantee. Benefiting from the framework design, the computational efficiency of our method remains as fast as the model-free approach, and the driving performance such as safety, compliance and failure rate, is better than the other two approaches.
\begin{table}[h]
\caption{Comparison of driving performance}
\begin{tabular}{lccc}
\hline
                    &IDC (Ours)    & Rule-based    & Model-free RL \\ \hline
Computing time [ms]\\
\quad Upper-quantile & 5.81 & 73.99  & 4.91          \\
\quad Standard deviation & 0.60 & 36.59  & 0.65          \\
Comfort index       & 1.84         & 1.41          & 3.21          \\
Time to pass [s]    & 7.86($\pm$3.52) & 24.4($\pm$16.48)  & 6.73($\pm$3.32)          \\
Collisions          & 0            & 0             & 31            \\
Failure Rate        & 0            & 13            & 0             \\
Decision Compliance & 0            & 0             & 17            \\ \hline
\end{tabular}
\label{table.comp}
\end{table}

\subsection{Experiment 3: Application of distributed control}
We also apply our trained policies in multiple vehicles for distributed control. Our method yields a surprisingly good performance by showing a group of complex and intelligent driving trajectories (Fig. \ref{fig.simulation_multi_demo}), which demonstrates the potential of the proposed method to be extended to the distributed control of large-scale connected vehicles. More videos are available online (https://youtu.be/J8aRgcjJukQ).

\begin{figure*}[htbp]
\centering
\captionsetup[subfigure]{justification=centering}
\subfloat[t=0.3s]{\includegraphics[width=0.2\linewidth]{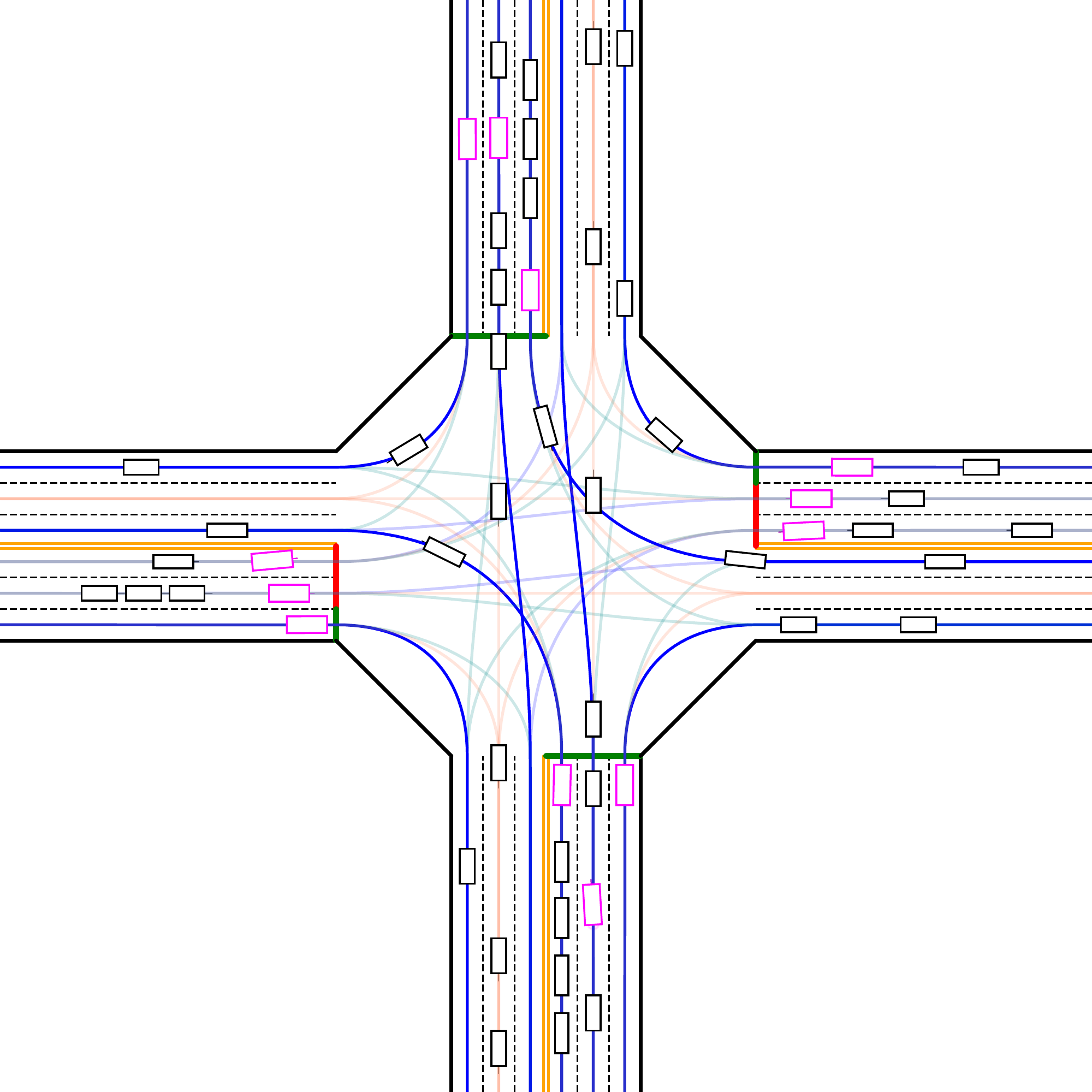}}
\subfloat[t=2.8s]{\includegraphics[width=0.2\linewidth]{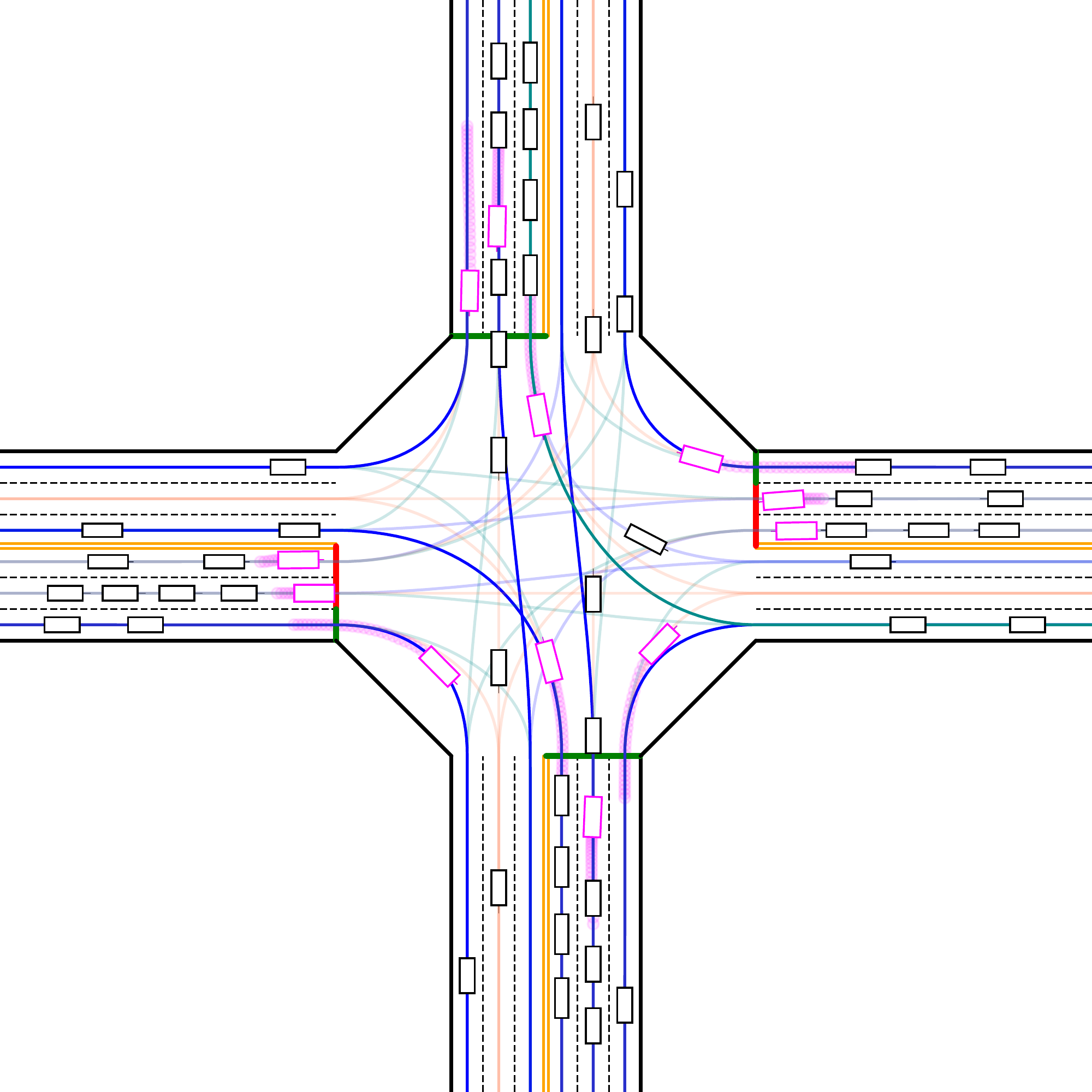}}
\subfloat[t=5.4s]{\includegraphics[width=0.2\linewidth]{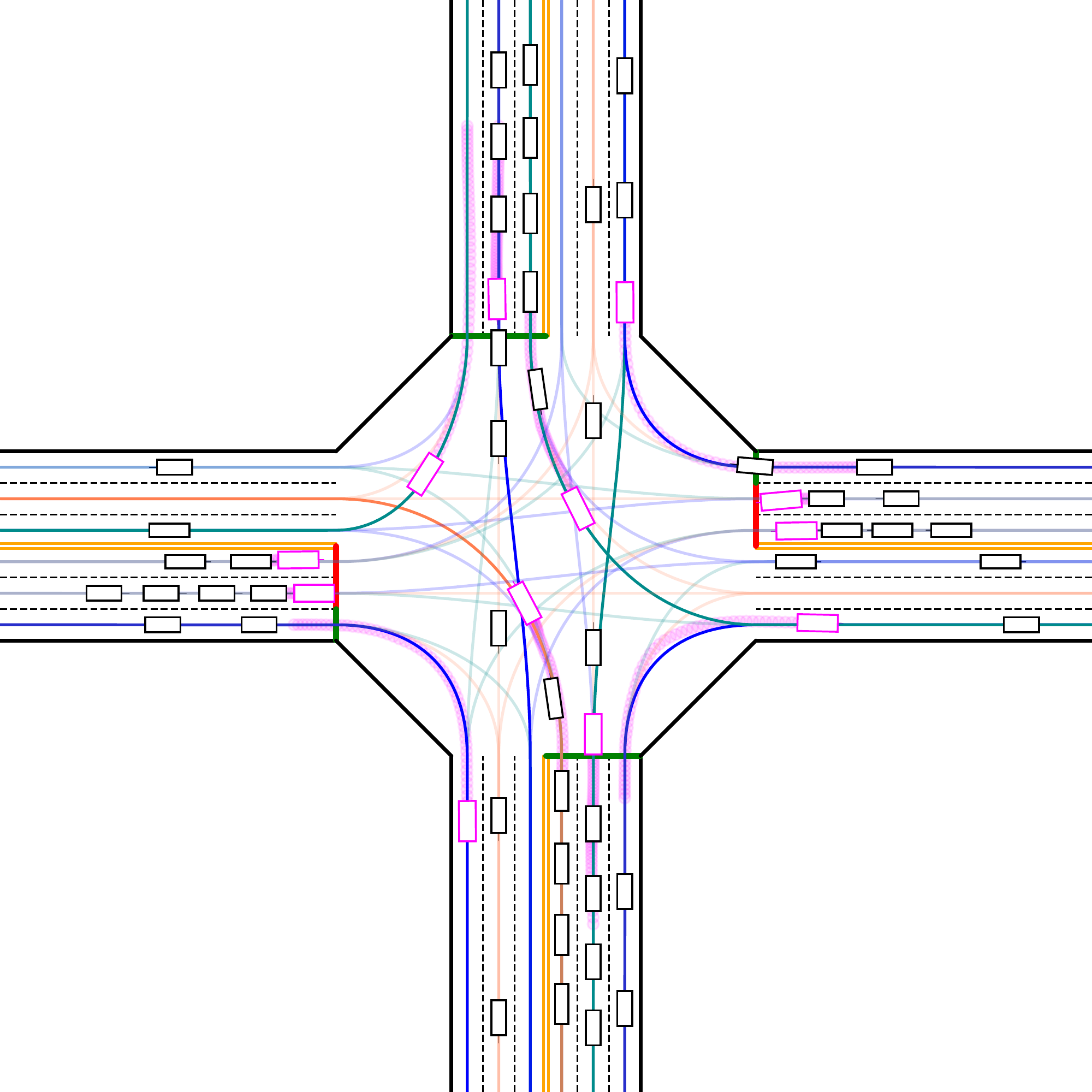}}
\subfloat[t=8.0s]{\includegraphics[width=0.2\linewidth]{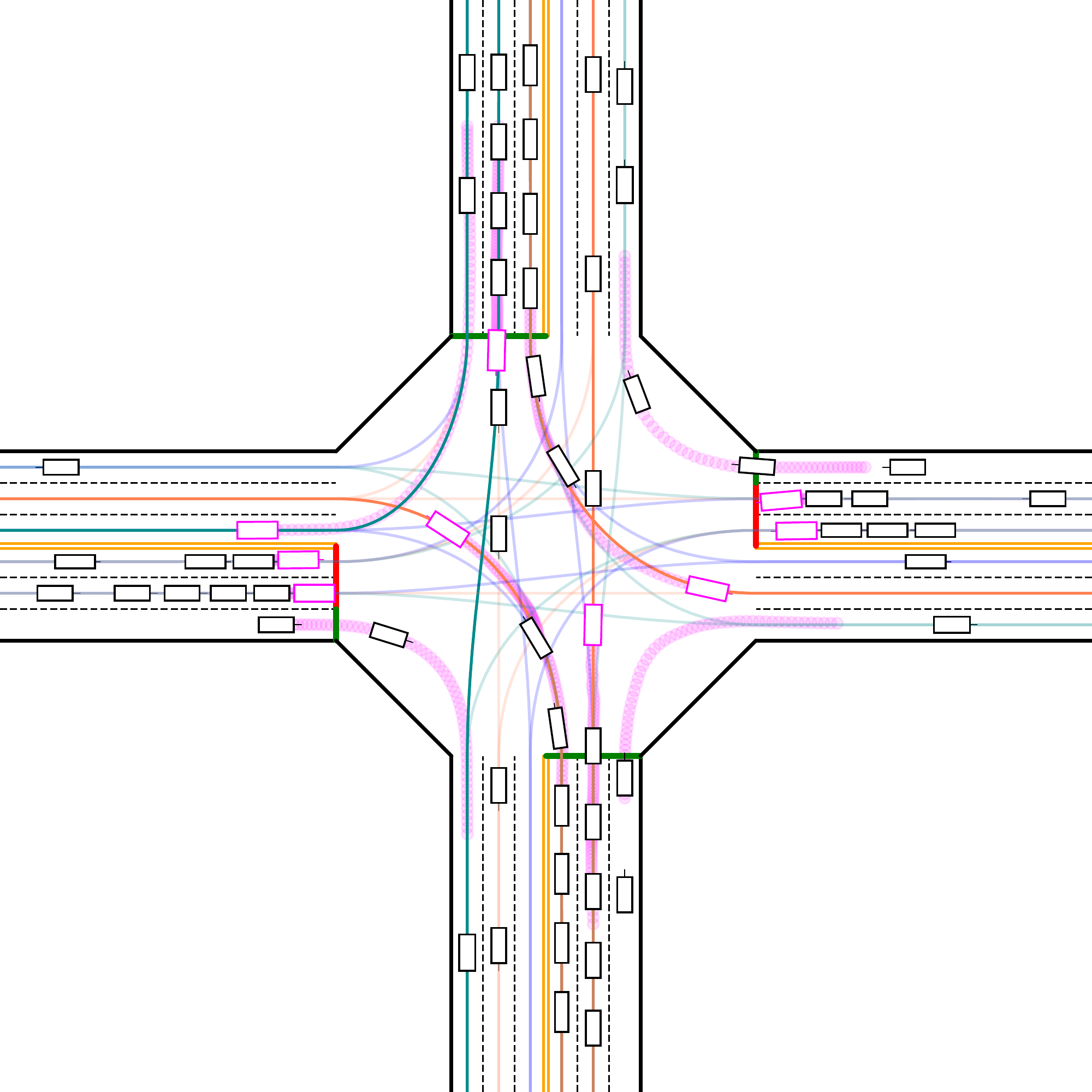}}
\subfloat[t=10.6s]{\includegraphics[width=0.2\linewidth]{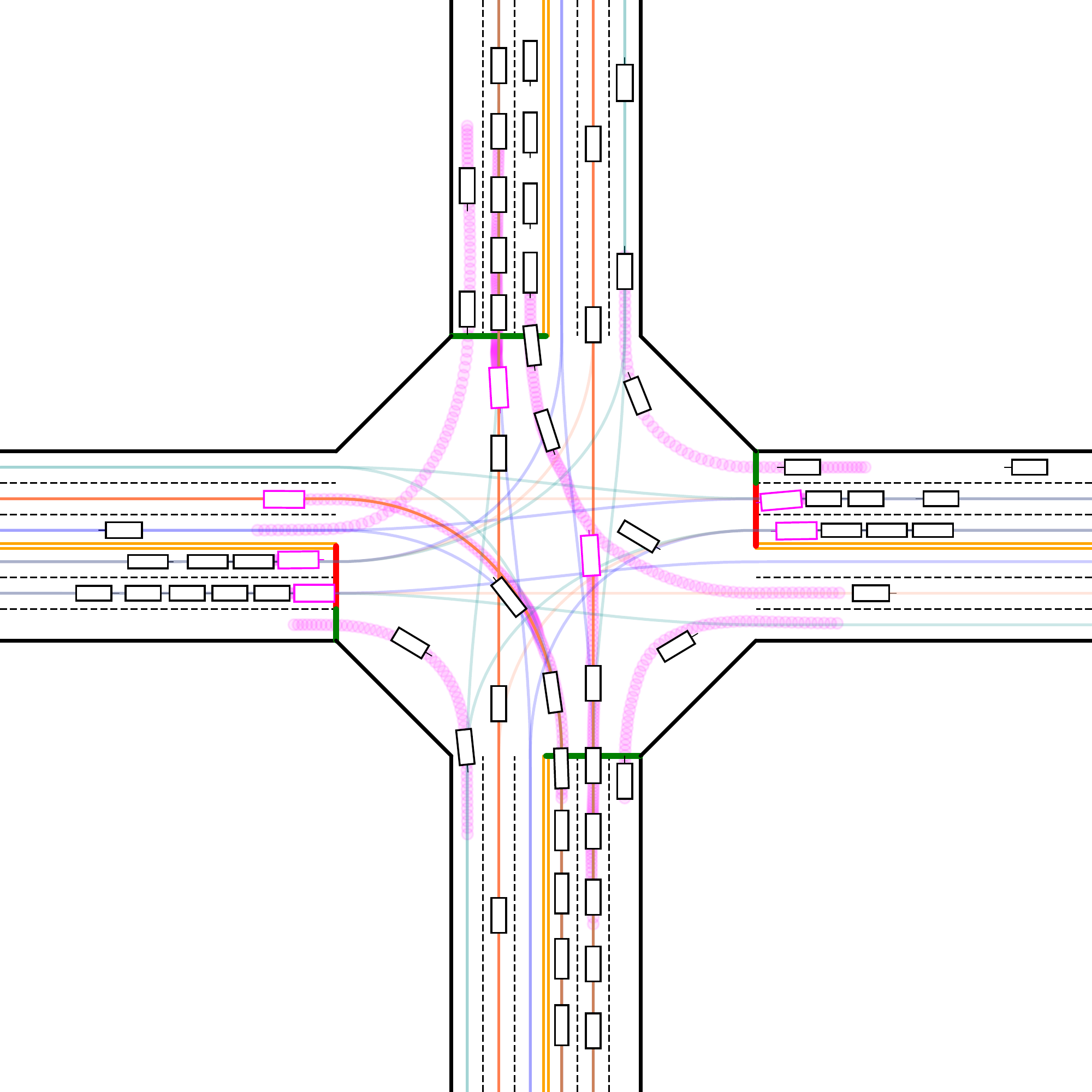}}\\
\subfloat[t=13.2s]{\includegraphics[width=0.2\linewidth]{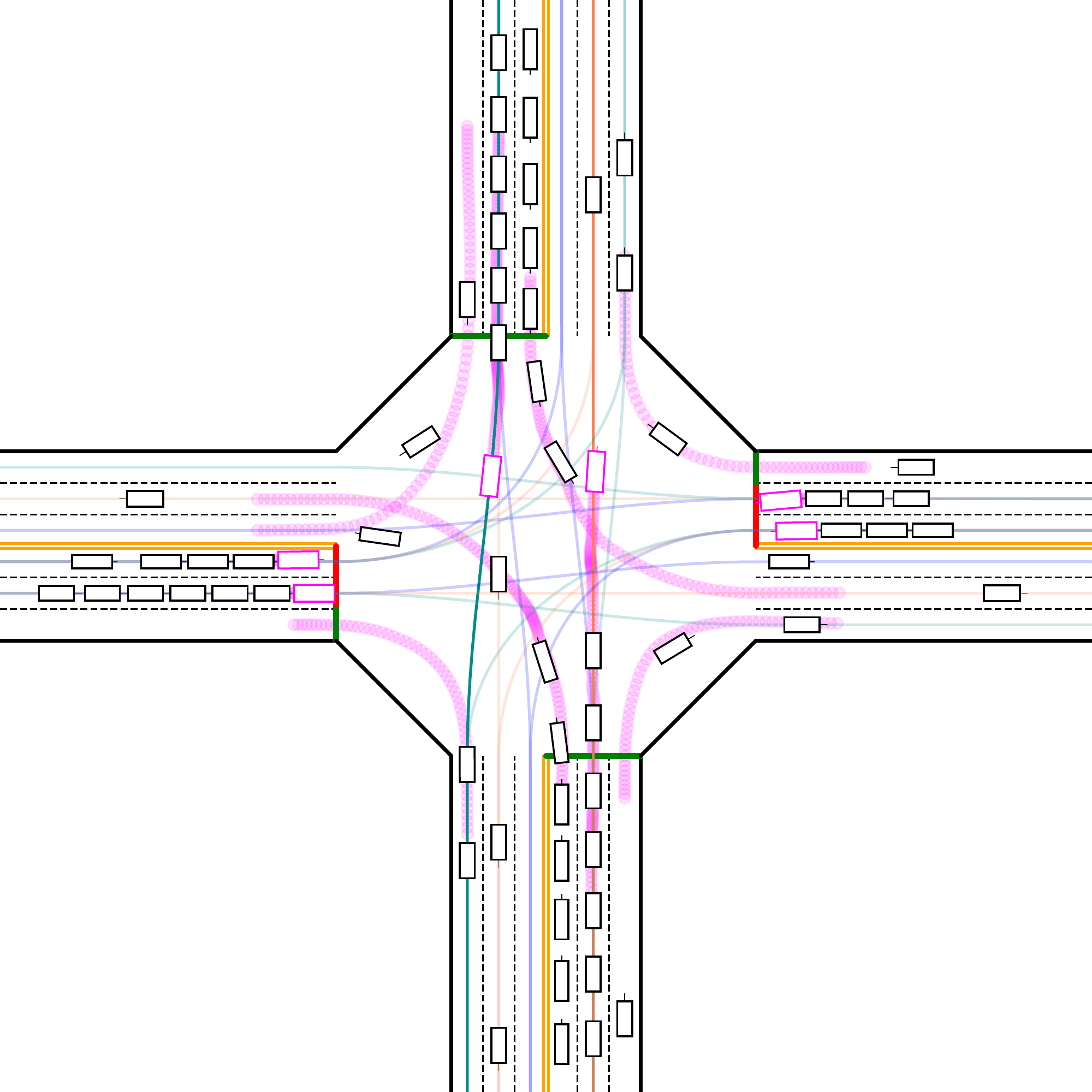}}
\subfloat[t=15.8s]{\includegraphics[width=0.2\linewidth]{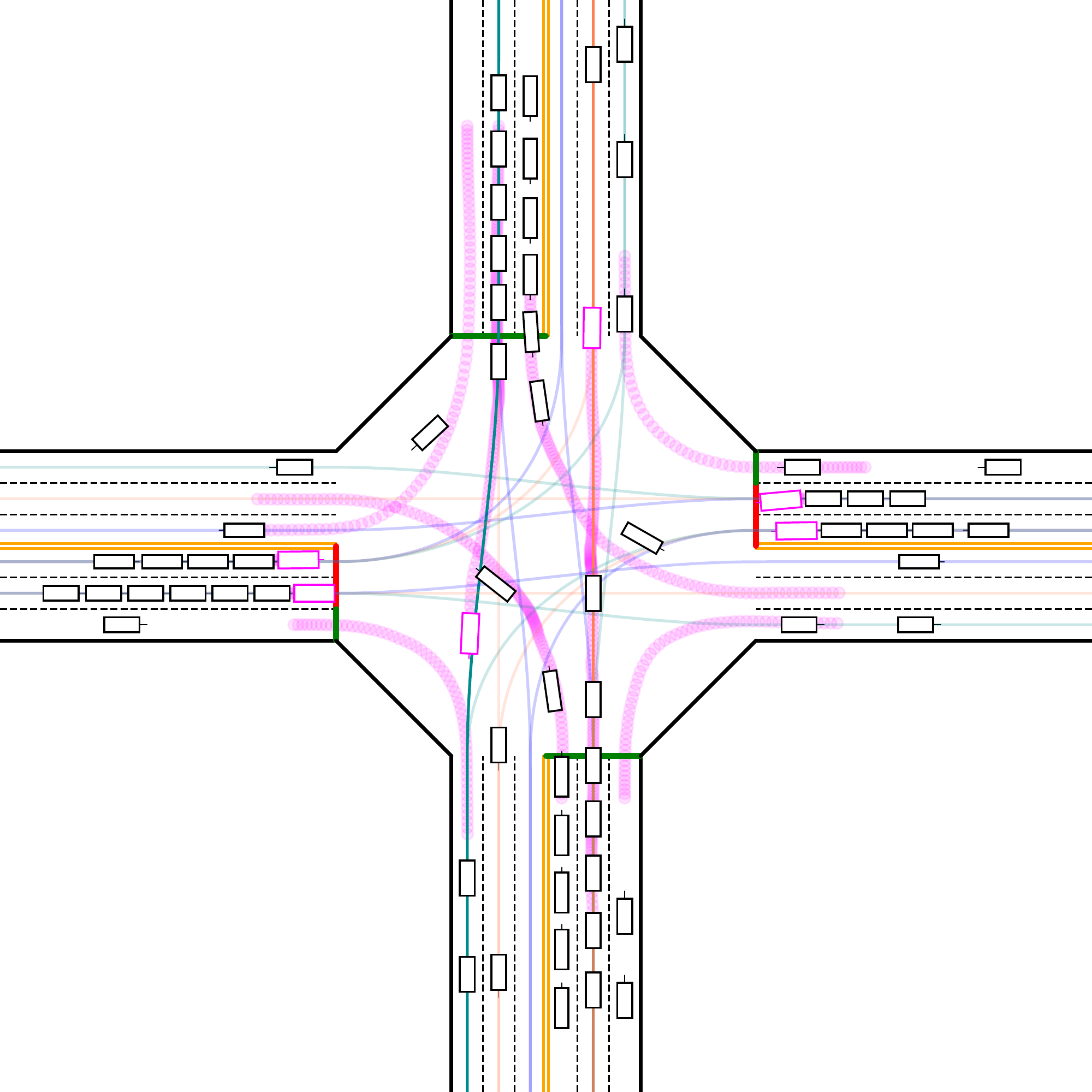}}
\subfloat[t=18.4s]{\includegraphics[width=0.2\linewidth]{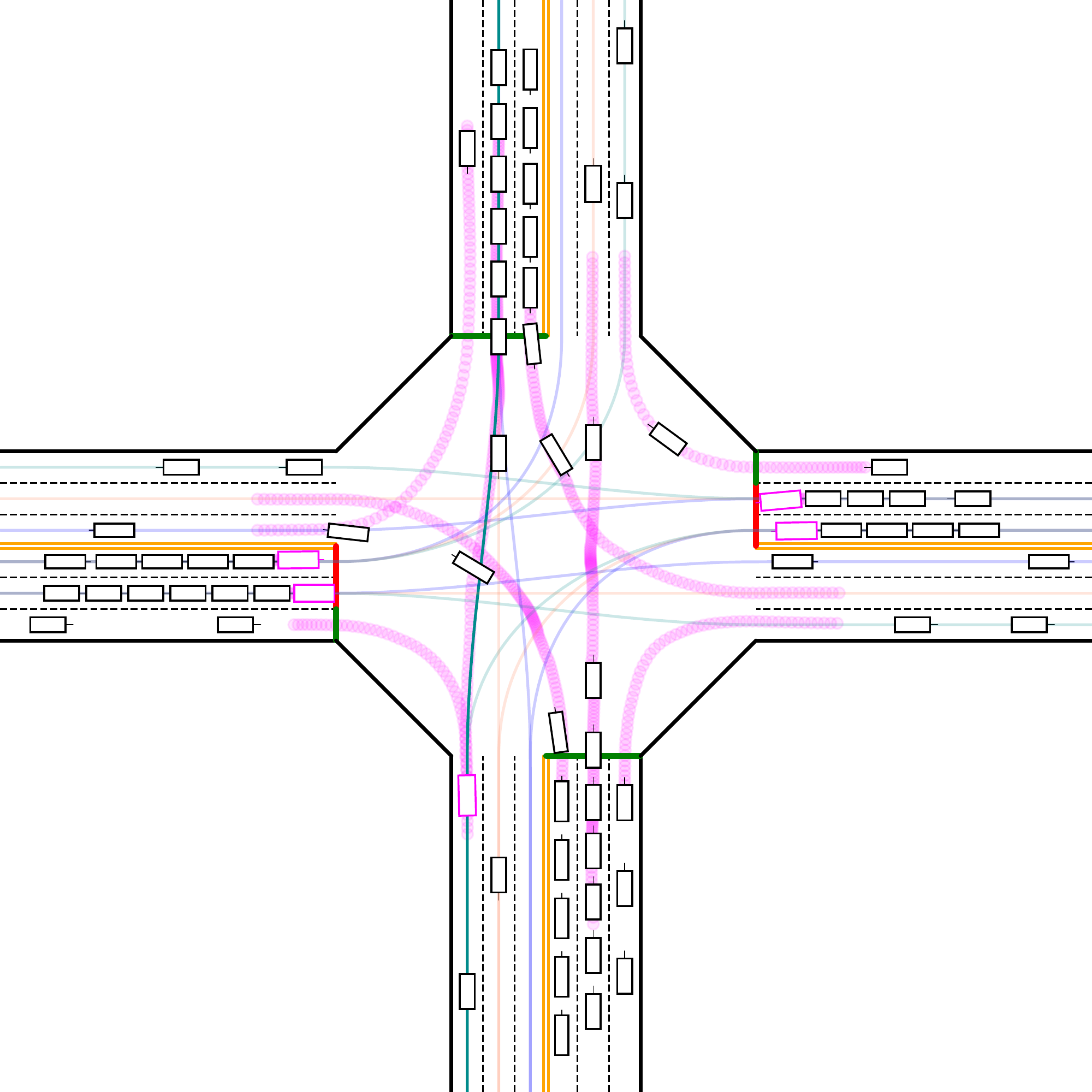}}
\subfloat[t=28.8s]{\includegraphics[width=0.2\linewidth]{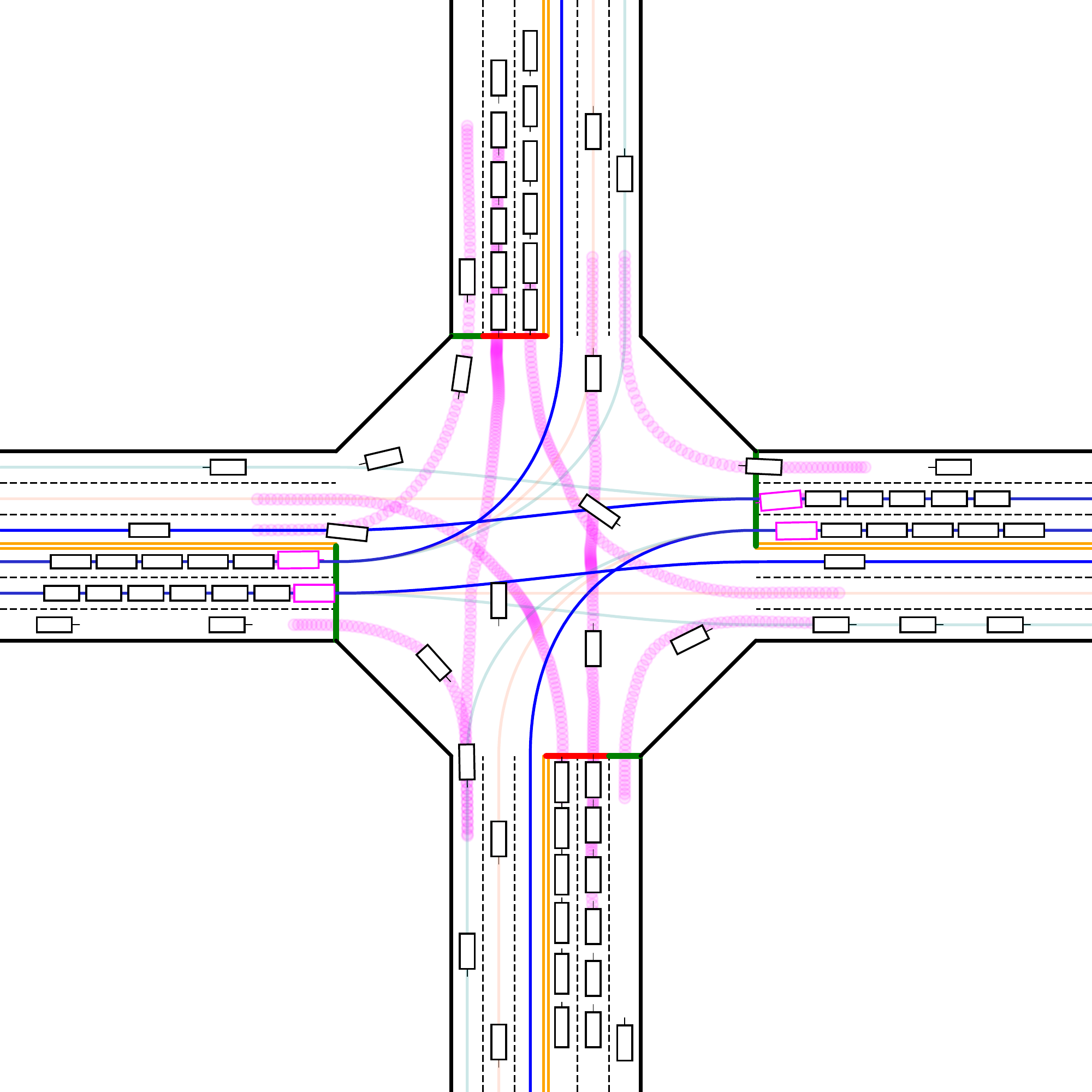}}
\subfloat[t=31.4s]{\includegraphics[width=0.2\linewidth]{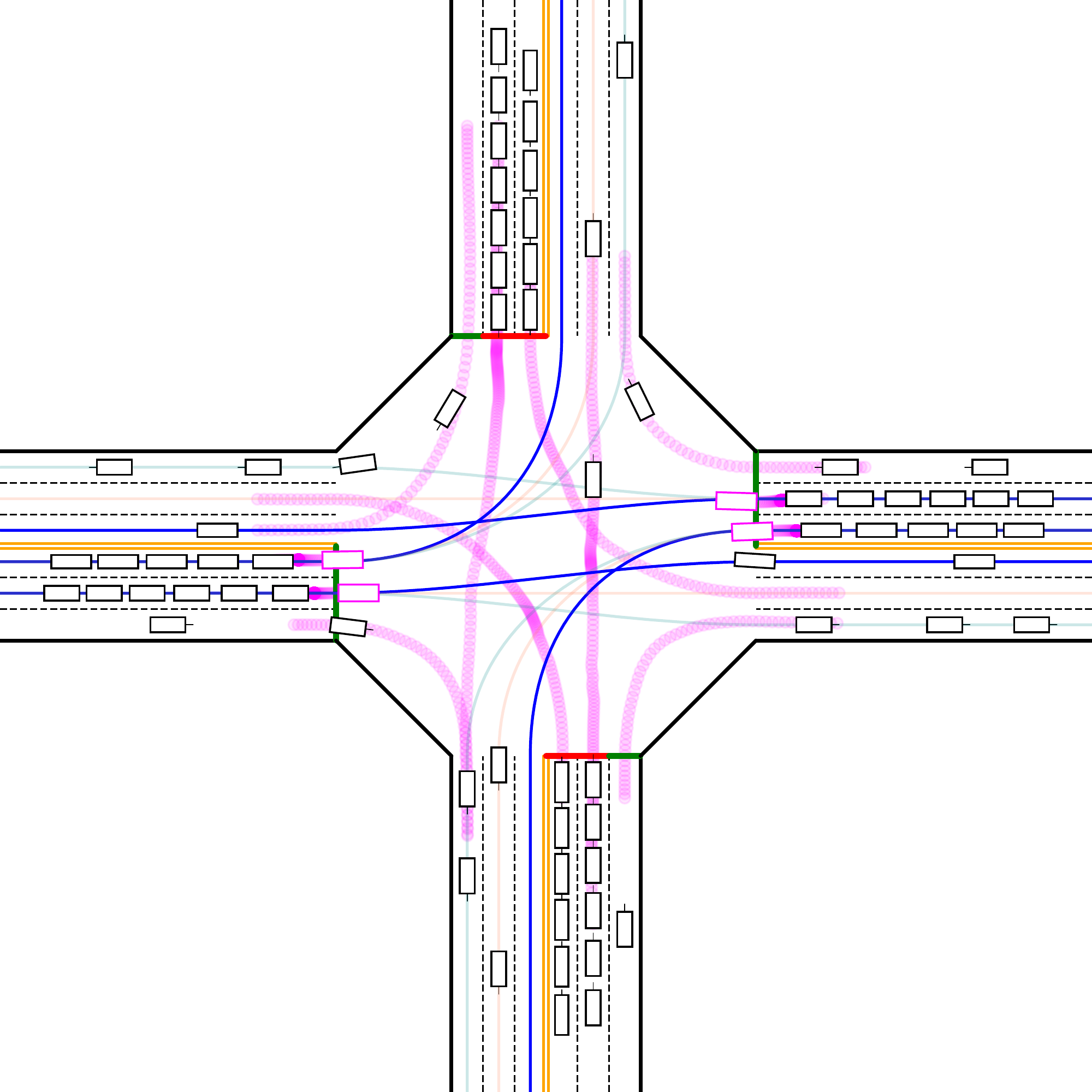}}\\
\subfloat[t=34.0s]{\includegraphics[width=0.2\linewidth]{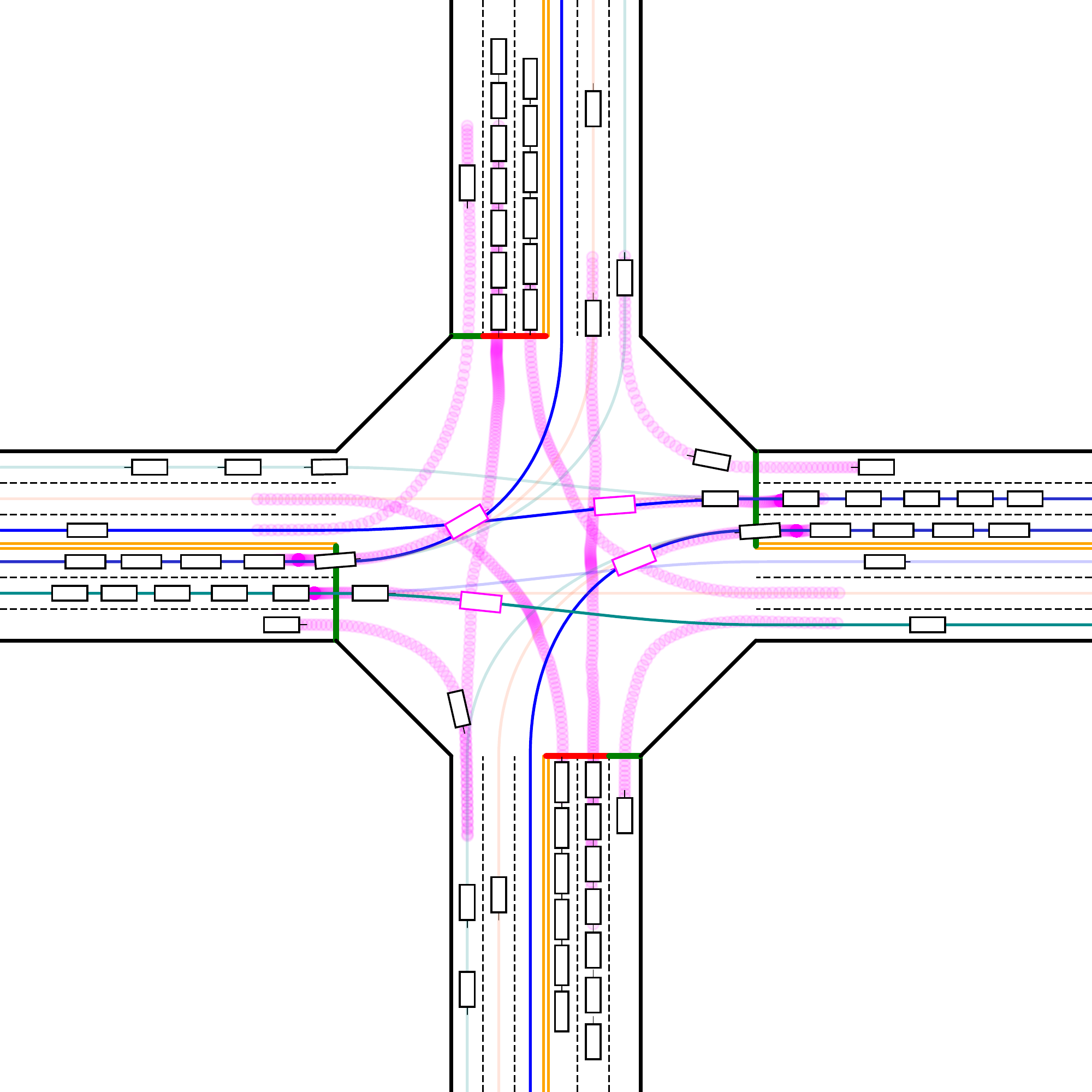}}
\subfloat[t=36.6s]{\includegraphics[width=0.2\linewidth]{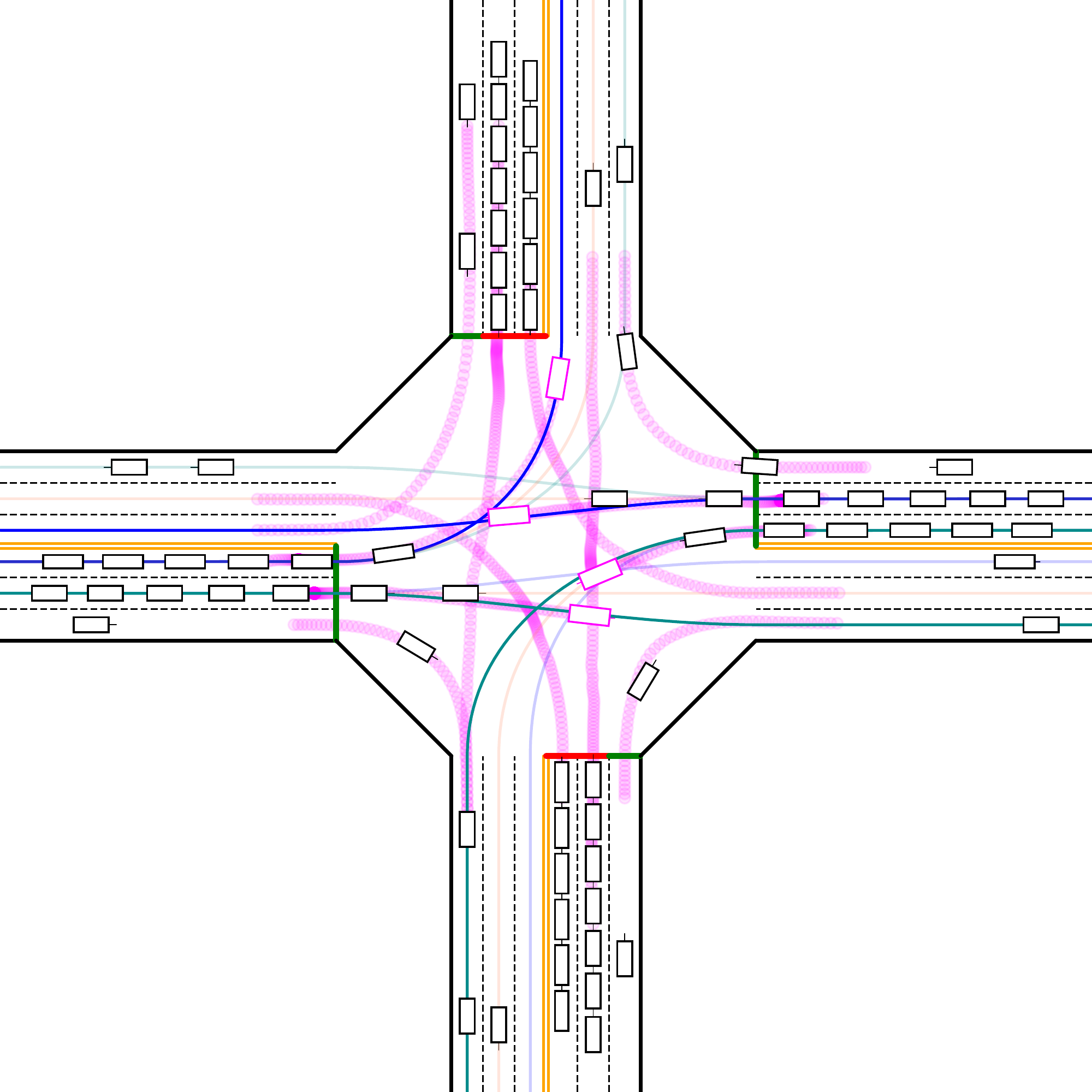}}
\subfloat[t=39.2s]{\includegraphics[width=0.2\linewidth]{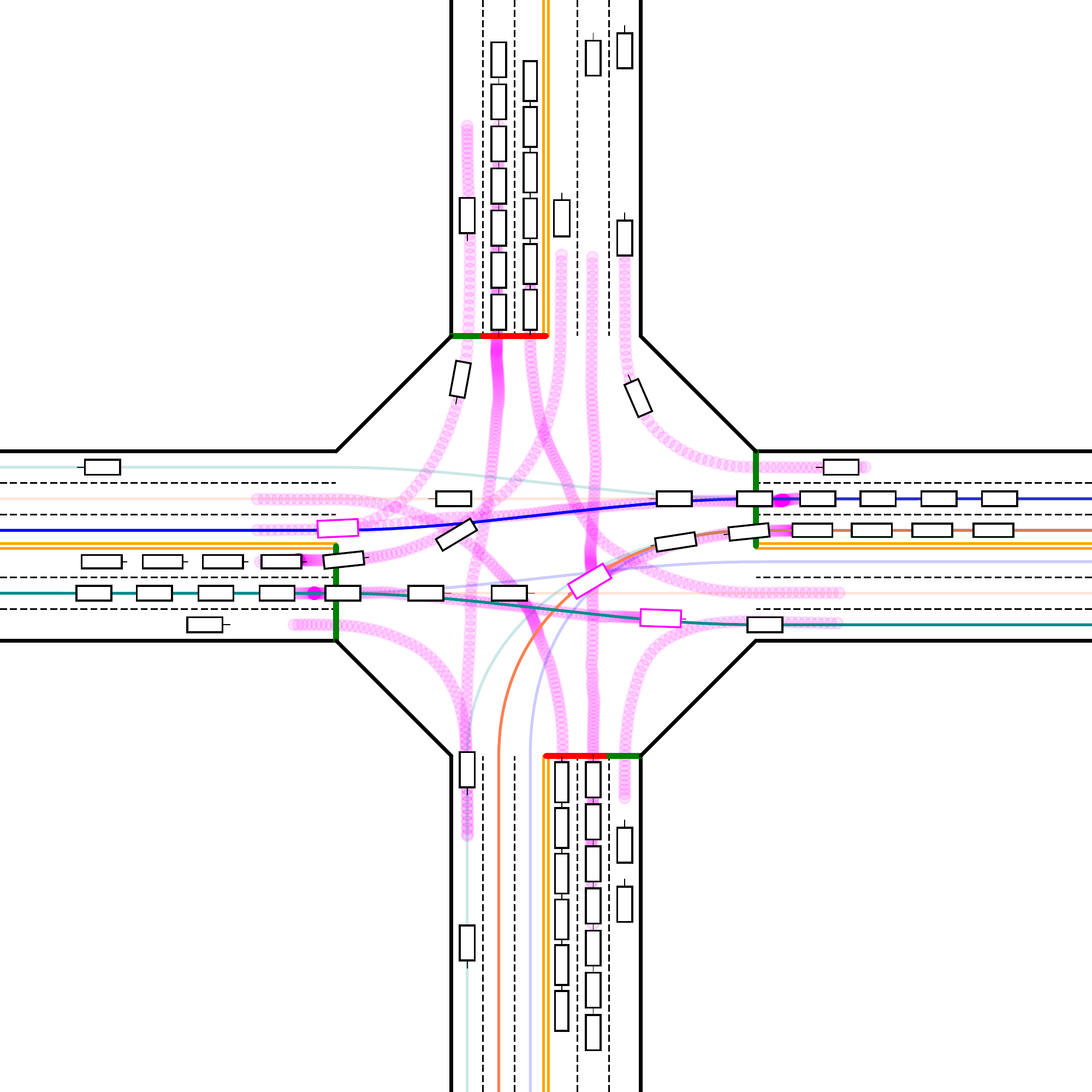}}
\subfloat[t=41.8s]{\includegraphics[width=0.2\linewidth]{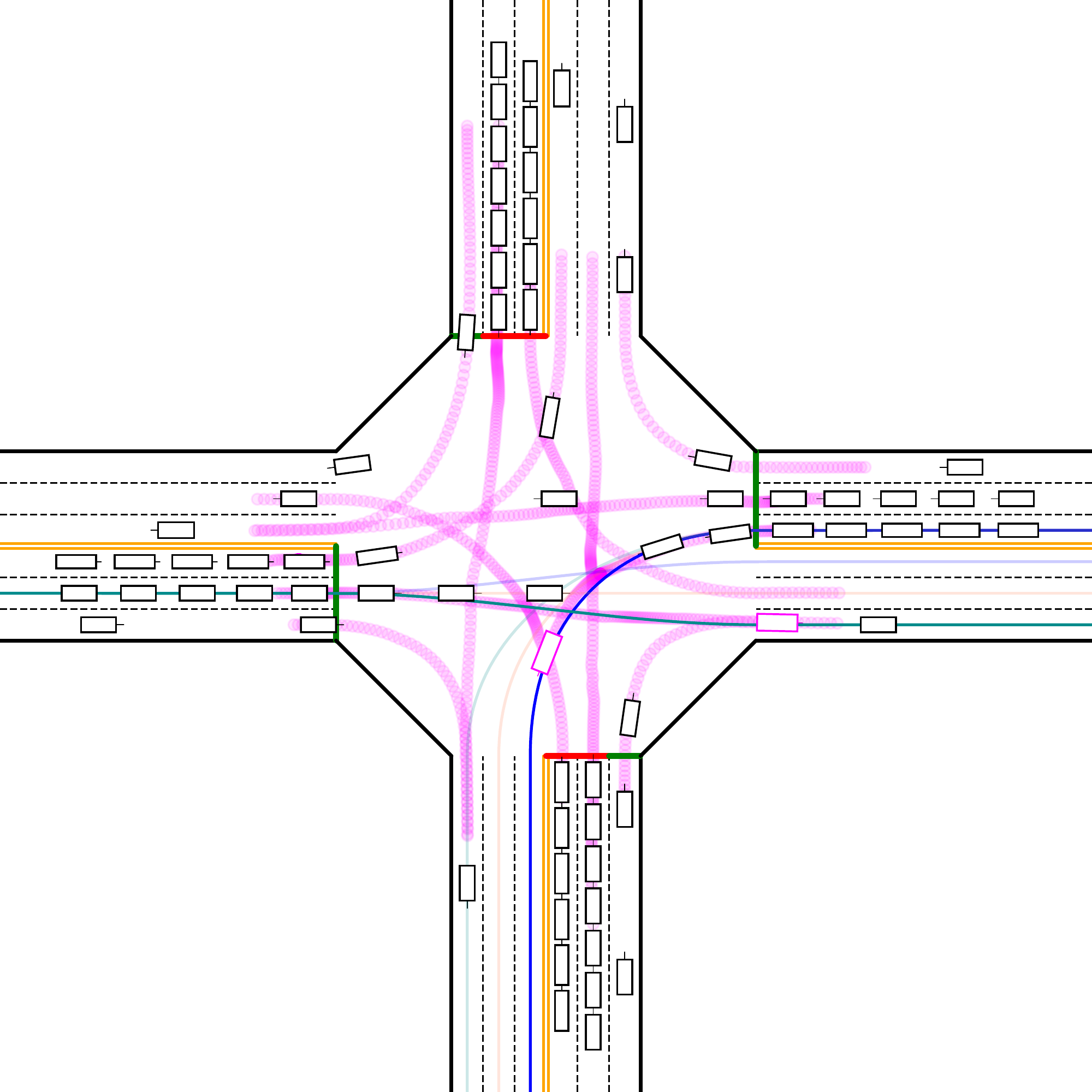}}
\subfloat[t=44.4s]{\includegraphics[width=0.2\linewidth]{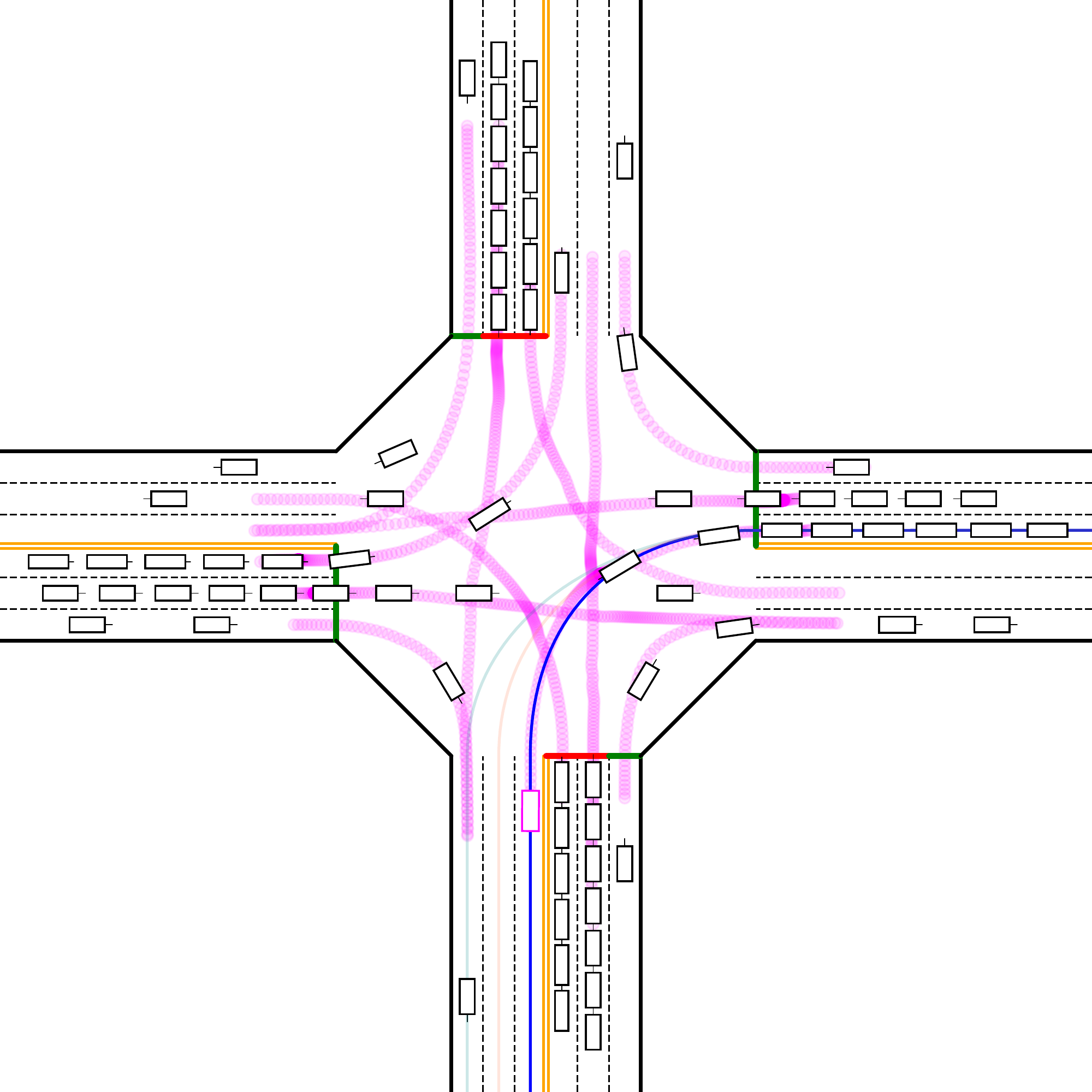}}\\
\caption{Demonstration of the distributed control carried out by the trained policies.}
\label{fig.simulation_multi_demo}
\end{figure*}
\section{Test on real-world roads}
\subsection{Scenario and equipment}
In the real-world test, we choose an intersection of two-way streets located at ($31^{\degree}08'13''$N, $120^{\degree}35'06''$E), shown in Fig. \ref{fig.real_test_true}. The east-west street has a eight-lane dual carriageway from both directions, while the north-south street has a only four-lane dual carriageway. The detailed size and functionality of each lane are illustrated in Fig. \ref{fig.real_test_illus}. To comply with legal requirements, we did not utilize the traffic flow and traffic signals of the real intersection. Instead, these traffic elements are designed and provided by SUMO. The experiment vehicle is CHANA CS55 equipped with RTK GPS, which realizes precise localization of the ego vehicle. In each time step, the ego states gathered from the CAN bus and the RTK are mapped into SUMO traffic to obtain the current traffic states including the states of surrounding vehicles and traffic signals. Then they both are sent to the industrial computer, where the trained policy and value functions are embedded. The computer is KMDA-3211 with a 2.6 GHz Intel Core I5-6200U CPU. Processed by our online algorithm, the safe actions including the steering angle and the expected acceleration are then delivered from the CAN bus to the real vehicle for its real time control. The experiment settings are illustrated in Fig. \ref{fig.real_test_setting}. Similar to section \ref{sec.simulation}, the ego vehicle enters the intersection from south, and is required to complete the same tasks (turn left, go straight, and turn right) under the signal control and a dense traffic flow.

\begin{figure}[htbp]
\centering
\captionsetup[subfigure]{justification=centering}
\subfloat[]{\label{fig.real_test_true}\includegraphics[width=0.24\textwidth]{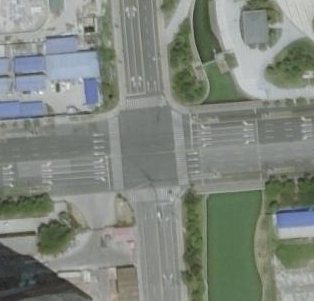}}\quad
\subfloat[]{\label{fig.real_test_illus}\includegraphics[width=0.22\textwidth]{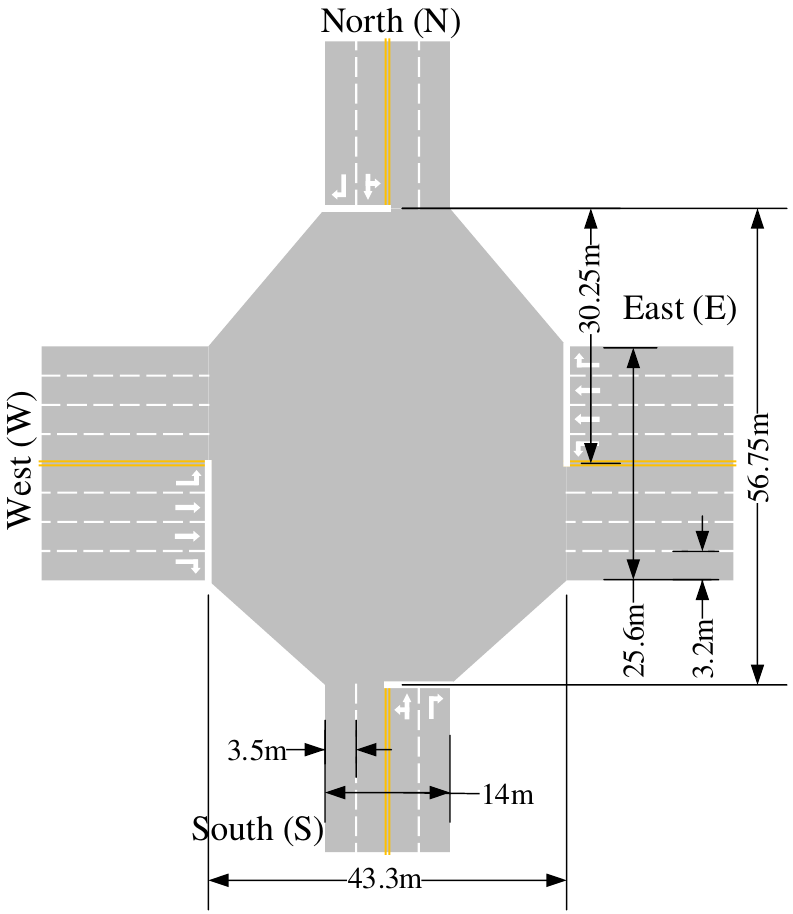}}\\
\caption{The intersection for the real test.}
\label{fig.real_scene}
\end{figure}

\begin{figure}[htbp]
\centerline{\includegraphics[width=0.9\linewidth]{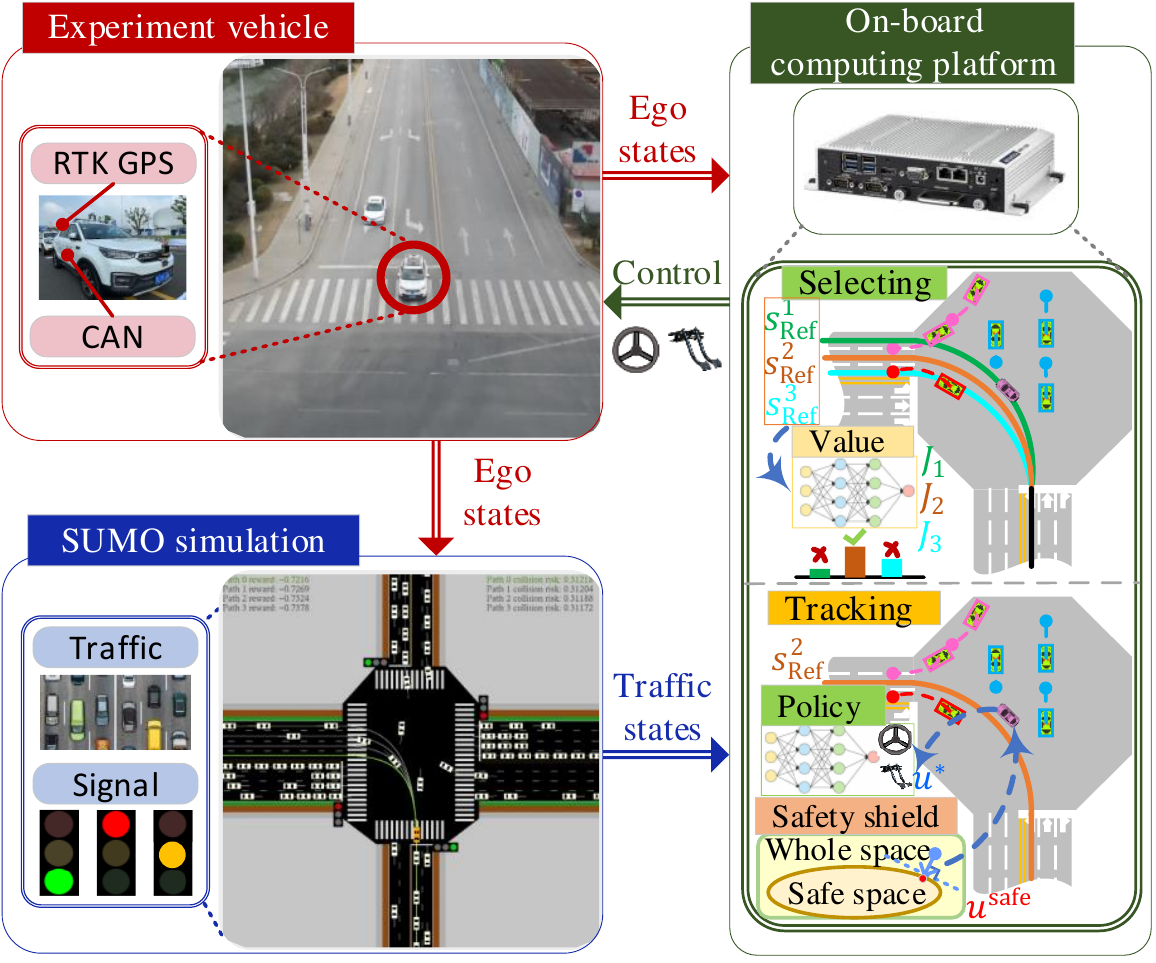}}
\caption{Diagram of the real-world road test.}
\label{fig.real_test_setting}
\end{figure}

\begin{figure*}[htbp]
\centering
\captionsetup[subfigure]{justification=centering}
\subfloat[]{\label{fig.real_test_step1}\includegraphics[width=0.3\linewidth]{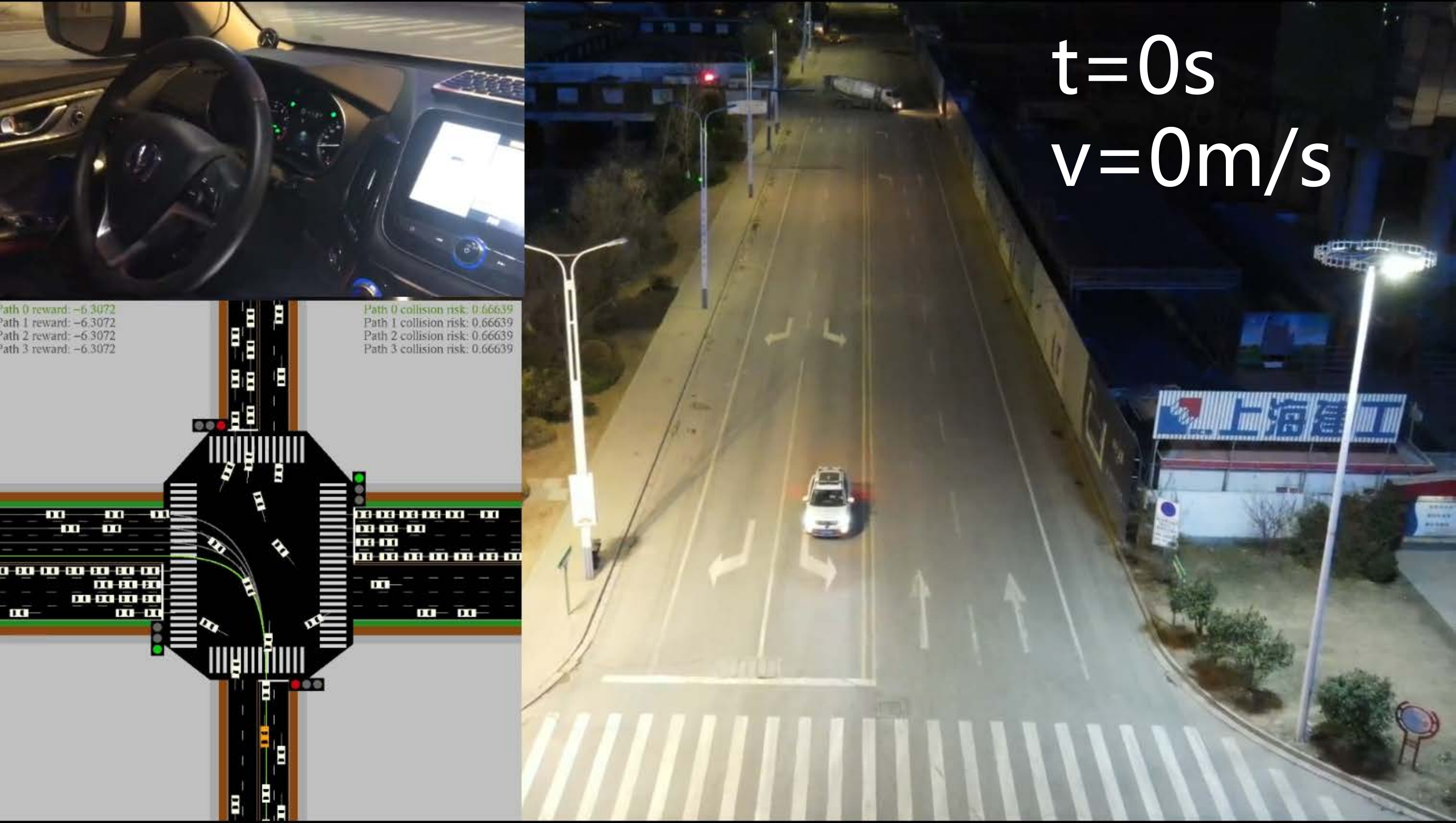}}\ 
\subfloat[]{\label{fig.real_test_step2}\includegraphics[width=0.3\textwidth]{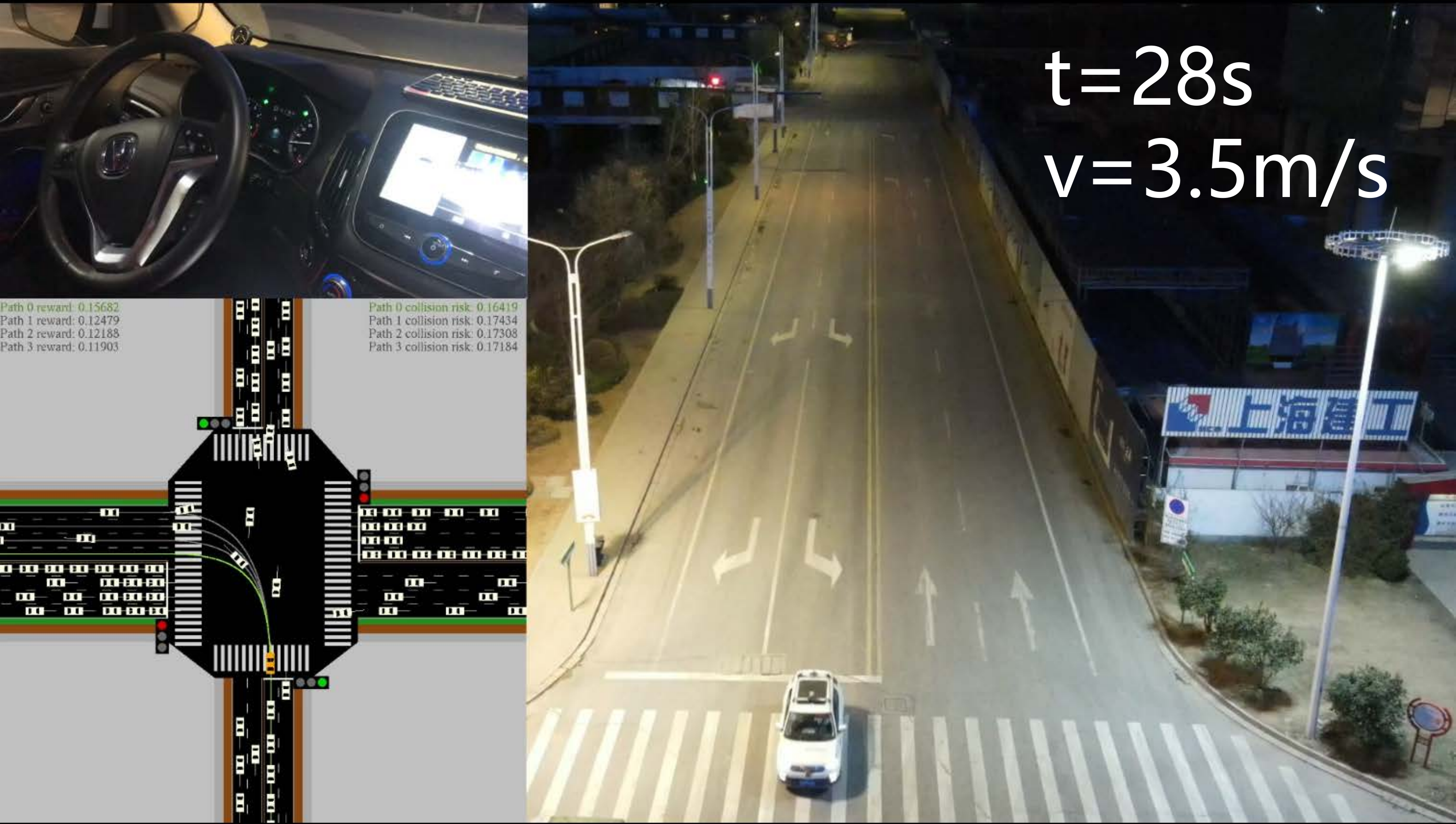}}\ 
\subfloat[]{\label{fig.real_test_step3}\includegraphics[width=0.3\textwidth]{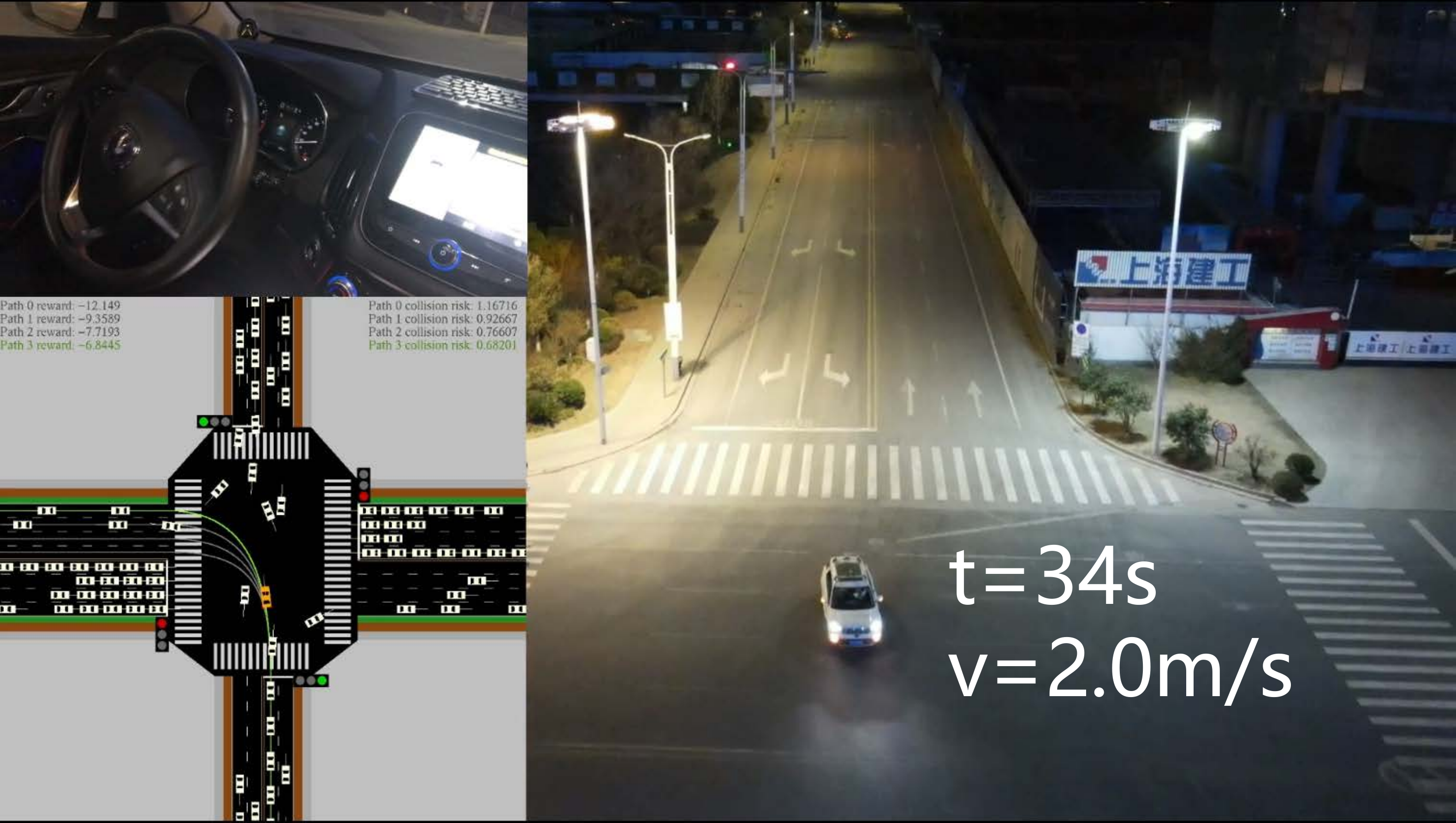}}\\
\subfloat[]{\label{fig.real_test_step4}\includegraphics[width=0.3\textwidth]{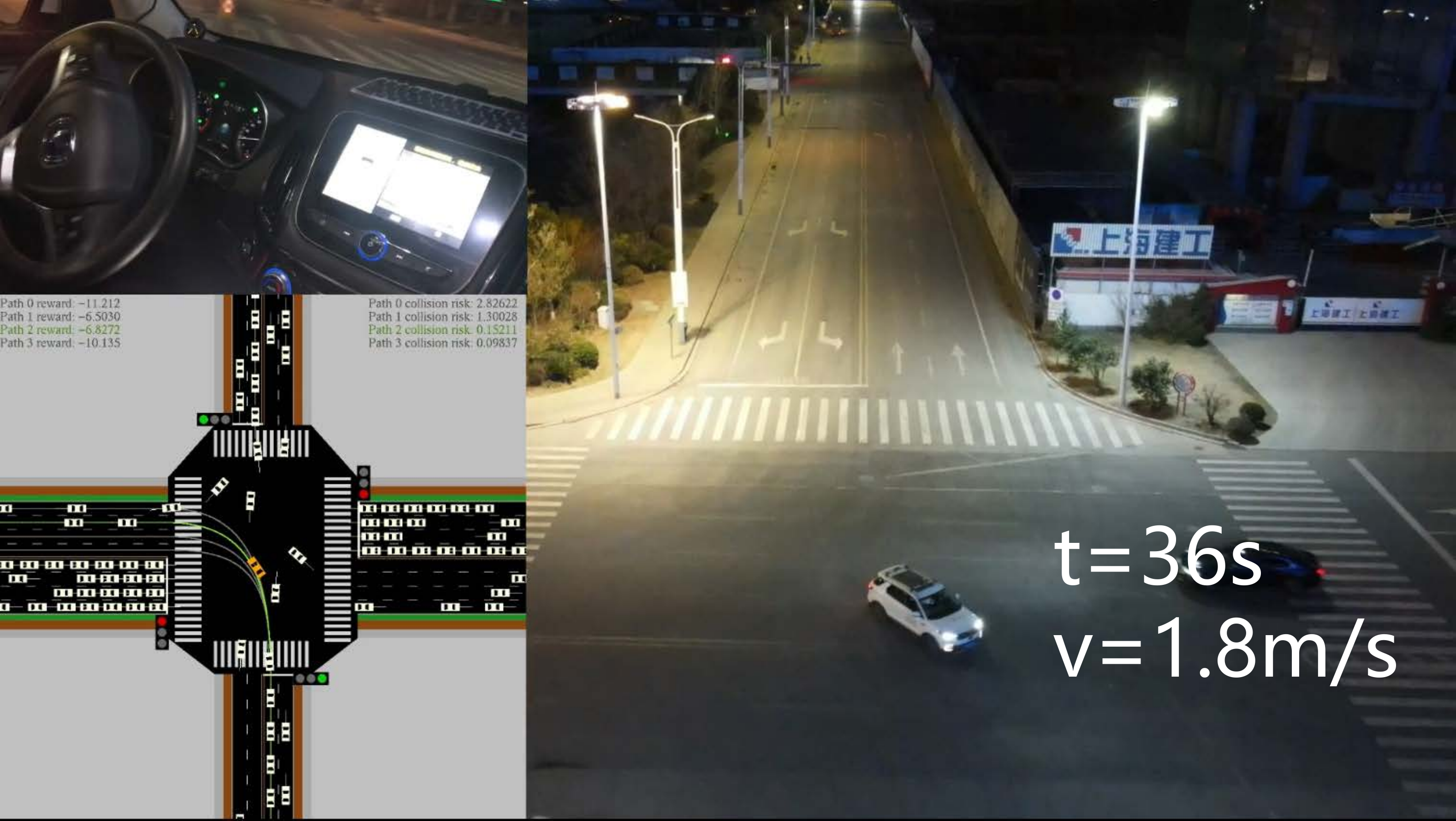}}\ 
\subfloat[]{\label{fig.real_test_step5}\includegraphics[width=0.3\textwidth]{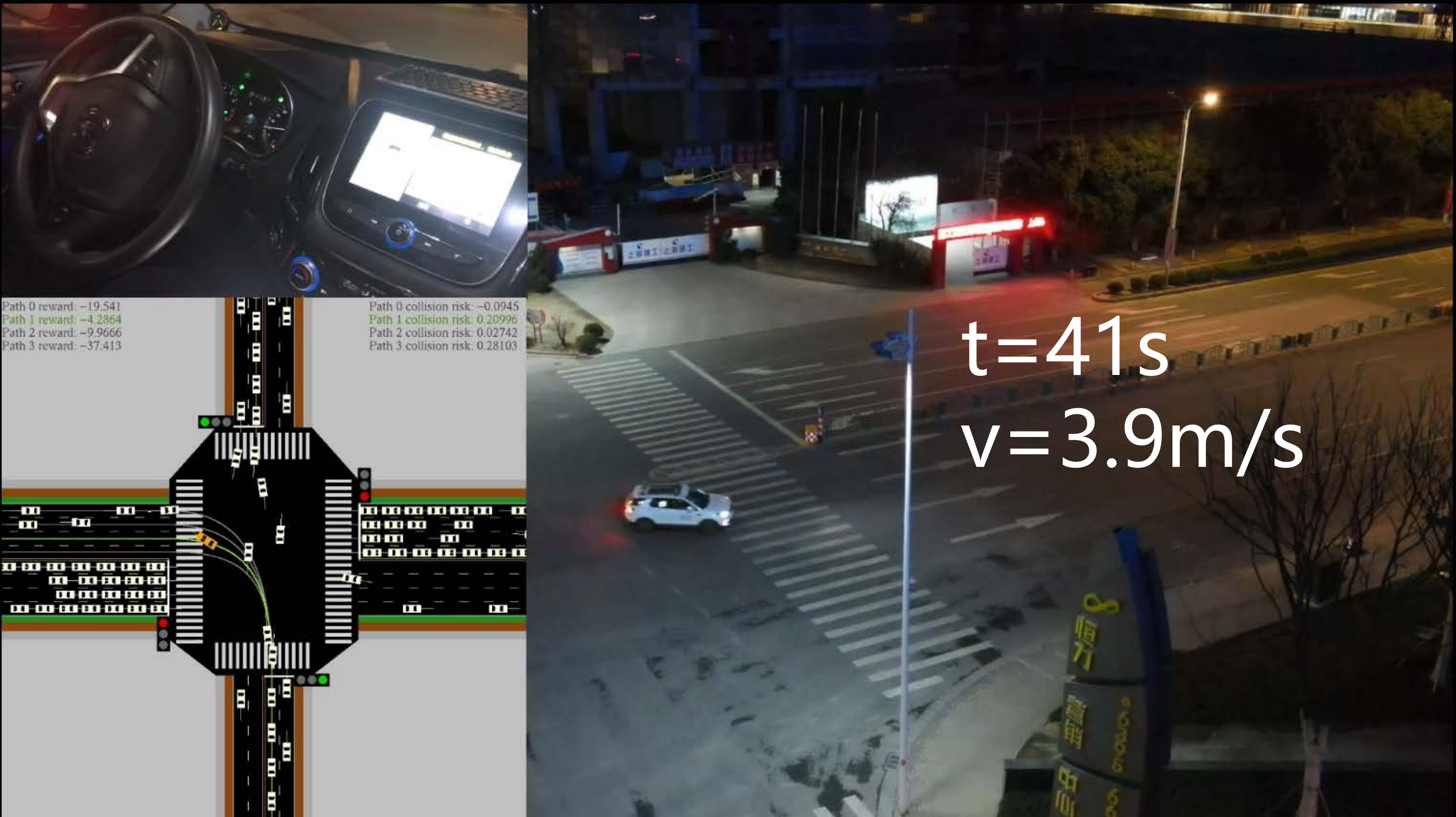}}\ 
\subfloat[]{\label{fig.real_test_step6}\includegraphics[width=0.3\textwidth]{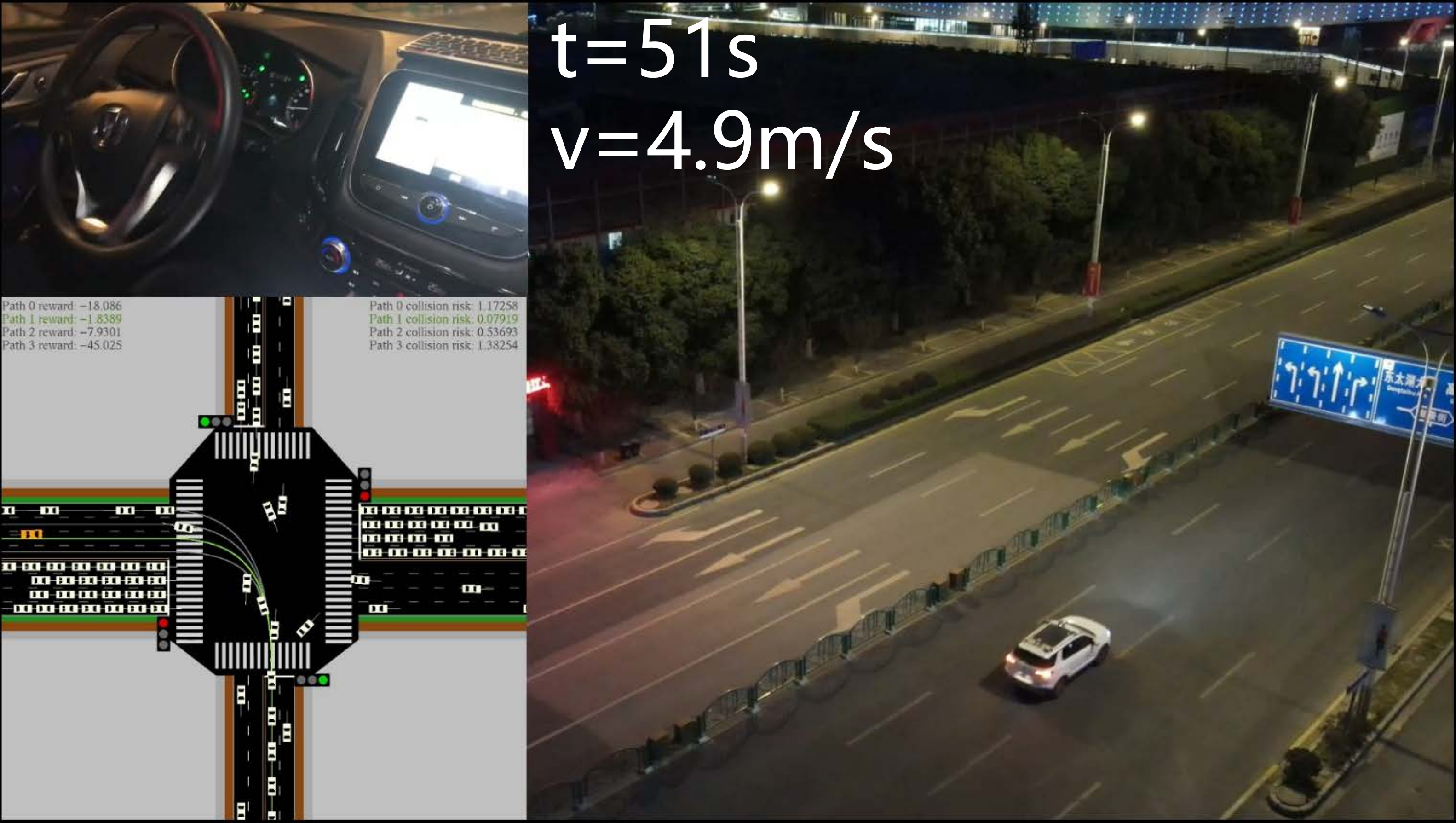}}\ 
\caption{Featured time steps of the left run.}
\label{fig.real_test_demo}
\end{figure*}

\begin{figure}[htbp]
\centering
\captionsetup[subfigure]{justification=centering}
\subfloat[Speed]{\includegraphics[width=0.15\textwidth]{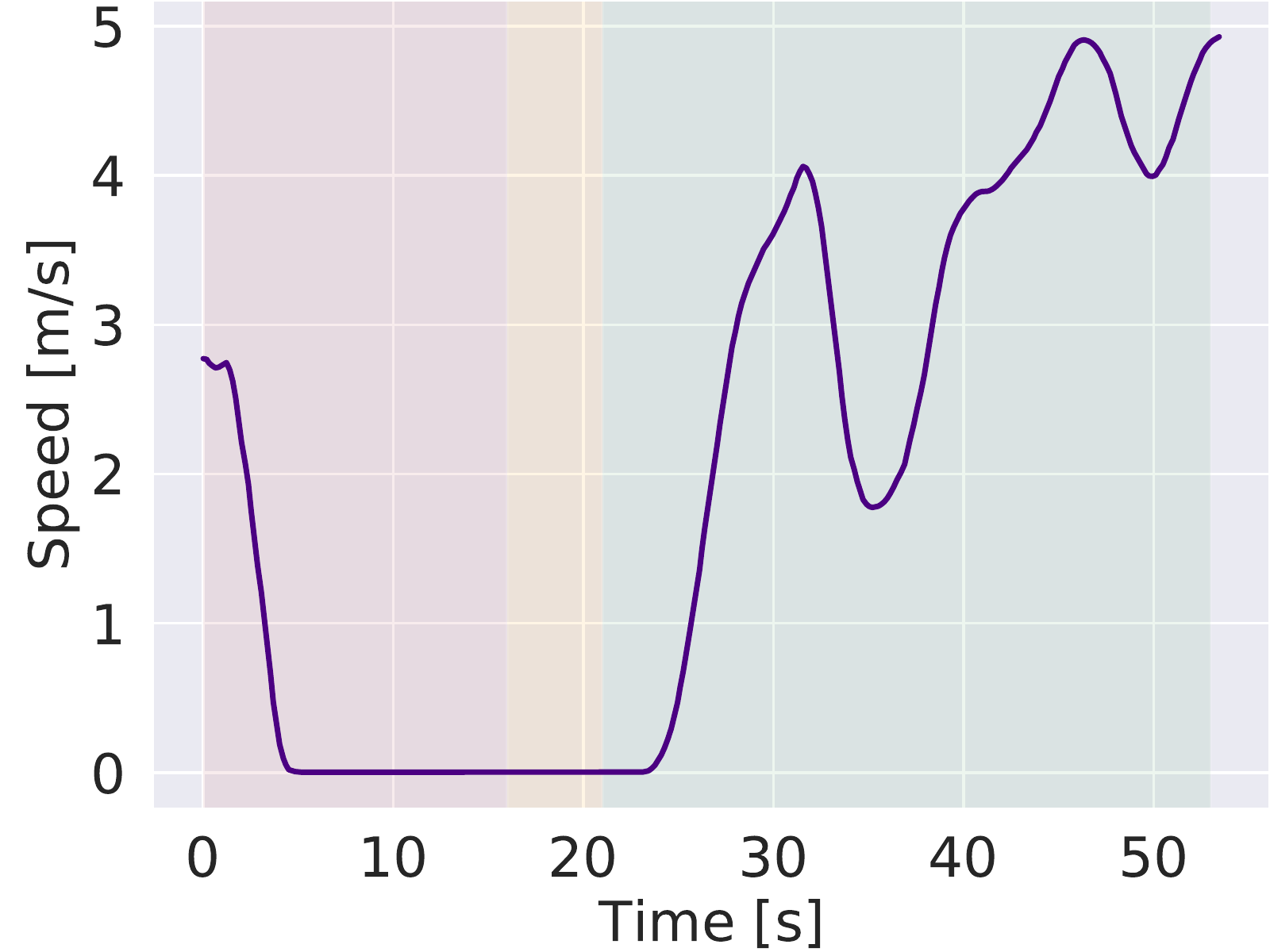}}
\subfloat[Ref index]{\includegraphics[width=0.15\textwidth]{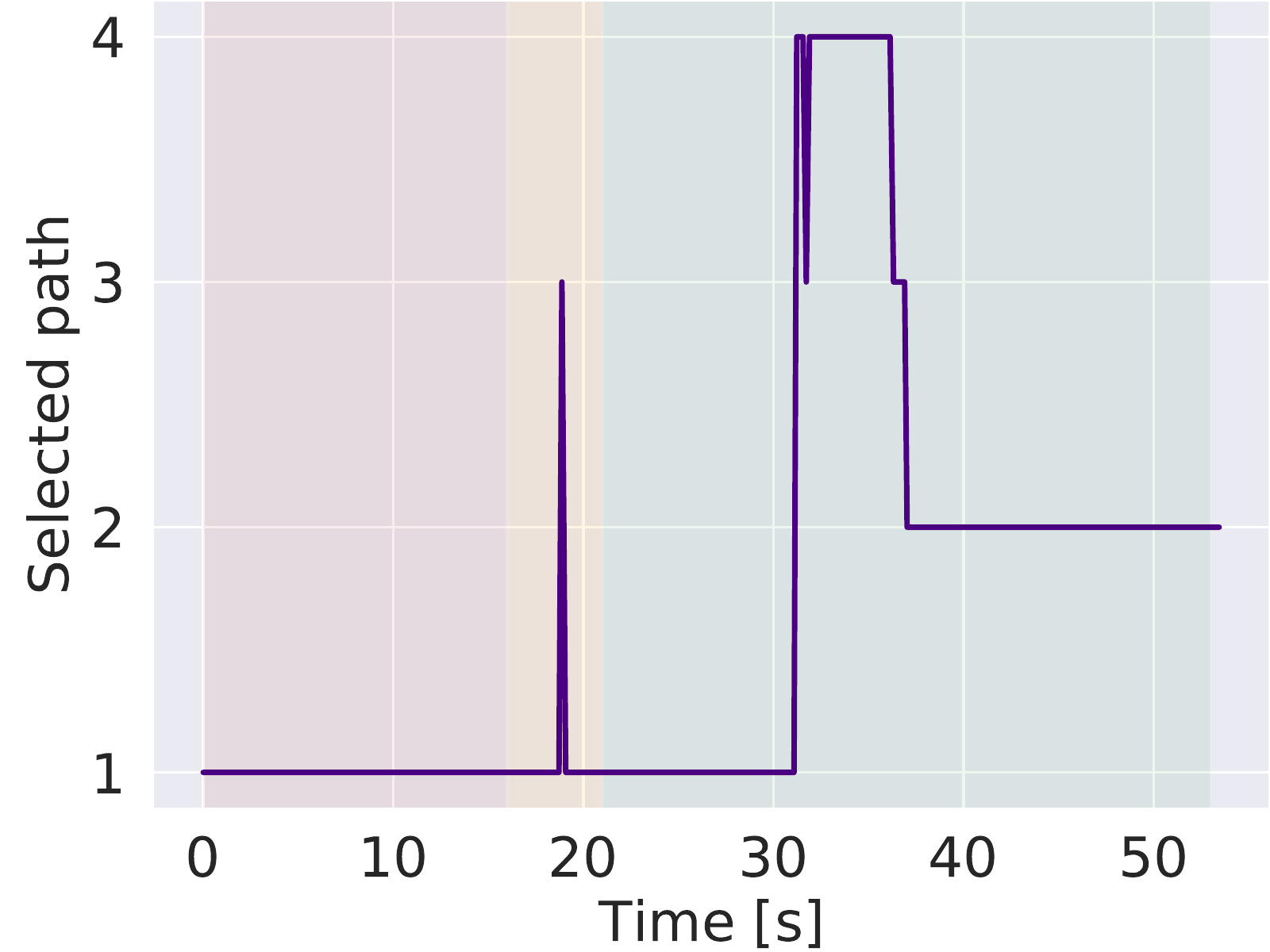}}
\subfloat[Heading angle]{\includegraphics[width=0.15\textwidth]{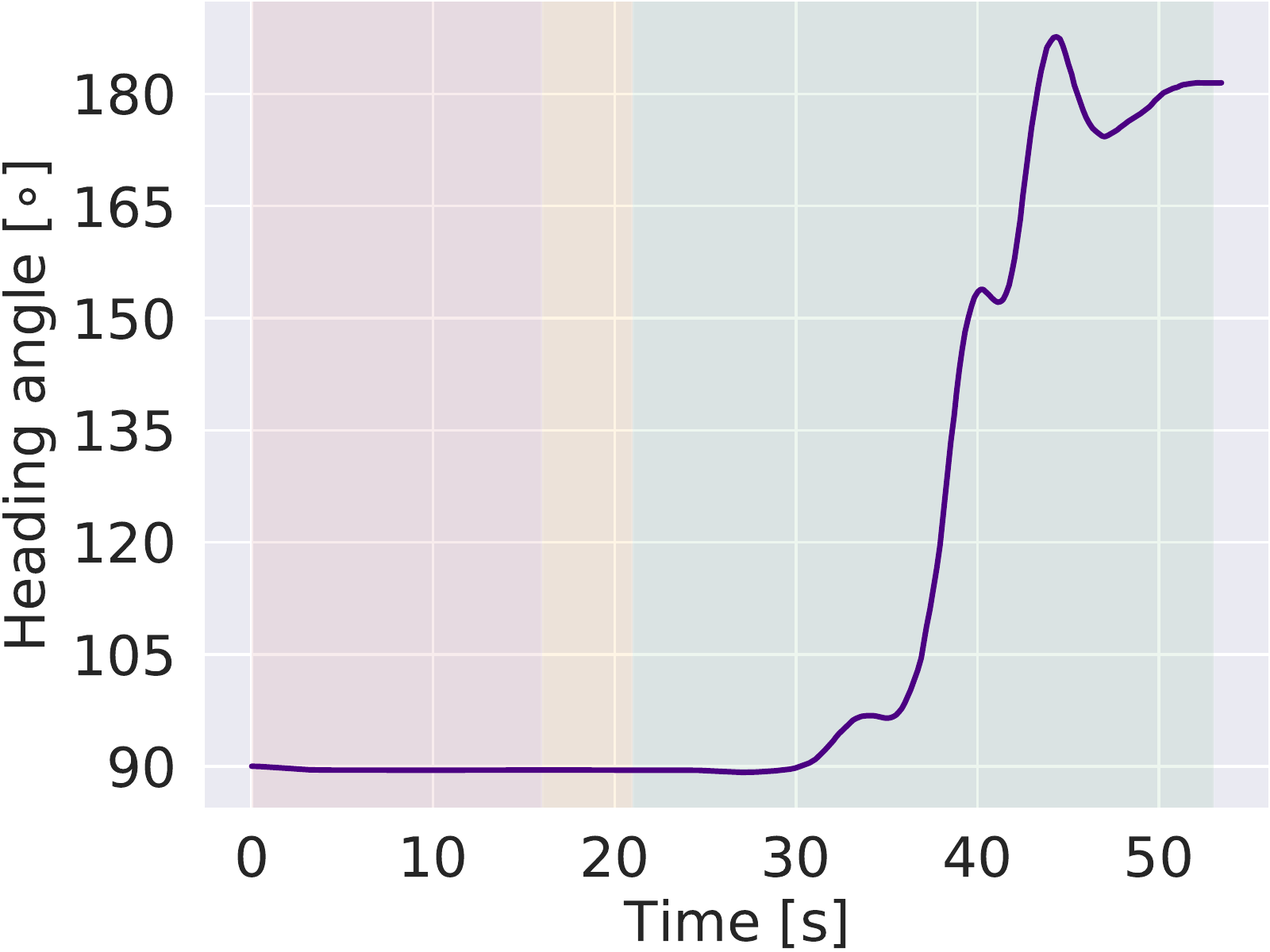}}\\
\subfloat[Position error]{\includegraphics[width=0.15\textwidth]{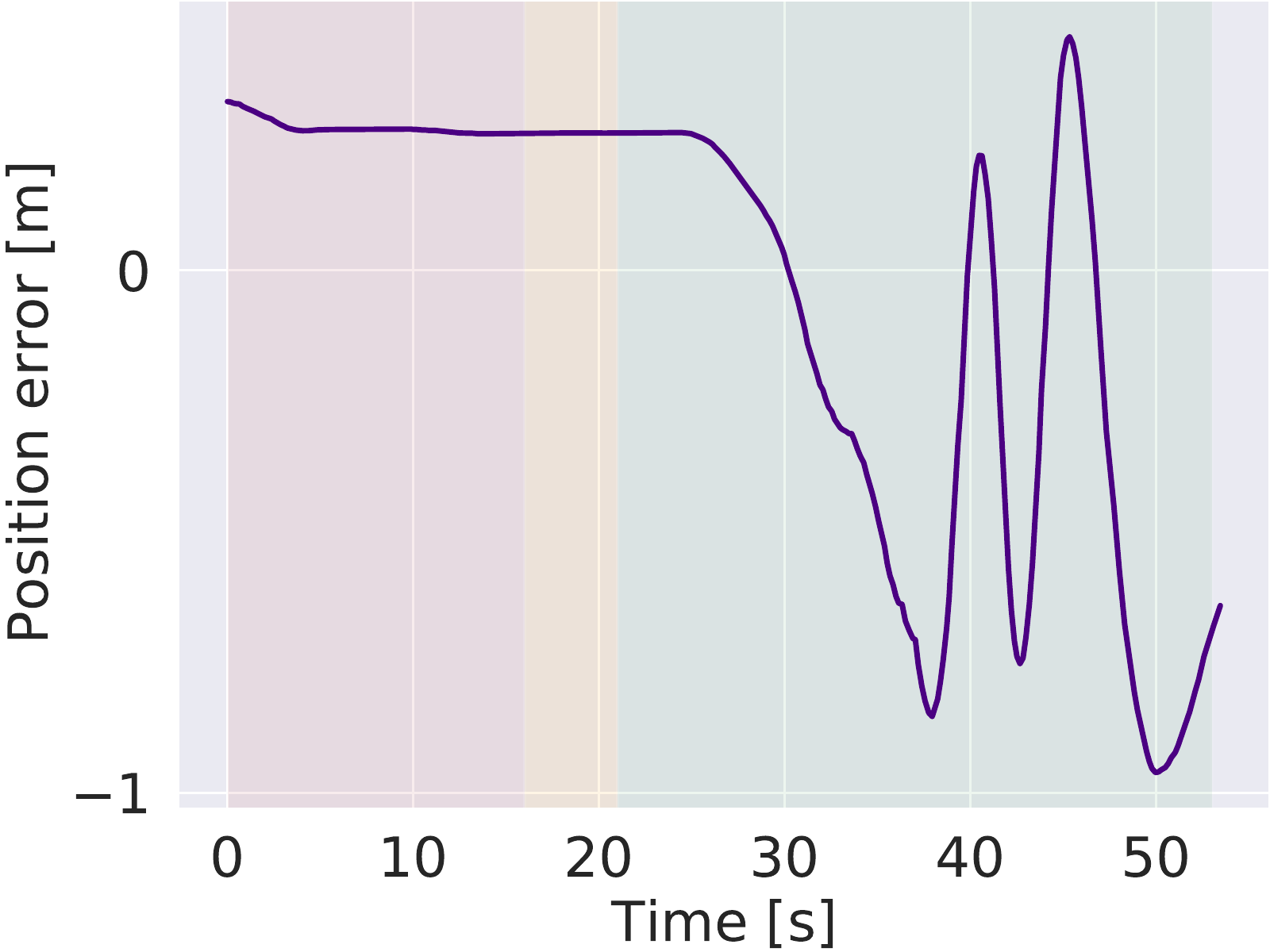}}
\subfloat[Velocity error]{\includegraphics[width=0.15\textwidth]{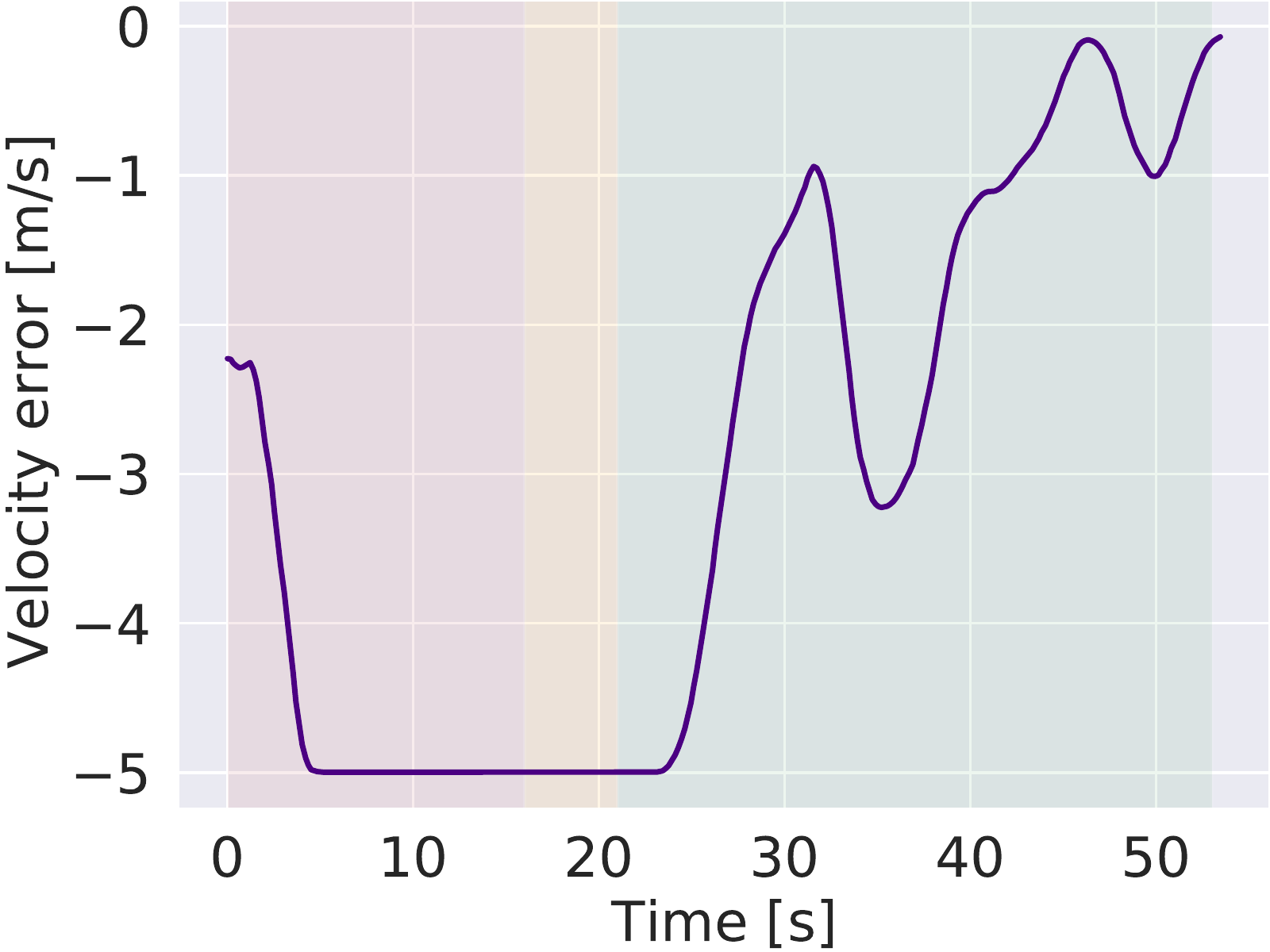}}
\subfloat[Heading error]{\includegraphics[width=0.15\textwidth]{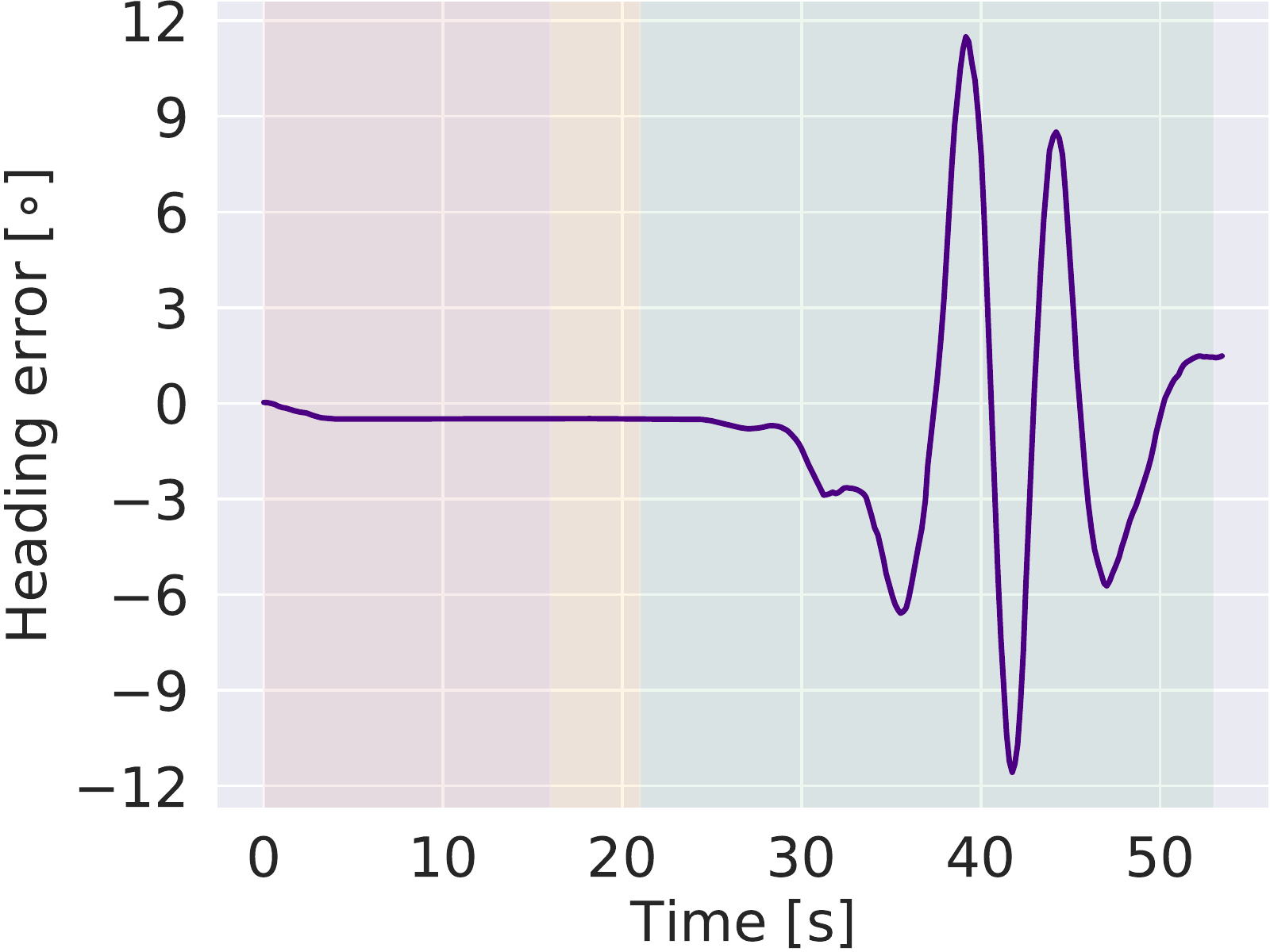}}\\
\subfloat[Steering angle]{\includegraphics[width=0.15\textwidth]{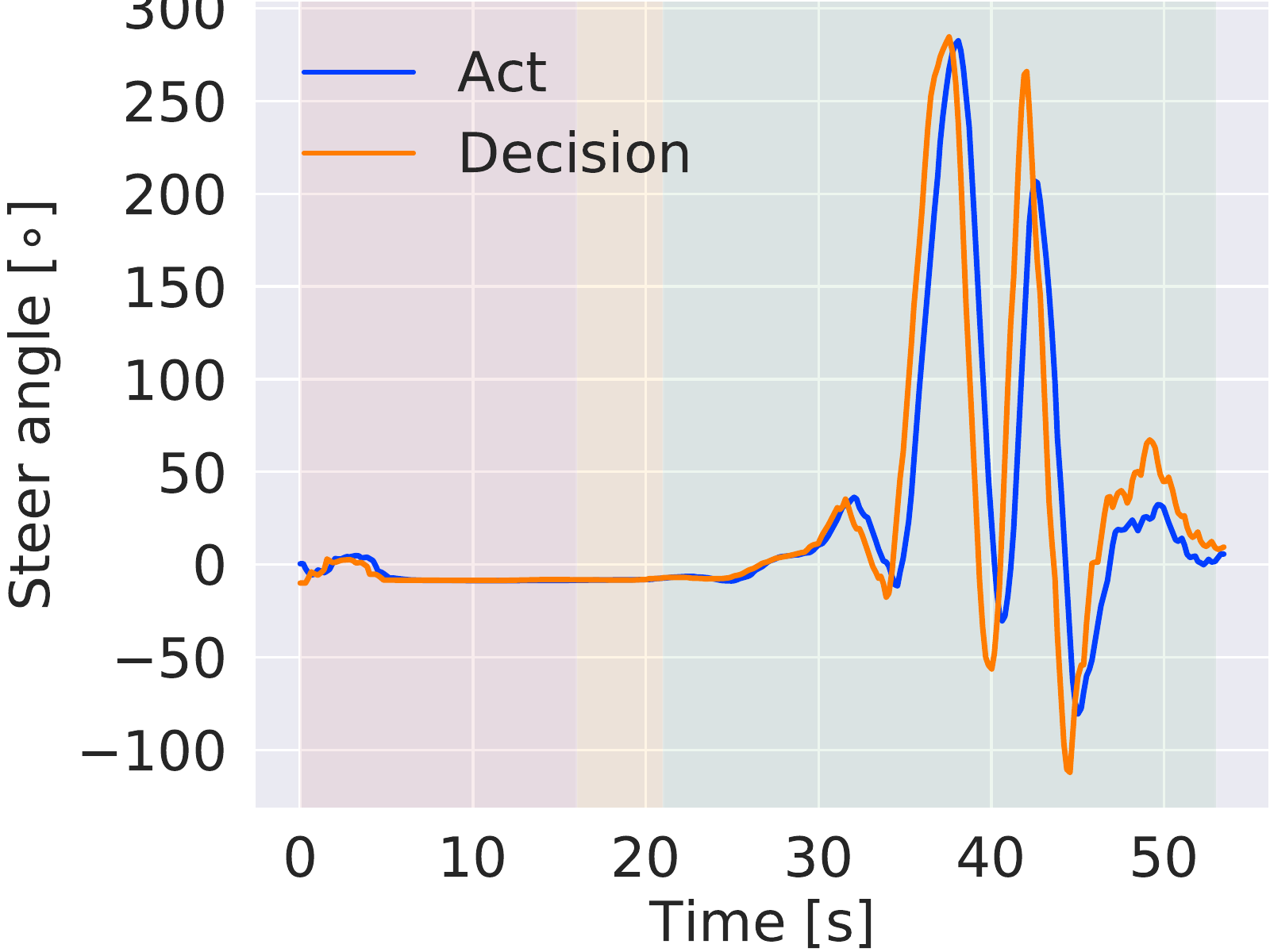}}
\subfloat[Acceleration]{\includegraphics[width=0.15\textwidth]{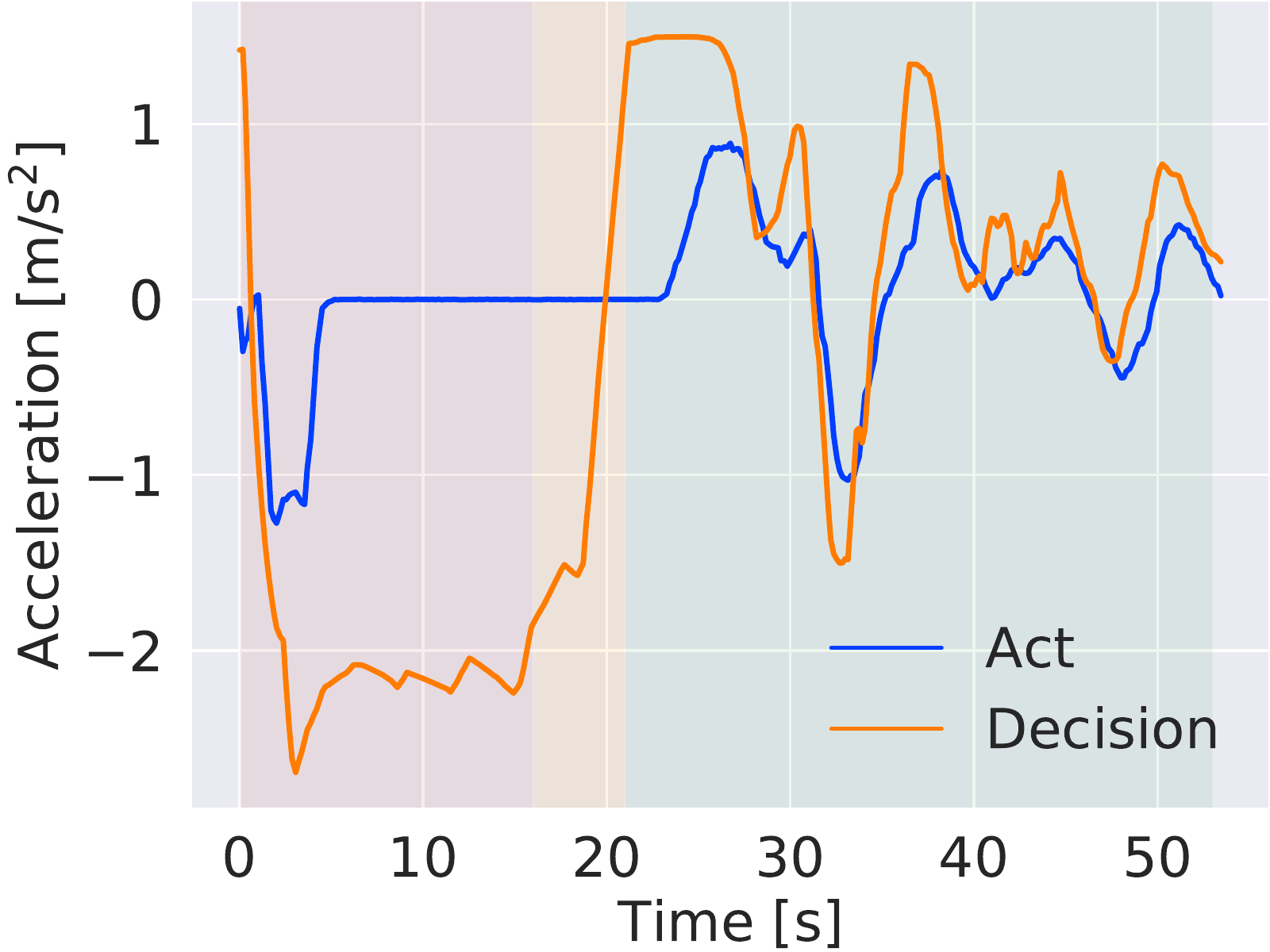}}
\subfloat[Computing time]{\includegraphics[width=0.15\textwidth]{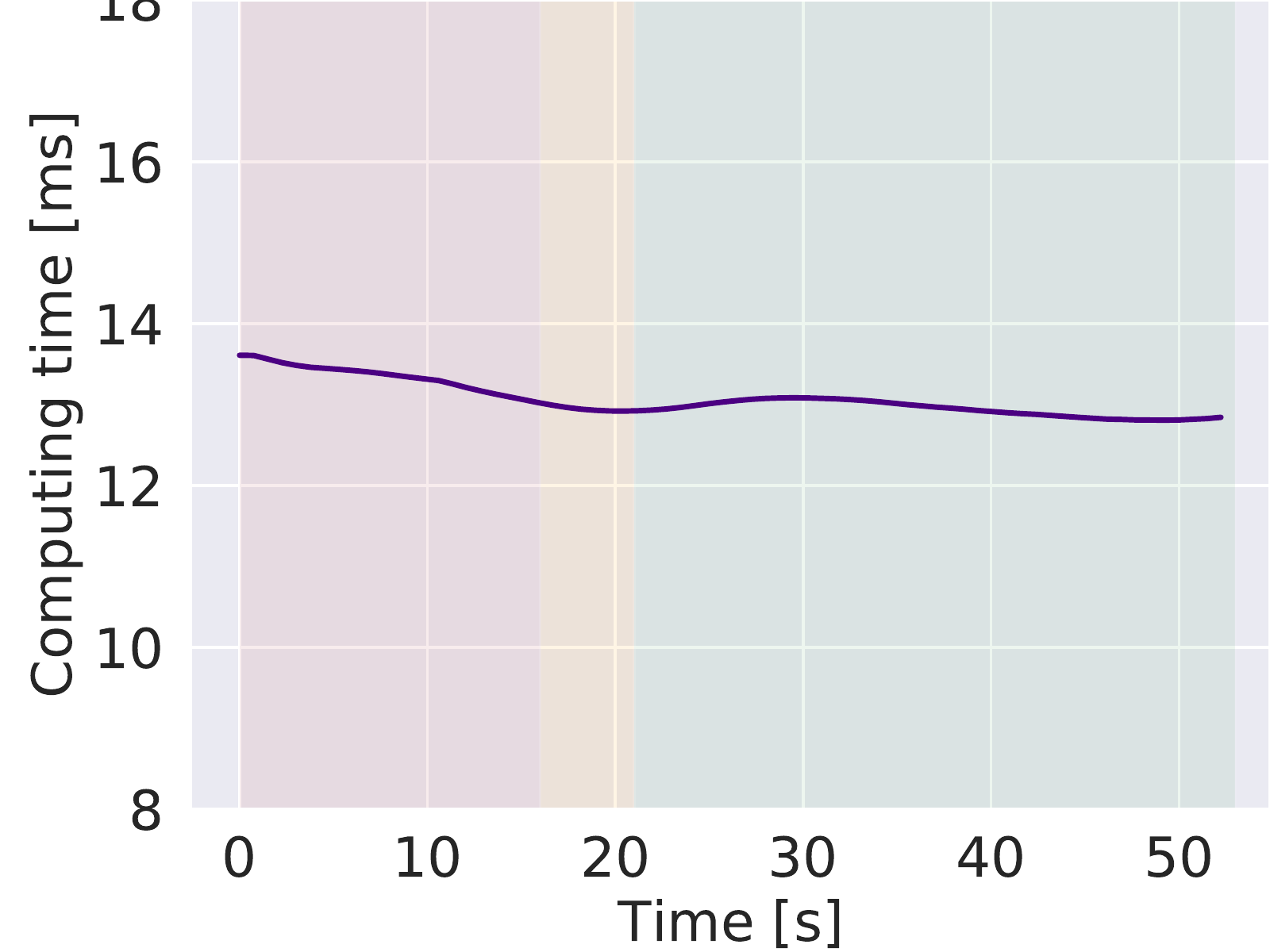}}
\caption{Key parameters in the left run.}
\label{fig.real_test_curves}
\end{figure}

\subsection{Experiment 1:  Functionality verification}
This experiment aims to verify the functionality of the IDC framework under different tasks and scenarios. In total, nine runs were carried out, three for each task. In each run, the ego vehicle is initialized before the south entrance with random states, meanwhile, the surrounding vehicles and signals are also initialized randomly. Following the diagram in Fig. \ref{fig.real_test_setting}, the run keeps going on until the ego passes the intersection successfully, i.e., without colliding with obstacles or breaking traffic rules. The diversity among different runs is guaranteed by using different random seeds. All the videos are available online (https://youtu.be/adqjor5KXxQ).

Since the left-turning is the most complex task with the most potential interactions with surrounding vehicles, we visualize one of the left runs by snapshotting its featured time steps shown in Fig.~\ref{fig.real_test_demo} and drawing the key parameters in Fig. \ref{fig.real_test_curves}. At the beginning, the ego pulls up before the stop line, waiting for the green light (Fig.~\ref{fig.real_test_step1}). When that comes, the ego accelerates into the intersection to reduce the velocity tracking error (Fig.~\ref{fig.real_test_step2}). In the center of the intersection, it encounters a straight-going vehicle with high speed from the opposite direction. In order to avoid collision, the ego slows itself down and switches to the path 4, with which it is able to bypass the vehicle from back (Fig.~\ref{fig.real_test_step3}). However, another straight-going vehicle comes over after the previous one passes through, but with a relative low speed. This time, the ego chooses to accelerate to pass first. Interestingly, as the vehicle approaches, the optimal path is automatically selected away from it, i.e., changing from the path 4 to the path 3 and finally the path 2, to minimize the tracking errors (Fig.~\ref{fig.real_test_step3} and Fig.~\ref{fig.real_test_step4}). Following the path 2, the ego finally passes the intersection successfully. The computing time of all step is within 15ms, showing the superior of our method in terms of the online computing efficiency.

\subsection{Experiment 2: Robustness to noise}
This experiment aims to compare the driving performance under different levels of noises added manually to verify the robustness of the trained policies. Referring to \cite{duan2021DistributedRL}, we take similar measure to divide the noises into 7 levels, i.e. 0-6, where all the noises are in form of Gaussian white noise with different variances varying with the level and are applied in several dimensions of RL states, as shown in Table \ref{tab.noise_level}.

We choose the left-turning task to perform seven experiments, one for each noise level in the Table \ref{tab.noise_level}, to show its influence on the effect of the proposed framework. For each noise level, i.e., each experiment, we make statistical analysis on the parameters related to vehicle stability, namely the yaw rate and lateral speed, and the control quantities, i.e., steering angle and acceleration, as shown in Fig. \ref{fig.real_test_exp2}. Our method is rather robust to low level noises (0-3) in which the distribution parameters including the median value, the standard variance, the quantile values and the bounds have no significant change. However, these parameters, especially the variance and the bounds, are inevitably enlarged if we add stronger noise. The fluctuations of the lateral velocity and the yaw rate are mainly caused by the sensitivity of the steering wheel, because large noises tend to yield large variance of the steering angle, which further leads to the swing of the vehicle body. But nevertheless the stability bounds always remain in a reasonable range, proving the robustness of the proposed method.
\begin{table}
\centering
\caption{Noise level and the corresponding standard deviation.}
\label{tab.noise_level}
\begin{tabular}{cccccccc}
\toprule
Noise level & 0 & 1 & 2 & 3 & 4 & 5 & 6 \\
\midrule
$\delta_p$ [m]& 0 & 0.017 & 0.033 & 0.051 & 0.068 & 0.085 & 0.102 \\
$\delta_{\phi}$ [$\degree$]& 0 & 0.017 & 0.033 & 0.051 & 0.068 & 0.085 & 0.102 \\
$p^j_x,p^j_y$ [m]& 0 & 0.05 & 0.10 & 0.15 & 0.20 & 0.25 & 0.30 \\
$v^j_{\rm lon}$ [m/s]& 0 & 0.05 & 0.10 & 0.15 & 0.20 & 0.25 & 0.30 \\
$\phi^j$ [$\degree$]& 0 & 1.4 & 2.8 & 4.2 & 5.6 & 7.0 & 8.4 \\
\bottomrule
\end{tabular}
\end{table}

\begin{figure}[htbp]
\centering
\captionsetup[subfigure]{justification=centering}
\subfloat[Lateral velocity]{\includegraphics[width=0.24\textwidth]{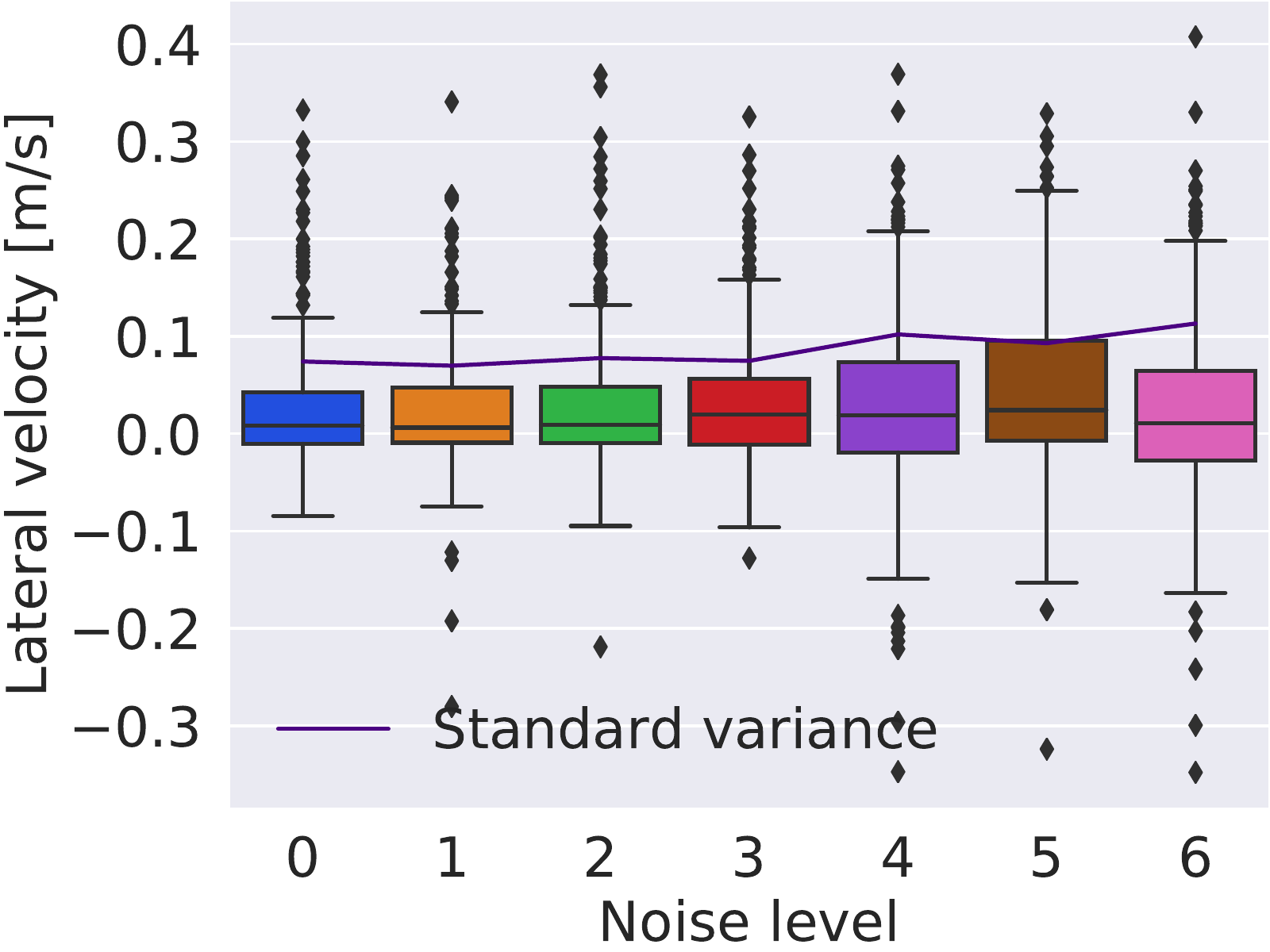}}
\subfloat[Yaw rate]{\includegraphics[width=0.24\textwidth]{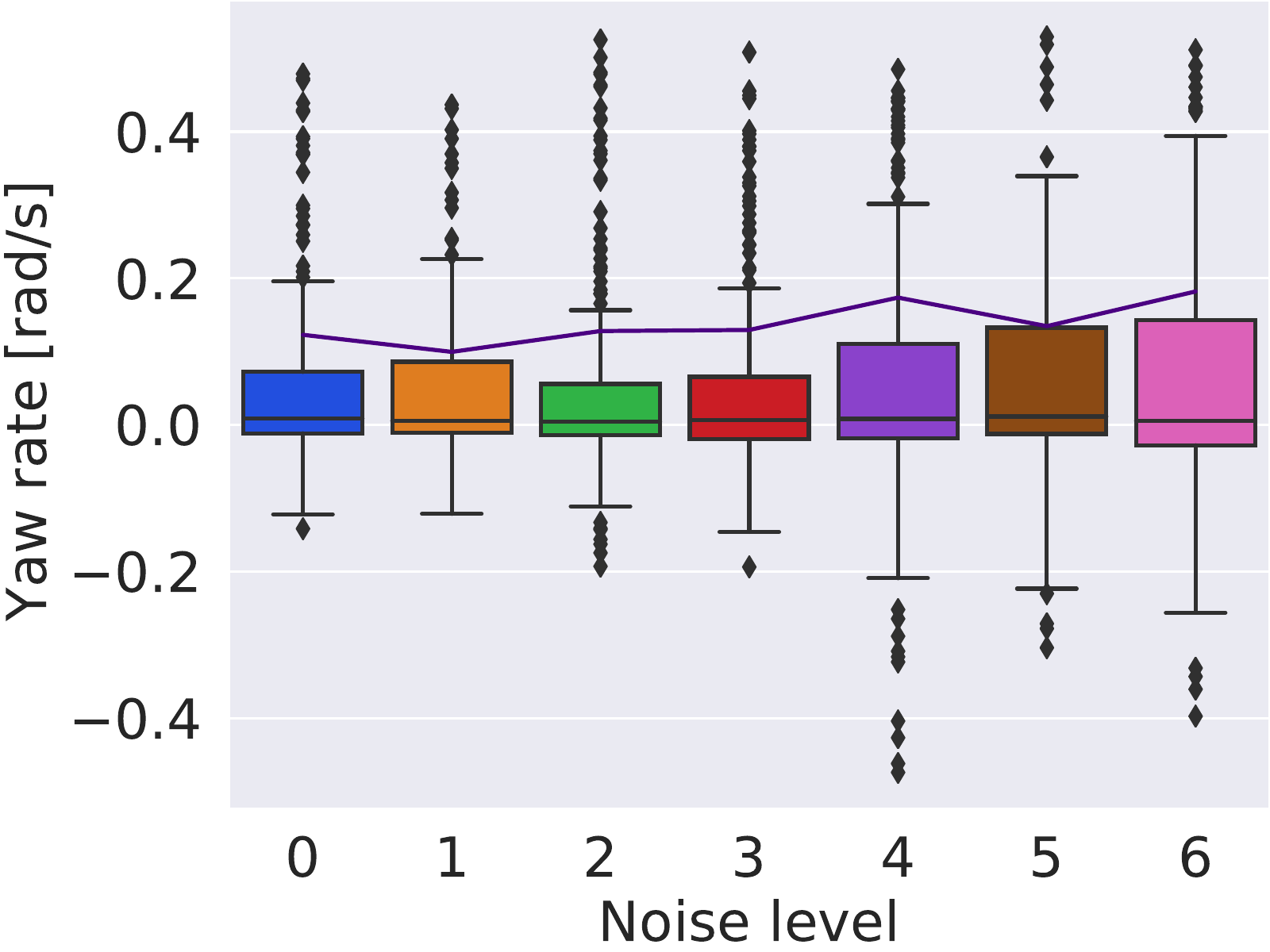}}\\
\subfloat[Steering angle]{\includegraphics[width=0.24\textwidth]{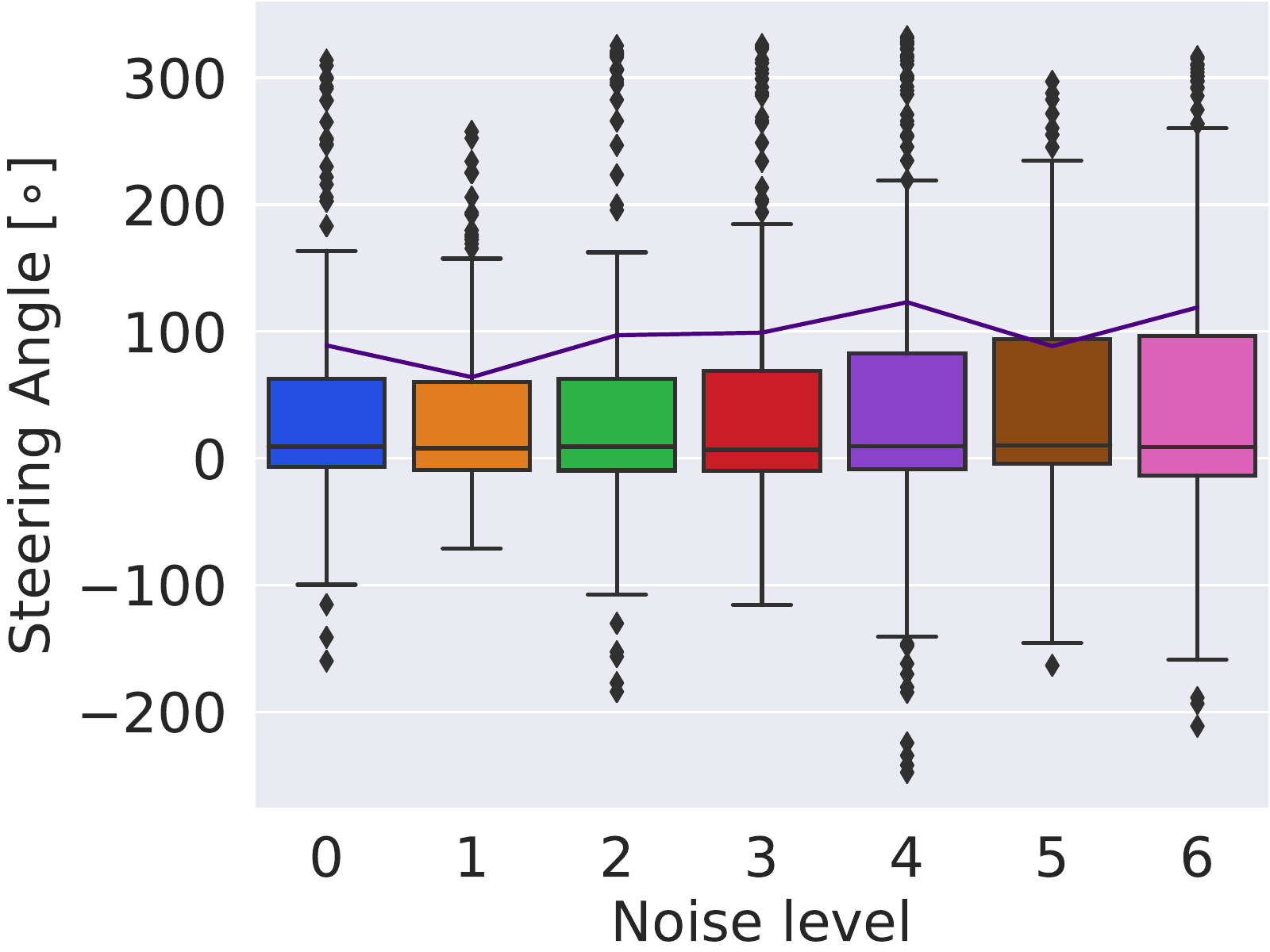}}
\subfloat[Acceleration]{\includegraphics[width=0.24\textwidth]{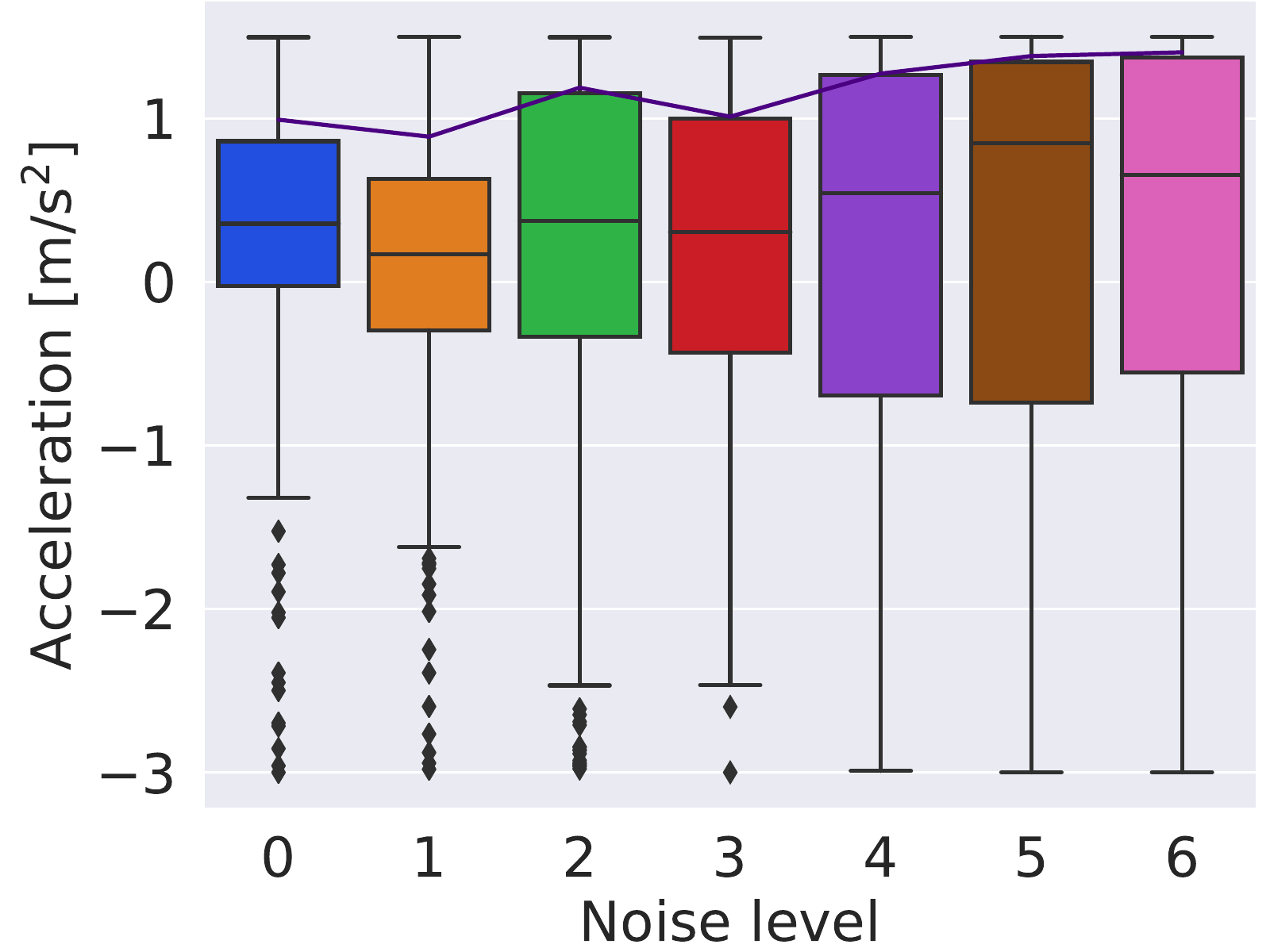}}\\
\caption{States and actions in different noise level.}
\label{fig.real_test_exp2}
\end{figure}

\subsection{Experiment 3: Robustness to human disturbance}
The experiment is to verify the ability of our method to cope with human disturbance. We also use a left-turning case, in which we perform two times of human intervention on the steer wheel, and then switch to the autonomous driving mode. We draw the states of the ego vehicle in Fig. \ref{fig.real_test_exp3}, where the colored region is when the human disturbance is acted on. The first one is acted on 10s, when the ego just enters the crossroad and we turn the steering wheel left to 100$\degree$ from $0\degree$ to make the ego head to the left. After that the driving system turn the steering wheel right immediately to correct the excessive ego heading. The second one happens at 16s, when the ego is turning to the left to pass the crossroad. We turn the steering wheel right from $90\degree$ to $0\degree$ to interrupt the process. After the take-over, the driving system is able to turn the steering wheel left to $240\degree$ right away to continue to complete the turn left operation. Results show that the proposed method is capable of dealing with the abrupt human disturbance on the steering wheel by quick responses to the interrupted state after taking over.

\begin{figure}[htbp]
\centering
\captionsetup[subfigure]{justification=centering}
\subfloat[Speed]{\includegraphics[width=0.24\textwidth]{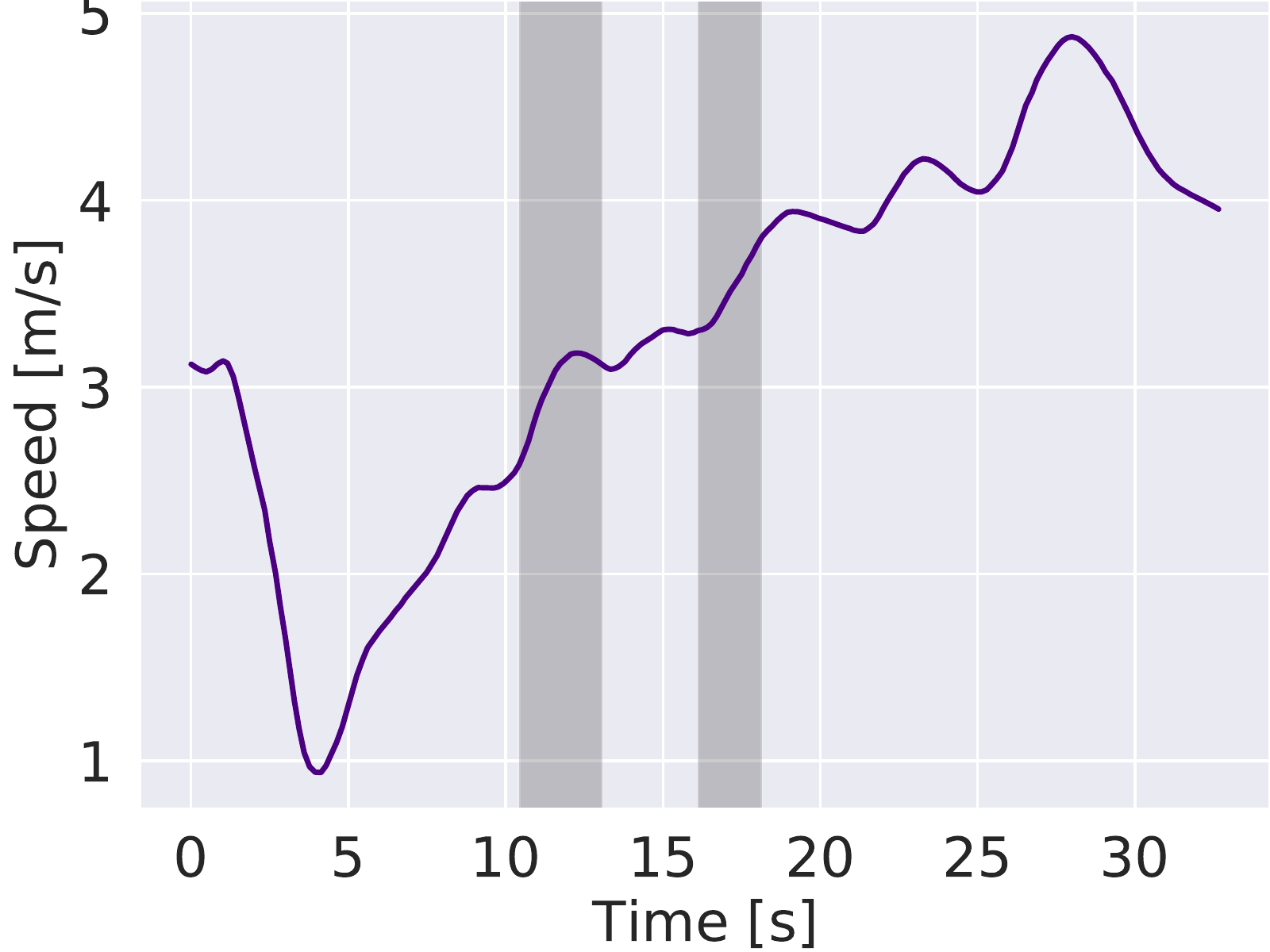}}
\subfloat[Heading angle]{\includegraphics[width=0.24\textwidth]{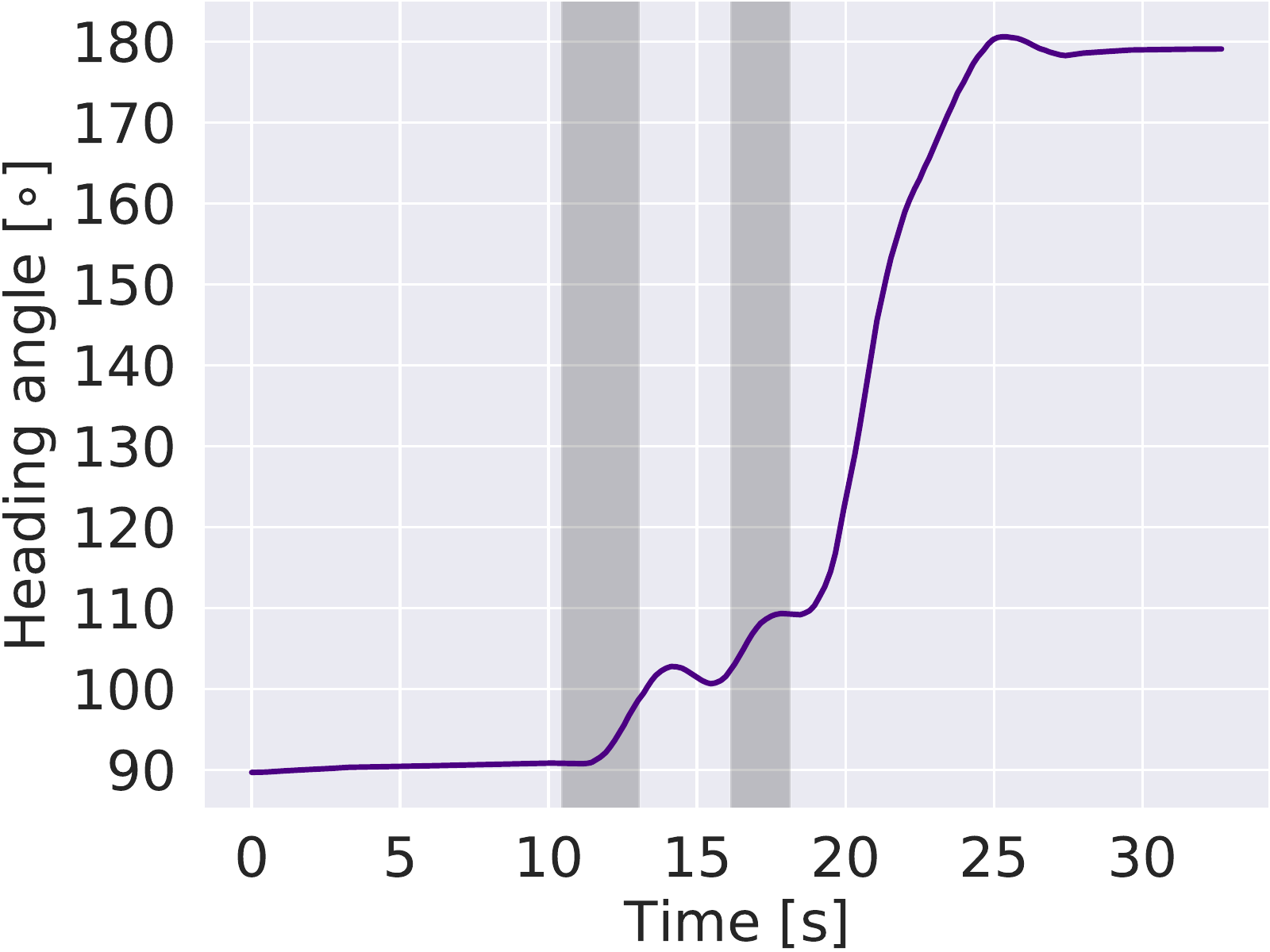}}\\
\subfloat[Steering angle]{\includegraphics[width=0.24\textwidth]{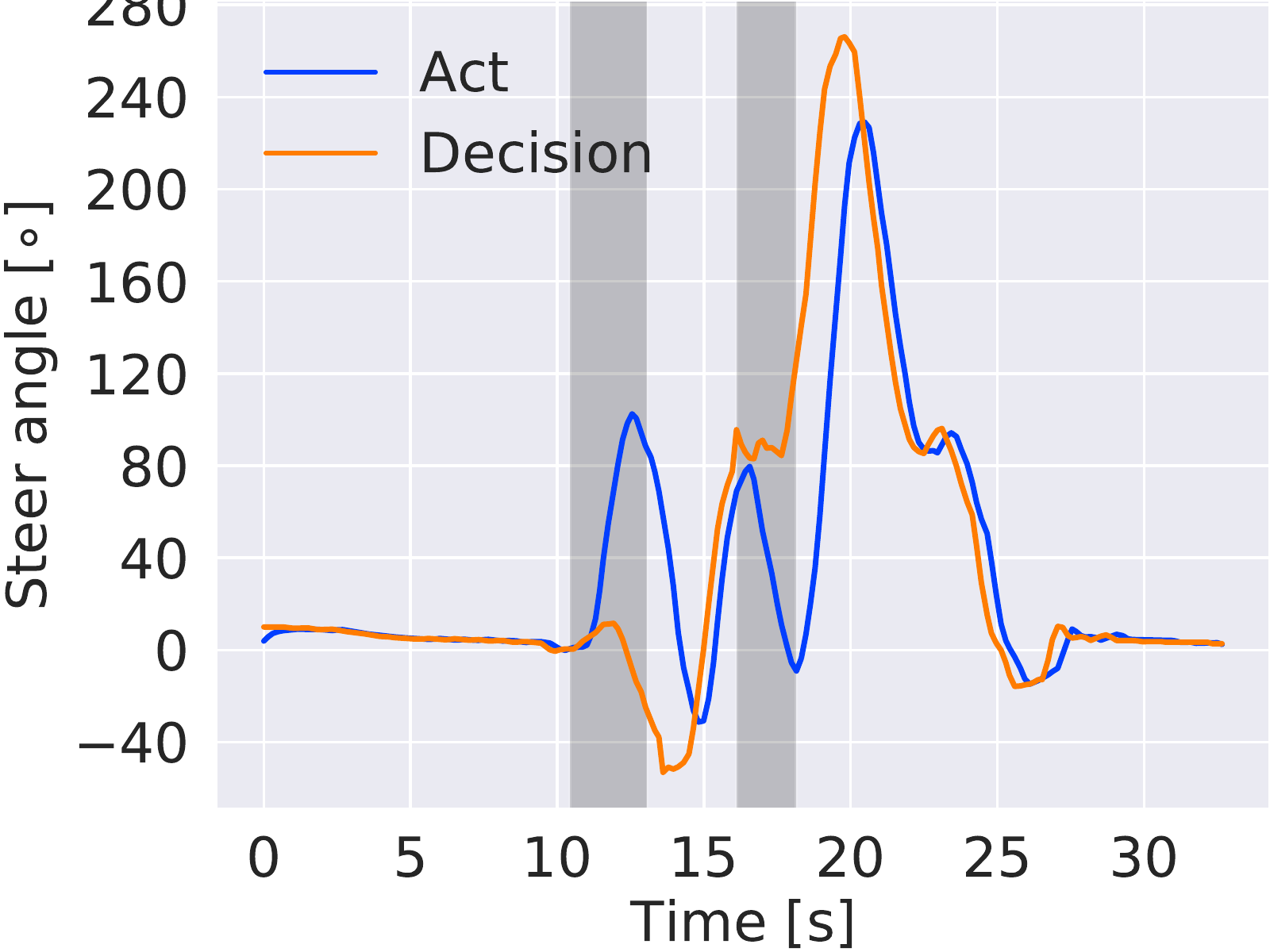}}
\subfloat[Acceleration]{\includegraphics[width=0.24\textwidth]{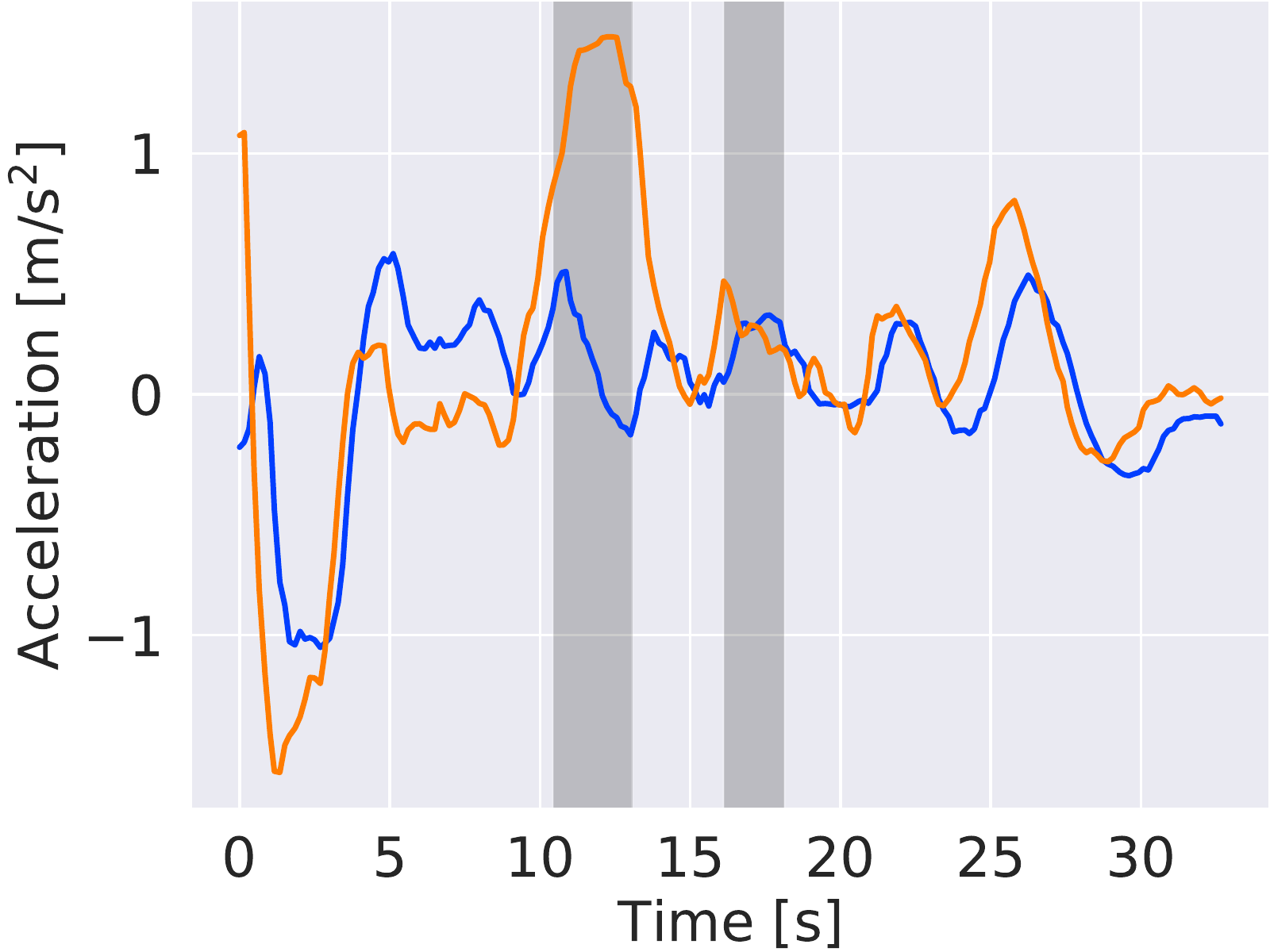}}\\
\caption{States and actions under human disturbance.}
\label{fig.real_test_exp3}
\end{figure}

\section{Conclusion}
In this paper, we propose integrated decision and control (IDC) framework for automated vehicles, for the purpose of building an interpretable learning system with high online computing efficiency and applicability among different driving tasks and scenarios. The framework decomposes the driving task into the static path planning and the dynamic optimal tracking hierarchically. The former is in charge of generating multiple paths only considering static constraints, which are then sent to the latter to be selected and tracked. The latter first formulates the selecting and tracking problem as constrained OCPs mathematically to take dynamic obstacles into consideration, and then solves it offline by a model-based RL algorithm we propose to seek an approximate solution of an OCP in form of neural networks. Notably, these solved approximation functions, namely value and policy, have a natural correspondence to the selecting and tracking problems, which originates the interpretability. Finally, the value and policy functions are used online instead, releasing the heavy computation due to online optimizations. We verify our framework in both simulation and in a real world intersection. Results show that our method has an order of magnitude higher online computing efficiency compared with the traditional rule-based method. In addition, it yields better driving performance in terms of traffic efficiency and safety, and shows great interpretability and adaptability among different driving tasks. In the future, we will place efforts on designing more general state representations to extend the tracking ability of the lower layer among paths in different tasks or even different scenarios.

\bibliographystyle{IEEEtran}
\bibliography{main.bbl}

\begin{thebibliography}{10}
\providecommand{\url}[1]{#1}
\csname url@samestyle\endcsname
\providecommand{\newblock}{\relax}
\providecommand{\bibinfo}[2]{#2}
\providecommand{\BIBentrySTDinterwordspacing}{\spaceskip=0pt\relax}
\providecommand{\BIBentryALTinterwordstretchfactor}{4}
\providecommand{\BIBentryALTinterwordspacing}{\spaceskip=\fontdimen2\font plus
\BIBentryALTinterwordstretchfactor\fontdimen3\font minus
  \fontdimen4\font\relax}
\providecommand{\BIBforeignlanguage}[2]{{%
\expandafter\ifx\csname l@#1\endcsname\relax
\typeout{** WARNING: IEEEtran.bst: No hyphenation pattern has been}%
\typeout{** loaded for the language `#1'. Using the pattern for}%
\typeout{** the default language instead.}%
\else
\language=\csname l@#1\endcsname
\fi
#2}}
\providecommand{\BIBdecl}{\relax}
\BIBdecl

\bibitem{paden2016survey}
B.~Paden, M.~{\v{C}}{\'a}p, S.~Z. Yong, D.~Yershov, and E.~Frazzoli, ``A survey
  of motion planning and control techniques for self-driving urban vehicles,''
  \emph{IEEE Transactions on intelligent vehicles}, vol.~1, no.~1, pp. 33--55,
  2016.

\bibitem{lefevre2014survey}
S.~Lef{\`e}vre, D.~Vasquez, and C.~Laugier, ``A survey on motion prediction and
  risk assessment for intelligent vehicles,'' \emph{ROBOMECH journal}, vol.~1,
  no.~1, pp. 1--14, 2014.

\bibitem{li2017estimation}
G.~Li, S.~E. Li, B.~Cheng, and P.~Green, ``Estimation of driving style in
  naturalistic highway traffic using maneuver transition probabilities,''
  \emph{Transportation Research Part C: Emerging Technologies}, vol.~74, pp.
  113--125, 2017.

\bibitem{sun2018proba}
L.~{Sun}, W.~{Zhan}, and M.~{Tomizuka}, ``Probabilistic prediction of
  interactive driving behavior via hierarchical inverse reinforcement
  learning,'' in \emph{2018 21st International Conference on Intelligent
  Transportation Systems (ITSC)}, 2018, pp. 2111--2117.

\bibitem{alahi2016social}
A.~Alahi, K.~Goel, V.~Ramanathan, A.~Robicquet, L.~Fei-Fei, and S.~Savarese,
  ``Social lstm: Human trajectory prediction in crowded spaces,'' in
  \emph{Proceedings of the IEEE conference on computer vision and pattern
  recognition}, 2016, pp. 961--971.

\bibitem{hou2019interactive}
L.~Hou, L.~Xin, S.~E. Li, B.~Cheng, and W.~Wang, ``Interactive trajectory
  prediction of surrounding road users for autonomous driving using
  structural-lstm network,'' \emph{IEEE Transactions on Intelligent
  Transportation Systems}, 2019.

\bibitem{montemerlo2008junior}
M.~Montemerlo, J.~Becker, S.~Bhat, H.~Dahlkamp, D.~Dolgov, S.~Ettinger,
  D.~Haehnel, T.~Hilden, G.~Hoffmann, B.~Huhnke \emph{et~al.}, ``Junior: The
  stanford entry in the urban challenge,'' \emph{Journal of field Robotics},
  vol.~25, no.~9, pp. 569--597, 2008.

\bibitem{olsson2016dt}
M.~Olsson, ``Behavior trees for decision-making in autonomous driving,'' 2016.

\bibitem{gray2012predictive}
A.~Gray, Y.~Gao, T.~Lin, J.~K. Hedrick, H.~E. Tseng, and F.~Borrelli,
  ``Predictive control for agile semi-autonomous ground vehicles using motion
  primitives,'' in \emph{2012 American Control Conference (ACC)}.\hskip 1em
  plus 0.5em minus 0.4em\relax IEEE, 2012, pp. 4239--4244.

\bibitem{nilsson2014manoeuvre}
J.~Nilsson, Y.~Gao, A.~Carvalho, and F.~Borrelli, ``Manoeuvre generation and
  control for automated highway driving,'' \emph{IFAC Proceedings Volumes},
  vol.~47, no.~3, pp. 6301--6306, 2014.

\bibitem{dolgov2010path}
D.~Dolgov, S.~Thrun, M.~Montemerlo, and J.~Diebel, ``Path planning for
  autonomous vehicles in unknown semi-structured environments,'' \emph{The
  International Journal of Robotics Research}, vol.~29, no.~5, pp. 485--501,
  2010.

\bibitem{lee2014local}
U.~Lee, S.~Yoon, H.~Shim, P.~Vasseur, and C.~Demonceaux, ``Local path planning
  in a complex environment for self-driving car,'' in \emph{The 4th Annual IEEE
  International Conference on Cyber Technology in Automation, Control and
  Intelligent}.\hskip 1em plus 0.5em minus 0.4em\relax IEEE, 2014, pp.
  445--450.

\bibitem{ardakani2015decremental}
M.~K. Ardakani and M.~Tavana, ``A decremental approach with the a* algorithm
  for speeding-up the optimization process in dynamic shortest path problems,''
  \emph{Measurement}, vol.~60, pp. 299--307, 2015.

\bibitem{lavalle1998rapidly}
S.~M. LaValle \emph{et~al.}, ``Rapidly-exploring random trees: A new tool for
  path planning,'' 1998.

\bibitem{kuwata2009real}
Y.~Kuwata, J.~Teo, G.~Fiore, S.~Karaman, E.~Frazzoli, and J.~P. How,
  ``Real-time motion planning with applications to autonomous urban driving,''
  \emph{IEEE Transactions on control systems technology}, vol.~17, no.~5, pp.
  1105--1118, 2009.

\bibitem{karaman2010incremental}
S.~Karaman and E.~Frazzoli, ``Incremental sampling-based algorithms for optimal
  motion planning,'' \emph{Robotics Science and Systems VI}, vol. 104, no.~2,
  2010.

\bibitem{shah2015autonomous}
J.~Shah, M.~Best, A.~Benmimoun, and M.~L. Ayat, ``Autonomous rear-end collision
  avoidance using an electric power steering system,'' \emph{Proceedings of the
  Institution of Mechanical Engineers, Part D: Journal of Automobile
  Engineering}, vol. 229, no.~12, pp. 1638--1655, 2015.

\bibitem{mouhagir2017trajectory}
H.~Mouhagir, V.~Cherfaoui, R.~Talj, F.~Aioun, and F.~Guillemard, ``Trajectory
  planning for autonomous vehicle in uncertain environment using evidential
  grid,'' \emph{IFAC-PapersOnLine}, vol.~50, no.~1, pp. 12\,545--12\,550, 2017.

\bibitem{xin2021enable}
L.~Xin, Y.~Kong, S.~E. Li, J.~Chen, Y.~Guan, M.~Tomizuka, and B.~Cheng,
  ``Enable faster and smoother spatio-temporal trajectory planning for
  autonomous vehicles in constrained dynamic environment,'' \emph{Proceedings
  of the Institution of Mechanical Engineers, Part D: Journal of Automobile
  Engineering}, vol. 235, no.~4, pp. 1101--1112, 2021.

\bibitem{eben2013economy}
S.~Eben~Li, K.~Li, and J.~Wang, ``Economy-oriented vehicle adaptive cruise
  control with coordinating multiple objectives function,'' \emph{Vehicle
  System Dynamics}, vol.~51, no.~1, pp. 1--17, 2013.

\bibitem{li2010model}
S.~Li, K.~Li, R.~Rajamani, and J.~Wang, ``Model predictive multi-objective
  vehicular adaptive cruise control,'' \emph{IEEE Transactions on Control
  Systems Technology}, vol.~19, no.~3, pp. 556--566, 2010.

\bibitem{EbenRL}
S.~E. Li, ``Reinforcement learning and control,'' 2020, tsinghua University:
  Lecture Notes. http://www.idlab-tsinghua.com/thulab/labweb/publications.html.

\bibitem{Kiran2021deep}
B.~R. {Kiran}, I.~{Sobh}, V.~{Talpaert}, P.~{Mannion}, A.~A.~A. {Sallab},
  S.~{Yogamani}, and P.~{Pérez}, ``Deep reinforcement learning for autonomous
  driving: A survey,'' \emph{IEEE Transactions on Intelligent Transportation
  Systems}, pp. 1--18, 2021.

\bibitem{sallab2017deep}
A.~E. Sallab, M.~Abdou, E.~Perot, and S.~Yogamani, ``Deep reinforcement
  learning framework for autonomous driving,'' \emph{Electronic Imaging}, vol.
  2017, no.~19, pp. 70--76, 2017.

\bibitem{wang2018reinforcement}
P.~Wang, C.-Y. Chan, and A.~de~La~Fortelle, ``A reinforcement learning based
  approach for automated lane change maneuvers,'' in \emph{2018 IEEE
  Intelligent Vehicles Symposium (IV)}.\hskip 1em plus 0.5em minus 0.4em\relax
  IEEE, 2018, pp. 1379--1384.

\bibitem{ngai2011multiple}
D.~C.~K. Ngai and N.~H.~C. Yung, ``A multiple-goal reinforcement learning
  method for complex vehicle overtaking maneuvers,'' \emph{IEEE Transactions on
  Intelligent Transportation Systems}, vol.~12, no.~2, pp. 509--522, 2011.

\bibitem{mnih2015human}
V.~Mnih, K.~Kavukcuoglu, D.~Silver, A.~A. Rusu, J.~Veness, M.~G. Bellemare,
  A.~Graves, M.~Riedmiller, A.~K. Fidjeland, G.~Ostrovski \emph{et~al.},
  ``Human-level control through deep reinforcement learning,'' \emph{nature},
  vol. 518, no. 7540, pp. 529--533, 2015.

\bibitem{lillicrap2015ddpg}
T.~P. Lillicrap, J.~J. Hunt, A.~Pritzel, N.~Heess, T.~Erez, Y.~Tassa,
  D.~Silver, and D.~Wierstra, ``Continuous control with deep reinforcement
  learning,'' in \emph{4th International Conference on Learning
  Representations, {ICLR} 2016, San Juan, Puerto Rico, May 2-4, 2016}, 2016.

\bibitem{guan2020centralized}
Y.~Guan, Y.~Ren, S.~E. Li, Q.~Sun, L.~Luo, and K.~Li, ``Centralized cooperation
  for connected and automated vehicles at intersections by proximal policy
  optimization,'' \emph{IEEE Transactions on Vehicular Technology}, vol.~69,
  no.~11, pp. 12\,597--12\,608, 2020.

\bibitem{chen2019model}
J.~Chen, B.~Yuan, and M.~Tomizuka, ``Model-free deep reinforcement learning for
  urban autonomous driving,'' in \emph{2019 IEEE Intelligent Transportation
  Systems Conference (ITSC)}.\hskip 1em plus 0.5em minus 0.4em\relax IEEE,
  2019, pp. 2765--2771.

\bibitem{ge2020numerically}
Q.~Ge, S.~E. Li, Q.~Sun, and S.~Zheng, ``Numerically stable dynamic bicycle
  model for discrete-time control,'' \emph{arXiv preprint arXiv:2011.09612},
  2020.

\bibitem{allen2019convergence}
Z.~Allen-Zhu, Y.~Li, and Z.~Song, ``A convergence theory for deep learning via
  over-parameterization,'' in \emph{International Conference on Machine
  Learning}.\hskip 1em plus 0.5em minus 0.4em\relax PMLR, 2019, pp. 242--252.

\bibitem{SUMO2018}
\BIBentryALTinterwordspacing
P.~A. Lopez, M.~Behrisch, L.~Bieker-Walz, J.~Erdmann, Y.-P. Fl{\"o}tter{\"o}d,
  R.~Hilbrich, L.~L{\"u}cken, J.~Rummel, P.~Wagner, and E.~Wie{\ss}ner,
  ``Microscopic traffic simulation using sumo,'' in \emph{The 21st IEEE
  International Conference on Intelligent Transportation Systems}.\hskip 1em
  plus 0.5em minus 0.4em\relax IEEE, 2018. [Online]. Available:
  \url{https://elib.dlr.de/124092/}
\BIBentrySTDinterwordspacing

\bibitem{guan2021mixed}
Y.~Guan, J.~Duan, S.~E. Li, J.~Li, J.~Chen, and B.~Cheng, ``Mixed policy
  gradient,'' \emph{arXiv preprint arXiv:2102.11513}, 2021.

\bibitem{clevert2015fast}
D.-A. Clevert, T.~Unterthiner, and S.~Hochreiter, ``Fast and accurate deep
  network learning by exponential linear units (elus),'' \emph{arXiv preprint
  arXiv:1511.07289}, 2015.

\bibitem{kingma2014adam}
D.~P. Kingma and J.~Ba, ``Adam: A method for stochastic optimization,''
  \emph{arXiv preprint arXiv:1412.6980}, 2014.

\bibitem{haarnoja2018soft}
T.~Haarnoja, A.~Zhou, P.~Abbeel, and S.~Levine, ``Soft actor-critic: Off-policy
  maximum entropy deep reinforcement learning with a stochastic actor,'' in
  \emph{International Conference on Machine Learning}.\hskip 1em plus 0.5em
  minus 0.4em\relax PMLR, 2018, pp. 1861--1870.

\bibitem{duan2021DistributedRL}
J.~Duan, ``Study on distributional reinforcement learning for decision-making
  in autonoumou driving,'' Ph.D. dissertation, Tsinghua University, Beijing,
  China, 2021.

\end{thebibliography}

\end{document}